\documentclass[11pt]{article} 
\usepackage{etoolbox}

 \usepackage[letterpaper, left=1in, right=1in, top=1in, bottom=1in]{geometry}

\usepackage[colorlinks=true, linkcolor=blue!70!black, citecolor=blue!70!black]{hyperref}
\usepackage{microtype}

\usepackage{natbib}
\bibpunct{(}{)}{;}{a}{,}{,}

\usepackage{amsthm}
\usepackage{mathtools}
\usepackage{amsmath}
\usepackage{bbm}
\usepackage{amsfonts}
\usepackage{amssymb}

\usepackage{xpatch}
\usepackage{pifont}
\usepackage{array}
\usepackage{booktabs}
\usepackage{floatrow}
\newfloatcommand{capbtabbox}{table}[][\FBwidth]
\usepackage{blindtext}
\usepackage{caption}

\theoremstyle{definition}  %Sets style of subsequent newtheorems to 'definition'
\theoremstyle{theorem}
\newtheorem{lemma}{Lemma}
\newtheorem{theorem}{Theorem}
\newtheorem{corollary}{Corollary}

\newtheorem{remark}{Remark}
\newtheorem{definition}{Definition}
\newtheorem{proposition}{Proposition}

\xpatchcmd{\proof}{\itshape}{\normalfont\proofnameformat}{}{}
\newcommand{\proofnameformat}{\bfseries}

\usepackage{prettyref}
\newcommand{\pref}[1]{\prettyref{#1}}

\newcommand{\savehyperref}[2]{\texorpdfstring{\hyperref[#1]{#2}}{#2}}
\newrefformat{eq}{\savehyperref{#1}{\textup{(\ref*{#1})}}}
\newrefformat{ineq}{\savehyperref{#1}{\textup{(\ref*{#1})}}}
\newrefformat{eqn}{\savehyperref{#1}{Equation~\ref*{#1}}}
\newrefformat{obs}{\savehyperref{#1}{Observation~\ref*{#1}}}
\newrefformat{lem}{\savehyperref{#1}{Lemma~\ref*{#1}}}
\newrefformat{def}{\savehyperref{#1}{Definition~\ref*{#1}}}
\newrefformat{thm}{\savehyperref{#1}{Theorem~\ref*{#1}}}
\newrefformat{fact}{\savehyperref{#1}{Fact~\ref*{#1}}}
\newrefformat{tab}{\savehyperref{#1}{Table~\ref*{#1}}}
\newrefformat{table}{\savehyperref{#1}{Table~\ref*{#1}}}
\newrefformat{line}{\savehyperref{#1}{Line~\ref*{#1}}}
\newrefformat{corr}{\savehyperref{#1}{Corollary~\ref*{#1}}}
\newrefformat{cor}{\savehyperref{#1}{Corollary~\ref*{#1}}}
\newrefformat{sec}{\savehyperref{#1}{Section~\ref*{#1}}}
\newrefformat{app}{\savehyperref{#1}{Appendix~\ref*{#1}}}
\newrefformat{ass}{\savehyperref{#1}{Assumption~\ref*{#1}}}
\newrefformat{ex}{\savehyperref{#1}{Example~\ref*{#1}}}
\newrefformat{fig}{\savehyperref{#1}{Figure~\ref*{#1}}}
\newrefformat{alg}{\savehyperref{#1}{Algorithm~\ref*{#1}}}
\newrefformat{rem}{\savehyperref{#1}{Remark~\ref*{#1}}}
\newrefformat{conj}{\savehyperref{#1}{Conjecture~\ref*{#1}}}
\newrefformat{prop}{\savehyperref{#1}{Proposition~\ref*{#1}}}
\newrefformat{proto}{\savehyperref{#1}{Protocol~\ref*{#1}}}
\newrefformat{prob}{\savehyperref{#1}{Problem~\ref*{#1}}}
\newrefformat{claim}{\savehyperref{#1}{Claim~\ref*{#1}}}
\newrefformat{tbl}{\savehyperref{#1}{Table~\ref*{#1}}}

\usepackage{enumitem}      

\newcommand\numberthis{\addtocounter{equation}{1}\tag{\theequation}}
\allowdisplaybreaks

\DeclarePairedDelimiter{\abs}{\lvert}{\rvert} 
\DeclarePairedDelimiter{\brk}{[}{]}
\DeclarePairedDelimiter{\crl}{\{}{\}}
\DeclarePairedDelimiter{\prn}{(}{)}
\DeclarePairedDelimiter{\nrm}{\|}{\|}
\DeclarePairedDelimiter{\tri}{\langle}{\rangle}

\DeclarePairedDelimiter{\ceil}{\lceil}{\rceil}

\let\Pr\undefined
 
\DeclareMathOperator{\En}{\mathbb{E}}

\DeclareMathOperator{\Pr}{Pr}

\DeclareMathOperator*{\argmin}{argmin} % * Places subscript directly under operator
\DeclareMathOperator*{\argmax}{argmax}

\newcommand{\ls}{\ell}

\newcommand{\ldef}{\vcentcolon=}

\newcommand{\wt}[1]{\widetilde{#1}}
\newcommand{\wh}[1]{\widehat{#1}}
\newcommand{\mb}[1]{\boldsymbol{#1}}

\def\ddefloop#1{\ifx\ddefloop#1\else\ddef{#1}\expandafter\ddefloop\fi}
\def\ddef#1{\expandafter\def\csname bb#1\endcsname{\ensuremath{\mathbb{#1}}}}
\ddefloop ABCDEFGHIJKLMNOPQRSTUVWXYZ\ddefloop
\def\ddefloop#1{\ifx\ddefloop#1\else\ddef{#1}\expandafter\ddefloop\fi}
\def\ddef#1{\expandafter\def\csname b#1\endcsname{\ensuremath{\mathbf{#1}}}}
\ddefloop ABCDEFGHIJKLMNOPQRSTUVWXYZ\ddefloop
\def\ddef#1{\expandafter\def\csname c#1\endcsname{\ensuremath{\mathcal{#1}}}}
\ddefloop ABCDEFGHIJKLMNOPQRSTUVWXYZ\ddefloop
\def\ddef#1{\expandafter\def\csname h#1\endcsname{\ensuremath{\widehat{#1}}}}
\ddefloop ABCDEFGHIJKLMNOPQRSTUVWXYZabcdefghijklmnopqrsuvwxyz\ddefloop    % Not defined for t. 
\def\ddef#1{\expandafter\def\csname hc#1\endcsname{\ensuremath{\widehat{\mathcal{#1}}}}}
\ddefloop ABCDEFGHIJKLMNOPQRSTUVWXYZ\ddefloop
\def\ddef#1{\expandafter\def\csname t#1\endcsname{\ensuremath{\widetilde{#1}}}}
\ddefloop ABCDEFGHIJKLMNOPQRSTUVWXYZ\ddefloop
\def\ddef#1{\expandafter\def\csname tc#1\endcsname{\ensuremath{\widetilde{\mathcal{#1}}}}}
\ddefloop ABCDEFGHIJKLMNOPQRSTUVWXYZ\ddefloop

\newcommand{\diag}{\textrm{diag}}

\usepackage{tikz}

\newcommand{\grad}{\nabla}

\newcommand{\proman}[1]{\prn*{\romannumeral #1}}

\newcommand{\overleq}[1]{\overset{ #1}{\leq{}}}
\newcommand{\overgeq}[1]{\overset{#1}{\geq{}}}
\newcommand{\overeq}[1]{\overset{#1}{=}}

\newcommand{\whj}{\wh j} 

\newcommand{\relu}[1]{\brk*{#1}_{+}}

\let\wt\undefined

\newcommand{\wt}[1]{\widetilde{#1}}

\newcommand{\tepsilon}{\tilde{\epsilon}}

\usepackage{amsmath,graphicx}

\newcommand{\convexclass}{\mathscr{D}_c}

\newcommand{\indicator}[1]{\mb{1}\crl*{#1}}

\newcommand{\fB}{f_{(\ref{eq:empfn_basic_cons})}}

\newcommand{\reals}{\mathbb{R}}

\newcommand{\w}{w}
\newcommand{\W}{\mathcal{W}}

\renewcommand{\epsilon}{\varepsilon}

\newcommand{\indic}{\mathbb{1}}
\newcommand{\fD}{f_{\eqref{eq:multipass-fn}}}

\newcommand{\wmp}{\wh w^{\text{MP}}}

\newcommand{\newn}{m}
\newcommand{\fullvariable}{w} 

\newcommand{\poly}{\text{poly}} 
\newcommand{\sign}[1]{\,\text{sign}\crl{#1}} 

\renewcommand{\relu}{\text{ReLU}}

\newcommand{\fA}{f_{(\ref{eq:empfn_basic_cons_sco})}}

\usepackage{xcolor} 

\usepackage{color-edits}
\addauthor{as}{red}
\addauthor{sk}{purple}
\addauthor{ks}{red}

\usepackage{mathrsfs} 

\usepackage{longtable}
\usepackage{import} 
\usepackage{enumitem} 
\usepackage{algorithm}
\usepackage{graphicx}

\usepackage[noend]{algpseudocode} 
\renewcommand{\algorithmicrequire}{\textbf{Input:}} 
\algblockdefx[class]{Class}{EndClass}[1]{\textbf{object} \textsc{#1}:}{\textbf{end object}}
\makeatletter 
\ifthenelse{\equal{\ALG@noend}{t}}%
  {\algtext*{EndClass}}
  {}%
\makeatother
\usepackage{amsmath}
\usepackage{makecell}
\usepackage{enumitem}  
\usepackage{bbm} 
\usepackage{bbold} 
\usepackage{xspace} 
\usepackage[scaled=.9]{helvet}
\usepackage{booktabs} 
\usepackage[export]{adjustbox}
\title{SGD: The Role of Implicit Regularization, Batch-size and Multiple-epochs} 

\author{Satyen Kale \\ 
Google Research, NY \\ 
        {\small\texttt{satyen@google.com}}\\ 
	  \and
	Ayush Sekhari\\ 
	Cornell University\\
                {\small\texttt{as3663@cornell.edu}}\\
       \and
        Karthik Sridharan\\ 
        Cornell University\\
        {\small\texttt{ks999@cornell.edu}}\\
} 
\date{} 

\begin{document} 

\maketitle 
 \begin{abstract}  
Multi-epoch, small-batch, Stochastic Gradient Descent (SGD) has been the method of choice for learning with large over-parameterized models. A popular theory for explaining why SGD works well in practice is that the algorithm has an implicit regularization that biases its output towards a good solution. Perhaps the theoretically most well understood learning setting for SGD is that of Stochastic Convex Optimization (SCO), where it is well known that SGD learns at a rate of $O(1/\sqrt{n})$, where $n$ is the number of samples. In this paper, we consider the problem of SCO and explore the role of implicit regularization, batch size and multiple epochs for SGD. Our main contributions are threefold: 
\begin{enumerate}
\item We show that for any regularizer, there is an SCO problem for which Regularized Empirical Risk Minimzation fails to learn. This automatically rules out any implicit regularization based explanation for the success of SGD.
\item We provide a separation between SGD and learning via Gradient Descent on empirical loss (GD) in terms of sample complexity. We show that there is an SCO problem such that GD with any step size and number of iterations can only learn at a suboptimal rate: at least $\wt{\Omega}(1/n^{5/12})$.
\item We present a multi-epoch variant of SGD commonly used in practice. We prove that this algorithm is at least as good as single pass SGD in the worst case. However, for certain SCO problems, taking multiple passes over the dataset can significantly outperform single pass SGD. 
\end{enumerate}
We extend our results to the general learning setting by showing a problem which is learnable for any data distribution, and for this problem, SGD is strictly better than RERM for any regularization function. We conclude by discussing the implications of our results for deep learning, and show a separation between SGD and ERM for two layer diagonal neural networks. 
\end{abstract} 

% Main body 
\section{Introduction} \label{sec:introduction} 
We consider the problem of  stochastic optimization of the form: 
\begin{align}
\mathrm{Minimize}~~ F(w)
 \end{align}
where the objective $F: \reals^d \mapsto \reals$ is given by $F(w) = \En_{z \sim D} \brk*{f(w; z)}$.
The goal is to perform the minimization based only on samples $S = \crl{z_1,\ldots,z_n}$ drawn i.i.d.\ from some distribution $\cD$. Standard statistical learning problems can be cast as  stochastic optimization problems, with $F(w)$ being the population loss and $f(w,z)$ being the instantaneous loss on sample $z$ for the model $\w$. 

\paragraph{Stochastic Convex Optimization (SCO).}Perhaps the most well studied stochastic optimization problem is that of SCO.
We define a problem to be an instance of a SCO problem if, 
 \begin{align}
\text{\textbf{Assumption I}: Population loss $F$ is convex}   \label{eq:ass1}. 
\end{align}
Notice above that we only require the population loss to be convex and do not impose such a condition on the instantaneous loss functions $f$. 

Algorithms like Stochastic Gradient Descent (SGD), Gradient Descent on training loss (GD), and methods like Regularized Empirical Risk Minimization (RERM) that minimize training loss with additional penalty in the form of a regularizer are all popular choices of algorithms used to solve the above problem and have been analyzed theoretically for various settings of Stochastic Optimization problems (convex and non-convex).  We discuss below a mix of recent empirical and theoretical insights about SGD algorithms that provide motivation for this work. 

\paragraph{SGD and Implicit Regularization.}
A popular theory for why SGD generalizes so well when used on large over-parameterized models has been that of implicit regularization. It has been oberved that in large models, often there are multiple global minima for the empirical loss. However not all of these empirical minima have low suboptimality at the population level. SGD when used as the training algorithm often seems to find empirical (near) global minima that also generalize well and have low test loss. Hence while a general Empirical Risk Minimization (ERM) algorithm might fail, the implicit bias of SGD seems to yield a well-generalizing ERM. The idea behind implicit regularization is that the solution of SGD is equivalent to the solution of a Regularized Empirical Risk Minimizer (RERM) for an appropriate implicit regularizer. 

The idea of implicit regularization of SGD has been extensively studied in recent years. In \citet{gunasekar2018characterizing}, the classical setting of linear regression (with square loss) is considered and it was shown that when considering over-parameterized setting, the SGD algorithm is equivalent to fitting with a linear predictor with the smallest euclidean norm. In \citet{soudry2018implicit, ji2018risk} linear predictors with logistic loss are considered and it was noted that SGD for this setting can be seen as having an implicit regularization of $\ls_2$ norm. \citet{gunasekar2018implicitb} considered multi-layer convolutional networks with linear activation are considered and showned that SGD for this model can be seen as having an implicit regularization of $\ls_{2/L}$ norm (bridge penality for depth $L$ networks) of the Fourier frequencies corresponding to the linear predictor. \citet{gunasekar2018implicit} considered matrix factorization and showed that running SGD is equivalent to having a nuclear norm based regularizer. More recent work of \citet{arora2019implicit, razin2020implicit} shows that in particular in the deep matrix factorization setting, SGD cannot be seen as having any norm based implicit regularizer but rather a rank based one. However, in all these cases, the behavior of SGD corresponds to regularization functions that are independent of the training data (e.g. rank, $l_p$-norm, etc). 

One could surmise that a grand program for this line of research is that for problems where SGD works well, perhaps there is a corresponding implicit regularization explanation. That is, there exists a regularizer $R$ such that SGD can been seen as performing exact or approximate RERM with respect to this regularizer. In fact, one can ask this question specific to  SCO problems. That is, for the problem of SCO, is there an implicit regularizer $R$ such that SGD can be seen as performing approximate RERM? In fact, a more basic question one can ask: is it true that SCO problem is always learnable using some regularized ERM algorithms? \emph{We answer both these questions in the negative.} 

\paragraph{SGD vs GD: Smaller the Batch-size Better the Generalization.} It has been observed that in practice, while SGD with small batch size, and performing gradient descent (GD) with empirical loss as the objective function both minimize the training error equally well, the SGD solution generalizes much better than the full gradient descent one \citep{keskar2016large, kleinberg2018alternative}. However, thus far, most existing literature on theorizing why SGD works well for over-parameterized deep learning models also work for gradient descent on training loss \citep{allen2019can, allen2018learning, allen2018convergence, arora2018optimization}. In this work, we construct a convex learning problem where single pass SGD provably outperforms GD run on empirical loss, which converges to a solution that has large excess risk. We thus provide \emph{a problem instance where SGD works but GD has strictly inferior sample complexity (with or without early stopping)}. 

\paragraph{Multiple Epochs Help.}  The final empirical observation we consider is the fact that multiple epochs of SGD tends to further continually decrease not only the training error but also the test error \citep{zhang2016understanding, ma2018power, bottou201113}. In this paper, we construct \emph{an SCO instance for which multiple epochs of single sample mini-batched SGD significantly outperforms single pass SGD}, and for the same problem, \emph{RERM fails to converge to a good solution}, and hence, so does GD when run to convergence.

\subsection{Our Contributions} 
We now summarize our main contributions in the paper:
\begin{enumerate} [leftmargin=5mm] 
\item  {\bf SGD and RERM Separation.} In \pref{sec:SGDimp}, we demonstrate a SCO problem where a single pass of SGD over $n$ data points obtains a $1/\sqrt{n}$ suboptimality rate. However, for any regularizer $R$, regularized ERM does not attain a diminishing suboptimality. We show that this is true even if the regularization parameter is chosen in a sample dependent fashion. Our result immediately rules out the explanation that SGD is successful in SCO due to some some implicit regularization. 
\item  {\bf SGD and GD Separation.} In \pref{sec:GD_fails}, we provide a separation between SGD and GD on training loss in terms of sample complexity for SCO. To the best of our knowledge, this is the first\footnote{Despite the way the result is phrased in \cite{amir2020gd}, their result does not imply separation between SGD and GD in terms of sample complexity but only number of iterations.} such separation result between SGD and GD in terms of sample complexity. In this work, we show the existence of SCO problems where SGD with $n$ samples achieves a suboptimality of $1/\sqrt{n}$, however, irrespective of what step-size is used and how many iterations we run for, GD cannot obtain a suboptimality better than $1/(n^{5/12} \log^2(n) )$. 
\item  {\bf Single-pass vs Multi-pass SGD.} In \pref{sec:multi_pass}, we provide an adaptive multi-epoch SGD algorithm that is provably at least as good as single pass SGD algorithm. On the other hand, we also show that this algorithm can far outperform single pass SGD on certain problems where SGD only attains a rate of $1/\sqrt{n}$. Also, on these problems RERM fails to learn, indicating that GD run to convergence fails as well.
\item {\bf SGD and RERM Separation in the Distribution Free Agnostic PAC Setting.} 
The separation result between SGD and RERM introduced earlier was for SCO. However, it turns out that the problem is not agnostically learnable for all distributions but only for distributions that make $F$ convex. In \pref{sec:dist_free_rerm}, we provide a learning problem that is distribution-free learnable, and where SGD provably outperforms any RERM. 
\item {\bf Beyond Convexity (Deep Learning).} The convergence guarantee for SGD can be easily extended to stochastic optimization settings where the population loss $F(w)$ is \textit{linearizable}, but may not be convex. We formalize this in \pref{sec:nn_connections}, and show that for two layer diagonal neural networks with $\relu$ activations, there exists a distribution for which the population loss is \textit{linearizable} and thus SGD works, but ERM algorithm fails to find a good solution. This hints at the possibility that the above listed separations between SGD and GD / RERM for the SCO setting, also extend to the deep learning setting. 
\end{enumerate} 

\subsection{Preliminaries}
A standard assumption made by most gradient based algorithms for stochastic optimization problems is the following: 
\begin{align} 
\textbf{Assumption II: }\begin{cases} & \text{$F$ is $L$-Lipschitz (w.r.t. $\ell_2$ norm)} \\ 
& \exists \w^\star \in \argmin_{\w} F(w) \textrm{ s.t. }\|w_1 - \w^\star\| \le B \\ 
&\sup_{\w} \En_{z \sim D} \|\nabla f(w,z) - \nabla F(w)\|^2  \le \sigma^2 \end{cases}. \label{eq:ass2} 
\end{align} 
In the above and throughout this work, the norm denotes the standard Euclidean norm and $\w_1$ is some initial point known to the algorithm. Next, we describe below more formally what the regularized ERM, GD and SGD algorithms are. 
\paragraph{(Regularized) Empirical Risk Minimization.} Perhaps the simplest algorithm one could consider is the Empirical Risk Minimization (ERM) algorithm where one returns a minimizer of training loss (empirical risk) 
$$F_S(\w) \ldef{} \frac{1}{n} \sum_{t=1}^n f(\w;z_t)~.$$ That is, 
$\w_{\mathrm{ERM}} \in \argmin_{\w \in \W} F_S(\w)$. A more common variant of this method is one where we additionally penalize complex models using a regularizer function $R: \reals^d \mapsto \reals$. That is, a Regularized Empirical Risk Minimization (RERM) method consists of returning:
\begin{align}  \label{eq:RERM_def} 
\w_{\mathrm{RERM}} = \argmin_{\w \in \W} F_S(\w)  + R(\w). 
\end{align}

\paragraph{Gradient Descent (GD) on Training Loss.} Gradient descent on training loss is the algorithm that performs the following update on every iteration:
\begin{align}\label{eq:GD}
\w^{\mathrm{GD}}_{i+1} \leftarrow \w^{\mathrm{GD}}_{i} - \eta \nabla F_S(\w^{\mathrm{GD}}_{i}), 
\end{align}
where $\eta$ denotes the step size. After $t$ rounds, we return the point $\widehat{\w}^{\mathrm{GD}}_{t} \ldef{} \frac{1}{t} \sum_{i=1}^t\w^{\mathrm{GD}}_{i}$. 

\paragraph{Stochastic Gradient Descent (SGD).} Stochastic gradient descent (SGD) has been the method of choice for training large over-parameterized deep learning models, and other convex and non-convex learning problems. Single pass SGD algorithm runs for $n$ steps and for each step takes a single data point and performs gradient update with respect to that sample. That, is on round $t$,
\begin{align}\label{eq:SGD}
\w^{\mathrm{SGD}}_{i+1} \leftarrow \w^{\mathrm{SGD}}_{i} - \eta \nabla f(\w^{\mathrm{SGD}}_i;z_i) 
\end{align} 
Finally, we return $\widehat{\w}^{\mathrm{SGD}}_n = \frac{1}{n} \sum_{i=1}^n \w^{\mathrm{SGD}}_{i}$. Multi-epoch (also known as multi-pass) SGD algorithm simply cycles over the dataset multiple times continuing to perform the same update specified above. It is well know that single pass SGD algorithm enjoys the following convergence guarantee: 
\begin{theorem}[\citet{nemirovskij1983problem}] \label{thm:sgd_works_main} On any SCO problem satisfying Assumption I in \pref{eq:ass1} and Assumption II in \pref{eq:ass2}, running SGD algorithm for $n$ steps with the step size of $\eta = 1/\sqrt{n}$ enjoys the guarantee 
$$ 
\En_{S}[F(\widehat{\w}^{\mathrm{SGD}}_n)] - \inf_{\w \in \reals^d} F(\w) \le O\prn[\Big]{\frac{1}{\sqrt{n}}}, 
$$ 
where the constant in the order notation only depends on constants $B, L$ and $\sigma$ in \pref{eq:ass2} and is independent of the dimension d. 
\end{theorem}

In fact, even weaker assumptions like one-point convexity or star convexity of $F$ w.r.t. an optimum on the path of the SGD suffices to obtain the above guarantee. We use this guarantee of SGD algorithm throughout this work. Up until \pref{sec:dist_free_rerm}  we only consider SCO problems for which SGD automatically works with the above guarantee.

\section{Related Work} 
On the topic of RERM, implicit regularization and SGD, perhaps the work most relevant to this paper is that of \cite{dauber2020can}. Just like this work, they also consider the general setting of stochastic convex optimization (SCO) and show that for this setting, for no so-called \emph{``admissible''} data independent regularizer, SGD can be seen as performing implicit regularized ERM. For the implicit regularization part of our work, while in spirit the work aims at accomplishing some of the similar goals, their work is in the setting where the instantaneous losses are also convex and Lipschitz. This means that, for their setting, while SGD and regularized ERM may not coincide, regularized ERM is indeed still an optimal algorithm as shown in \cite{shalev2009stochastic}. In this work we show a strict separation between SGD and RERM with an example where SGD works but no RERM can possibly provide a non-trivial learning guarantee. Second, qualitatively, their separation result in Theorems 2 and 3 are somewhat unsatisfactory. This is because while they show that for every regularizer there is a distribution for which the SGD solution is larger in value than the regularized training loss of the regularized ERM w.r.t. that regularizer, the amount by which it is larger can depend on the regularizer and can be vanishingly small. Hence it might very well be that SGD is an approximate RERM. In fact, if one relaxes the requirement of their admissible relaxations then it is possible that their result doesn't hold. For instance, for square norm regularizer, the gap on regularized objective between RERM and SGD is only shown to be as large as the regularization parameter which is typically set to be a diminishing function of $n$. 

On the topic of comparison of GD on training loss with single epoch SGD, one can hope for three kinds of separation. First, separation in terms of work (number of gradient computations), second, separation in terms of number of iterations,  and finally separation in terms of sample complexity. A classic result from \cite{nemirovskij1983problem} tells us that for the SCO setting, the number of gradient computations required for SGD to obtain a suboptimality guarantee of $\epsilon$ is equal to the optimal sample complexity (for $\epsilon$) and hence is optimal (in the worst case). On the other hand, a single iteration of GD on training loss requires the same number of gradient computations and hence any more than a constant number of iterations of GD already gives a separation between GD and SGD in terms of work. This result has been explored in \cite{ShalevSiSr07} for instance. The separation of GD and SGD in terms of number of iterations has been considered in \cite{amir2020gd}. \citet{amir2020gd} demonstrate a concrete SCO problem on which GD on the training loss requires at least $\Omega(1/\epsilon^4)$ steps to obtain an $\epsilon$-suboptimal solution for the test loss. Whereas, SGD only requires $O(1/\epsilon^2)$ iterations. However, their result does not provide any separation between GD and SGD in terms of sample complexity. Indeed, in their example, using the upper bound for GD using \cite{bassily2020stability} one can see that if GD is run on $n$ samples for $T = n^{2}$ iterations with the appropriate step size, then it does achieve a $1/\sqrt{n}$ suboptimality. In comparison our work provides a much stronger separation between GD and SGD. We show a sample complexity separation, meaning that to obtain a specific suboptimality, GD requires more samples than SGD, irrespective of how many iterations we run it for. Our separation result also yields separation in terms of both number of iterations and number of gradient computations.

\section{Regularized ERM, Implicit Regularization and SGD} \label{sec:SGDimp}
%!TEX root=../paper.tex
In \citet{shalev2009stochastic} (see also \cite{feldman2016generalization}), SCO problems where not just the population loss $F$ but where also for each $z \in \cZ$, the instantaneous loss $f(\cdot,z)$ is convex is considered. In this setting, the authors show that the appropriate $\ell_2$ norm square regularized RERM always obtains the optimal rate of $1/\sqrt{n}$. However, for SGD to obtain a $1/\sqrt{n}$ guarantee, one only needs convexity at the population level.   In this section, based on construction of SCO problems that are convex only at population level and not at the empirical level, we show a strict separation between SGD and RERM. Specifically, while SGD is always successful for any SCO problem we consider here, we show that for any regularizer $R$, there is an instance of an SCO problem for which RERM with respect to this regularizer has suboptimality lower bounded by a constant. 

\begin{theorem} 
\label{thm:lowersco} 
For any regularizer $R$, there exists an instance of a SCO problem that satisfies both Assumptions I and II given in Equations (\ref{eq:ass1}) and (\ref{eq:ass2}), for which  
$$ 
\En_{S}[F(\w_{\mathrm{RERM}})] - \inf_{\w \in \reals^d} F(\w) \ge \Omega(1), 
$$
where $\w_{\mathrm{RERM}}$ is the solution to  \pref{eq:RERM_def} with respect to the prescribed regularizer.  
\end{theorem} 

The regularizer $R$ we consider in the above result is sample independent, that is, it has to be chosen before receiving any samples.  In general, if one is allowed an arbitrary sample dependent regularizer, one can encode any learning algorithm as an RERM. This is because, for any algorithm, one can simply contrive the regularizer $R$ to have its minimum at the output model of the algorithm on the given sample, and a very high penalty on other models. Hence, to have a meaningful comparison between SGD (or for that matter any algorithm) and RERM, one can either consider sample independent regularizers or at the very least, consider only some specific restricted family of sample dependent regularizers. One natural variant of considering mildly sample dependent regularizers is to first, in a sample independent way pick some regularization function $R$ and then allow arbitrary sample dependent regularization parameters that multiply the regularizer $R$. The following corollary shows that even with such mildly data dependent regularizers, one can still find SCO instances for which the RERM solution has no non-trivial convergence guarantees.

\begin{corollary}  \label{corr:lowersco_lambda} For any regularizer $R$, there exists an instance of a SCO problem that satisfies both Assumptions I and II, for which the point $\w_{\mathrm{RERM}} = \argmin_{\w \in \W} F_S(\w)  + \lambda R(\w), $ where $\lambda$ is any arbitrary sample dependent regularization parameter, has the lower bound 
$$ 
\En_{S}[F(\w_{\mathrm{RERM}})] - \inf_{\w \in \reals^d} F(\w) \ge \Omega(1). 
$$
\end{corollary}
 
The SCO problem in \pref{thm:lowersco} and \pref{corr:lowersco_lambda} is based on the function $\fA$ given by:  
\begin{align} 
  \fA(\w; z) = y \nrm*{(\w - \alpha) \odot x} \tag{A}, \label{eq:empfn_basic_cons_sco} 
\end{align}  where each instance $z \in \cZ$ can be written as a triplet $z = (x,y,\alpha)$ and the notation $\odot$ denotes Hadamard product (entry wise product) of the two vectors. We set $x \in \{0,1\}^d$, $y \in \{\pm1\}$ and $\alpha \in \{0,e_1,\ldots,e_d\}$ where $e_1$ to $e_d$ denote the standard basis in $d$ dimensions. We also set $d > 2^n$.

In the following, we provide a sketch for why $\ell_2$-norm square regularization fails and show no regularization parameter works. In the detailed proof provided in the Appendix, we deal with arbitrary regularizers. The basic idea behind the proof is simple, Consider the distribution: $x \sim \mathrm{Unif}(\{0,1\})^d$, and $y$ is set to be $+1$ with probability $0.6$ and $-1$ with probability $0.4$, and set $\alpha = e_1$ deterministically. In this case, note that the population function is $0.2 \En_{x\sim \mathrm{Unif}\{0,1\}^d}\|x \odot (\w - e_1)\|_2$ which is indeed convex, Lipchitz and sandwiched between $0.1 \|\w - e_1\|$ and $0.2\|\w - e_1\|$. Hence, any $\epsilon$ sub-optimal solution $\wh \w$ must satisfy $\|\wh \w - e_1\| \le 10 \epsilon$.

Since $d > 2^n$, with constant probability there is at least one coordinate, say $\wh {j} \in [d]$, such that for any data sample $(x_t, y_t)$ for $t \in [n]$, we have $x_t[\wh {j}]=0$ whenever $y_t = +1$ and $x_t[\wh {j}] = 1$ whenever $y_t = -1$. Hence, ERM would simply put large weight on the $\wh j$ coordinate and attain a large negative value for training loss (as an example, $\w = e_{\wh {j}}$ has a training loss of roughly $-0.4$). However, for test loss to be small, we need the algorithm to put little weight on coordinate 
$\wh {j}$. Now for any square norm regularizer $R(\w) = \lambda \|\w\|_2^2$, to prevent RERM from making this $\wh {j}$ coordinate large, $\lambda$ has to be at least as large as a constant. 

However, we already argued that any $\epsilon$ sub-optimal solution $\wh {\w}$ should be such that $\|\wh {\w} - e_1\| \le 10 \epsilon$. When $\lambda$ is chosen to be as large as a constant, the regularizer will bias the solution towards $0$ and thus the returned solution will never satisfy the inequality $\|\wh {\w} - e_1\| \le 10 \epsilon$. This leads to a contradiction: To find an $\epsilon$ sub-optimal solution, we need a regularization that would avoid picking large value on the spurious coordinate $\wh {j}$ but on the other hand, any such strong regularization, will not allow RERM to pick a solution that is close enough to $e_1$. Hence, any such regularized ERM cannot find an $\epsilon$ sub-optimal solution. This shows that no regularization parameter works for norm square regularization. In the Appendix we expand this idea for arbitrary regularizers $R$. 

The construction we use in this section has dimensionality that is exponential in number of samples $n$. However similarly modifying the construction in \cite{feldman2016generalization}, using the $y = \pm1$ variable multiplying the construction there, one can extend this result to a case where dimensionality is $\Theta(n)$. Alternatively, noting that for the maximum deviation from mean over $d$ coordinates is of order $\sqrt{\log d/n}$, one can use a simple modification of the exact construction here, and instead of getting the strong separation like the one above where RERM does not learn but SGD does, one can instead have $d = n^{p}$ for $p \in (0,1]$ and obtain a separation where SGD obtains the $1/\sqrt{n}$ rate but no RERM can beat a rate of order $\sqrt{\log n/n}$.

\paragraph{Implicit Regularization.}  As mentioned earlier, a proposed theory for why SGD algorithms are so successful is that they are finding some implicitly regularized empirical risk minimizers, and that this implicit bias helps them learn effectively with low generalization error. However, at least for SCO problems, the above strict separation result tells us that SGD cannot be seen as performing implicitly regularized ERM, neither exactly nor approximately.  

\section{Gradient Descent (Large-batch) vs SGD (Small-batch)} \label{sec:GD_fails} 
\newcommand{\wgd}{\wh w^{\text{GD}}} 

In the previous section, we provided an instance of a SCO problem on which no regularized ERM works as well as SGD. When gradient descent algorithm (specified in  \pref{eq:GD}) is used for training, one would expect that GD, and in general large batch SGD, will also eventually converge to an ERM and so after enough iterations would also fail to find a good solution. In this section, we formalize this intuition to provide lower bounds on the performance of GD. 

\begin{theorem} 
\label{thm:GD_vs_SGD_comparison} 
There exists an instance of a SCO problem such that for any choice of step size $\eta$ and number of iterations $T$, the following lower bound on performance of GD holds:
\begin{align*} 
\En_{S}[F(\wh w^{\text{GD}}_T)]  - \inf_{\w \in \reals^d} F(\w) \geq \Omega\prn[\Big]{\frac{1}{n^{5/12}}}. 
\end{align*} 
\end{theorem} 

\pref{thm:GD_vs_SGD_comparison} suggests that there is an instance of a SCO problem for which the performance of GD is lower bound by $1/n^{0.42}$. On the other hand, SGD with the step size of $\eta = 1/ n^{0.5}$ learns at a rate of $1/n^{0.5}$ for this problem. This suggests that GD (large batch size) is a worse learning algorithm than SGD. We defer the proof to \pref{app:gd_fails} and give a brief sketch below. 

Our lower bound proof builds on the recent works of  \cite{amir2020gd}, which gives an instance of a SCO problem for which the performance guarantee of GD algorithm is $\Omega \prn{\eta \sqrt{T} + 1 / \eta T}$. We provide an instance of a SCO problem for which GD has a lower bound of $\Omega \prn{\eta T /n}$. Adding the two instances together gives us an SCO problem for which GD has lower bound of 
\begin{align*}
	\Omega\prn[\Big]{ \eta \sqrt{T} + \frac{1}{\eta T} + \frac{\eta T}{n}}. \numberthis \label{eq:our_GD_lb_sum} 
\end{align*}
This lower bound is by itself not sufficient. In particular, for $\eta = 1 / n^{3/2}$ and $T = n^2$, the right hand side  of \pref{eq:our_GD_lb_sum} evaluates to $O(1/\sqrt{n})$, which matches the performance guarantee of SGD for this SCO problem. Hence, this leaves open the possibility that GD may work as well as SGD. However, note that in order to match the performance guarantee of SGD, GD needs to be run for quadratically more number of steps and with a smaller step size. In fact, we can show that  in order to attain a $O(1/\sqrt{n})$ rate in \pref{eq:our_GD_lb_sum}, we must set $\eta = O(1/n^{3/2})$ and $T = \omega(n^2)$. 

The lower bound in \pref{thm:GD_vs_SGD_comparison} follows by adding to the SCO instance in \pref{eq:our_GD_lb_sum} another objective that rules out small step-sizes. This additional objective is added by increasing the dimensionality of the problem by one and on this extra coordinate adding a stochastic function that is convex in expectation. The expected loss on this coordinate is a piecewise linear convex function. However, the stochastic component on this coordinate has random kinks at intervals of width $1/n^{5/4}$ that vanish in expectation, but can make the empirical loss point in the opposite direction with probability $1/2$. Since the problem is still an SCO problem, SGD works as earlier. On the other hand, when one considers training loss, there are up to $n$ of these kinks and roughly half of them make the training loss flat. Thus, if GD is used on training loss with step size smaller than $1/n^{5/4}$ then it is very likely that GD hits at least one such kink and will get stuck there. This  function, thus, rules out GD with step size $\eta$ smaller than $1/ n^{5/4}$. Restricting the step size $\eta = \Omega\prn{1/ n^{5/4}}$, the lower bound in \pref{eq:our_GD_lb_sum} evaluates to the $\Omega(1 / n^{5/12})$ giving us the result of \pref{thm:GD_vs_SGD_comparison}.

Our lower bound in \pref{eq:our_GD_lb_sum} matches the recently shown performance guarantee for GD algorithm by \cite{bassily2020stability};  albeit under slightly different assumptions on the loss function $f(\cdot; z)$. Their work assumes that the loss functions $f(w; z)$ is convex in $w$ for every $z \in \cZ$. On the other hand, we do not require convexity of $f$ but only that $F(w)$ is convex in $w$ (see our Assumptions I and II).

\section{Single-pass vs Multi-pass SGD} \label{sec:multi_pass}  
%!TEX root=../paper.tex
State of the art neural networks are trained by taking multiple passes of SGD over the dataset. However, it is not well understood when and why multiple passes help. In this section, we provide theoretical insights into the benefits of taking multiple passes over the dataset. The multi-pass SGD algorithm that we consider is: 
\begin{enumerate}
\setlength\itemsep{-0.1em} 
\item Split the dataset $S$ into two equal sized datasets $S_1$ and $S_2$. 
\item Run $k$ passes of SGD algorithm using a fixed ordering of the samples in $S_1$ where, 
\begin{enumerate}[label=$\bullet$, leftmargin=5mm]
\item The step size for the $j$th pass is set as $\eta_j = 1 / \sqrt{nj}$. 
\item At the end of $j$th pass, compute $\wh \w_j = \frac{2}{n j} \sum_{t=1}^{nj/2} w^{\mathrm{SGD}}_j$ as the average of all the iterates generated so far. 
\end{enumerate} 

\item Output the point $\wh \w^{\mathrm{MP}} \ldef{} \argmin_{\w \in \wh \cW} F_{S_2}(\w)$ where $\wh \cW \ldef{} \crl*{\wh \w_1, \ldots, \wh \w_k}$. 
\end{enumerate}  
The complete pseudocode is given in the Appendix. In the following, we show that the above multi-pass SGD algorithm performs at least as well as taking a single pass of SGD. 

\begin{proposition} 
\label{prop:multi_pass_works}
The output $\wh \w^{\mathrm{MP}}$ of multipass-SGD algorithm satisfies 
$$ 
\En_{S}[ F(\wh w^{\mathrm{MP}})] \leq \En_{S}[F(\wh w^{\mathrm{SGD}}_{n/2})] + \wt O\prn[\Big]{\frac{1}{\sqrt{n}}}, 
$$
where $\wh w^{\mathrm{SGD}}_{n/2}$ denotes the output of running (one pass) SGD algorithm for $n/2$ steps with the step size of $1/\sqrt{n}$. 
\end{proposition}  

This suggests that the output point of the above multi-pass SGD algorithm is not too much worse than that of SGD (single pass). For problems in SCO, SGD has a rate of $O(1/\sqrt{n})$  (see \pref{thm:sgd_works_main}), and in for these problems, the above bound implies that multi-pass SGD also enjoys the rate of $\wt{O}(1/\sqrt{n})$. 

\subsection{Multiple Passes Can Help!} 
In certain favorable situations the output of multi-pass SGD can be much better:  
\begin{theorem} \label{thm:multipass} Let $k$ be a positive integer and let $R(\cdot)$ be a regularization function. There exists an instance of a SCO problem such that: 

\begin{enumerate}[label=(\alph*)] 
\item For any step size $\eta$, the output of the SGD algorithm has the lower bound 
\begin{align*}  
\En_{S}[F(\wh w^{\text{SGD}}_n)]  - \inf_{\w \in \reals^d} F(\w) \geq \Omega\prn[\Big]{\frac{1}{\sqrt{n}}}. 
\end{align*}
Furthermore, running SGD with $\eta = 1/\sqrt{n}$ achieves the above $1/\sqrt{n}$ rate. 
\item On the other hand, multi-pass SGD algorithm with $k$ passes has the following guarantee: 
\begin{align*} 
\En_{S}[F( \wh \w^{\mathrm{MP}} )]  - \inf_{\w \in \reals^d} F(\w) \leq O\prn[\Big]{\frac{1}{\sqrt{nk}}}. 
\end{align*}
\item RERM algorithm has the lower bound: $ 
\En_{S}[F(\w_{\mathrm{RERM}})] - \inf_{\w \in \reals^d} F(\w) \ge \Omega(1).$
\end{enumerate} 
\end{theorem} 

We defer the proof details to \pref{app:multi_pass} and provide the intuition below. Parts (a) and (b) follow easily by taking a standard construction in \citep[Section 3.2.1]{nesterov}. Here, a convex deterministic function $F_N$ over an optimization variable $v$ is provided such that Assumption II is satisfied with $L, B = O(1)$, and any gradient based scheme (such as SGD, which is equivalent to GD on $F_N$) has a suboptimality of $\Omega(1/\sqrt{n})$ after $n$ iterations (which corresponds to single pass SGD). On the other hand, because $F_N$ satisfies Assumption II, $k$ passes of SGD, which correspond to $nk$ iterations of GD, will result in suboptimality of $O(1/\sqrt{nk})$ for the given step sizes.

The more challenging part is to prove part (c). It is tempting to simply add the SCO instance from \pref{thm:lowersco}, but that may break the required upper bound of $O(1/\sqrt{nk})$ for multipass SGD. To get around this issue, we construct our SCO instance by making $z$ consist of $k$ components $\crl{\xi_1, \ldots, \xi_k}$ drawn independently from the distribution considered in \pref{thm:lowersco}. Furthermore, the loss function considered in \pref{thm:lowersco} also defines the loss corresponding to each $\xi_i$ component for an optimization variable $w$ that is different from the optimization variable $v$.

The key idea in the construction is that, while each data sample consists of $k$ independently sampled components, at any time step, SGD gets to see only one of these components. Specifically, in the first pass over the dataset, we only observe the $\xi_1$ component of every sample, and in the second pass we only observe the  $\xi_2$ component for every sample and so on for further passes. This behavior for SGD is induced by using an independent control variable $u$. The optimization objective for $u$ is such that SGD increases $u$ monotonically with every iteration. The value of $u$ controls which of the $\xi_i$ components is observed during that time step. In particular, when we run our multipass SGD algorithm, during the first pass, the value of $u$ is such that we get to see $\xi_1$ only. However, on the second pass, $u$ has become large enough to reveal $\xi_s2$, and so on for further passes. Thus, in $k$ passes, multipass SGD gets to see $nk$ ``fresh'' samples and hence achieves an suboptimality of $O(1/\sqrt{nk})$, as required in part (b). Finally, part (c) follows from the same reasoning as in \pref{thm:lowersco} since the same SCO instance is used.

\section{Distribution Free Learning} \label{sec:dist_free_rerm} 

%!TEX root=../paper.tex
In this section, we consider the general learning setting \citep{vapnik2013nature} and aim for a distribution free learnability result. That is, we would like to provide problem settings where one has suboptimality that diminishes with $n$ for any distribution over the instance space $\cZ$. For the SCO setting we considered earlier, the assumptions in \pref{eq:ass1} and the assumption in \pref{eq:ass2} that the population loss is convex, imposes restrictions on the distributions allowed to be considered. Specifically, since $f(\cdot ~; z)$ need not be convex for every $z$, one can easily construct distributions for which  the problem is not SCO and in fact, may not even be learnable. E.g., consider the distribution that puts all its mass on $z$ for which $f(\cdot ~; z)$ is non-convex. In this section we will consider problems that are so called learnable, meaning that there is a learning algorithm with diminishing in $n$ suboptimality for any distribution $\cD$ on the instance space $\cZ$. Under this setting, we show a separation between SGD and RERM. Specifically, we provide a problem instance that is learnable at a rate of $c_n > \Omega(1/n^{1/4})$ for any distribution over the instance space $\cZ$. In particular, we show that the worst case rate of SGD is $c_n$ for this problem. However, on the subset of distributions for which the problem is SCO, SGD as expected obtains a rate of $1/\sqrt{n}$. On the other hand, for the same problem we show that RERM while having worst case rate no better than $c_n$, has a lower bound of $\Omega(c_n^2)$ on SCO instances. 

Our lower bound in this section is based on the following learning setting: For $d > 2^n$, let $\cZ = \{0,1\}^d \times \{0,e_1,\ldots,e_d\}$ be the instance space, let the instantaneous loss function be given by:  
\begin{align} 
	f_{(\ref{eq:empfn_basic_cons})}(w; z) = & \frac{1}{2} \nrm*{\prn*{w - \alpha} \odot x}^2 - \frac{c_n}{2} \nrm*{w - \alpha}^2 	+  \max\{1, \nrm*{w}^4\}, \tag{B} \label{eq:empfn_basic_cons}  
\end{align} 
where $z = (x,\alpha)$ and $c_n \ldef{}  n^{-(\frac{1}{4} - \gamma)}$ for some $\gamma > 0$. 

\par Additionally, let $\convexclass$ denote the set of distributions over instances space $\cZ$ for which $F$ is convex. Specifically, 
\begin{align}\label{eq:convdis} 
\convexclass \ldef{} \crl*{\cD \mid F(\w) = \En_{z \sim \cD} \brk*{\fB(\w, z)} \text{ is convex in $\w$}}. 
\end{align} 
Note that under the distributions from $\convexclass$, the problem is an instance of an SCO problem. We show a separation between SGD and RERM over the class $\convexclass$. In the theorem below we claim that the above learning problem is learnable with a worst case rate of order $c_n$ (by SGD), and whenever $\cD \in \convexclass$, SGD learns at a faster rate of $1/\sqrt{n}$. To begin with we first provide the following proposition for the problem described by the function $\fB$ that shows that no algorithm can have a worst case rate better than $c_n$ for this problem.

\begin{proposition} 
\label{prop:general_learning_lb} 
For any algorithm $\mathrm{ALG}$ that outputs $\widehat{\w}^\mathrm{ALG}$, there is a distribution $\cD$ on the instance space, such that for the learning problem specified by function $f_{(\ref{eq:empfn_basic_cons})}$:
$$\En_{S}[F(\widehat{\w}^\mathrm{ALG})] - \inf_{\w \in \W} F(\w) \geq \frac{c_n}{4}$$
\end{proposition}

Next, we show that SGD obtains this worst case rate, and a much better rate of $1/\sqrt{n}$ for any distribution in $\convexclass$. However, we also show that for this problem, no RERM can obtain a rate better than $c_n^2$ for every distribution in $\convexclass$. This shows that while SGD and RERM have the same worst case rate, SGD outperforms any RERM whenever the problem turns out to be convex.

\begin{theorem} \label{thm:sgd_learns} For the learning problem specified by function $f_{(\ref{eq:empfn_basic_cons})}$:
\begin{enumerate}[label=(\alph*)]
\item  For every regularizer $R$, there exists a distribution $\cD \in \convexclass$ such that,
$$
\En_{S}[F(\w_{\mathrm{RERM}})] - \inf_{\w \in \reals^d} F(\w) \ge \Omega(c_n^2).
$$ 
\item For the SGD algorithm we have the following upper bounds:  
\begin{align*}
\forall \cD \in \convexclass, ~ \En_{S}[F(\widehat{\w}^{\mathrm{SGD}}_n)] - \inf_{\w \in \reals^d} F(\w) \le O\left(\frac{1}{\sqrt{n}}\right), \\ \forall \cD \notin \convexclass, ~
\En_{S}[F(\widehat{\w}^{\mathrm{SGD}}_n)] - \inf_{\w \in \reals^d} F(\w) \le O\left(c_n\right).~~~~
\end{align*}
\end{enumerate}   
\end{theorem}
 
As an example, plugging in $c_n = n^{-\frac{1}{8}}$ implies that when $\cD \notin \convexclass$, the suboptimality of SGD is bounded by $O(1/n^{1/8})$, and when $\cD \in \convexclass$, the suboptimality of SGD is bounded by $O(1/\sqrt{n})$. However, for any RERM, there exists a distribution $D \in \convexclass$, on which the RERM has a suboptimality of $\Omega(1/n^{1/4})$ and the worst case rate of any RERM is also $n^{-1/8}$. This suggests that SGD is a superior algorithm to RERM for any regularizer $R$, even in the distribution free learning setting.

\section{$\alpha$-Linearizable Functions and Deep Learning}    \label{sec:nn_connections}  
%!TEX root=../paper.tex 
While the classic convergence proof for SGD is shown for SCO setting, a reader familiar with the proof technique will recognize that the same result also holds when we only assume that the population loss $F$ is star-convex or one-point-convex with respect to  any optimum $\w^*$, or in fact even if it is star-convex only on the path of SGD. The following definition of Linearizable population loss generalizes star-convexity and one-point-convexity. 
\begin{definition}[$\alpha$-Linearizable] A stochastic optimization problem with population loss $F(w)$ is  $\alpha$-Linearizable if there exists a $w^* \in \argmin F(w)$ such that for every point $w \in \bbR^d$, 
\begin{equation*}
		F(w) - F(w^*) \leq \alpha \tri*{\grad F(w), w - w^*}. 
\end{equation*}
\end{definition}  
For linearizable function, one can upper bound the suboptimality at any point $w$ by a linear function given by $\grad F(w)$. The convergence guarantee for SGD now follows by bounding the cumulative sum of this linear function using standard arguments, giving us the following performance guarantee for SGD. 
\begin{theorem} \label{thm:sgd_works_linearizable} On any $\alpha$-Lineariazable stochastic optimization problem satisfying Assumption II in \pref{eq:ass2}, running SGD algorithm for $n$ steps with the step size of $\eta = 1/\sqrt{n}$ enjoys the guarantee: 
$$ 
\En_{S}[F(\widehat{\w}^{\mathrm{SGD}}_n)] - \inf_{\w \in \reals^d} F(\w) \leq O\prn[\Big]{\frac{\alpha}{\sqrt{n}}}, 
$$ 
where the constant in the order notation only depends on constants $B, L$ and $\sigma$ in \pref{eq:ass2} and is independent of the dimension $d$.  
\end{theorem} 
The hypothesis that this phenomenon is what makes SGD successful for deep learning has been proposed and explored in various forms \citep{zhou2019sgd, kleinberg2018alternative}. On the other hand our lower bound results hold even for the simpler SCO setting. Of course, to claim such a separation of SGD with respect to GD or RERM in the deep learning setting, one would need to show that our lower bound constructions can be represented as deep neural networks with roughly the same dimensionality as the original problem. In fact, all the functions that we considered so far can be easily expressed by restricted deep neural networks (where some weights are fixed) with square activation functions as we show in \pref{app:NN_representations}. Although, it would be a stretch to claim that practical neural networks would look anything like our restricted neural network constructions; it still opens the possibility that the underlying phenomena we exploit to show these separations hold in practical deep learning setting. In the following, we give an example of a simple two layer neural network with ReLU activation function where SGD enjoys a rate of $1/\sqrt{n}$, but any ERM algorithm fails to find an $O(1)$-suboptimal solution. 

Let the input sample $(x, y)$ be such that $x \in \crl*{0, 1}^d$ and $y \in \crl{-1, 1}$. Given the weights $w = (w_1, w_2)$, where $w_1 \in \bbR^d$ and $w_2 \in \bbR^d$, we define a two layer ReLU neural network that on the input $x$ outputs  
 \begin{align*} 
	h(w; x) = \relu\prn{w^\top_2 \relu(w_1 \odot x)}. 
\end{align*} 
 This is a two layer neural network with the input layer having a diagonal structure and output layer being fully connected (hence the name diagonal neural network). Suppose the network is trained using the absolute loss, i.e. on data sample $z = (x, y)$ and for weights $w = (w_1, w_2)$, we use the loss
 \begin{equation}
 f(w; z) = \abs{y - h(w; z)} = \abs{y - \relu\prn{w^\top_2 \relu(w_1 \odot x)}}.  \label{eq:nn_loss_fn_defn} 	
 \end{equation}
\begin{theorem}[Two layer diagonal network] \label{thm:two_layer_failure} For the loss function given in \pref{eq:nn_loss_fn_defn}  using a two layer diagonal neural network, there exists a distribution $\cD$ over the instance space $\cZ$ such that: \begin{enumerate}[label=$(\alph*)$, ]
\item $F(w)$ is $1/2$-Linearizable, and thus $SGD$ run with step-size $1/\sqrt{n}$ has excess risk $O(1/\sqrt{n})$. 
\item For $d \geq 2^n$, with probability at least $0.9$, ERM algorithm fails to find an $O(1)$-suboptimal point. 
\end{enumerate} 
\end{theorem} 

\begin{remark} The result of \pref{thm:two_layer_failure} can be extended to diagonal two layer neural networks trained with linear loss $f(w; z) = y h(w; x)$, or with hinge loss $f(w; x) = \max\crl{0, 1 - y h(w; z)}$. 
\end{remark} 

While the above result shows that for a simple two layer neural network, SGD performs better than ERM, our construction requires the first layer to be diagonal. It is an interesting future research direction to explore whether a similar phenomena can be demonstrated in more practical network architectures, for eg. fully connected networks, convolutional neural networks (CNN), recurrent neural networks (RNN), etc. It would also be interesting to extend our lower bounds for GD algorithm from \pref{sec:GD_fails} to these network architectures. The key idea is that SGD only requires certain nice properties (eg. convex, Linearizable, etc) at the population level, which might fail to hold at the empirical level in large dimensional models; hence, batch algorithms like GD and RERM might fail.

\subsection*{Acknowledgements} 
We thank Dylan Foster, Roi Livni, Robert Kleinberg and Mehryar Mohri for helpful discussions. AS was an intern at Google Research, NY when a part of the work was performed. KS acknowledges support from NSF CAREER Award 1750575.  

{
	\setlength{\bibsep}{6pt}
	\bibliography{refs} 
}

% Appendix 
\newpage
\appendix 
\renewcommand{\contentsname}{Contents of Appendix} 
\tableofcontents  
\addtocontents{toc}{\protect\setcounter{tocdepth}{3}} 
\clearpage 

\setlength\parindent{0pt}
\setlength{\parskip}{0.25em} 

\section{Preliminaries} 
\subsection{Additional notation} 
For a vector $\w \in \bbR^d$, fo any $j \in [d]$, $\w[j]$ denotes the $j$-th coordinate of $w$, $\nrm{\w}$ denotes the Euclidean norm and $\nrm{\w}_\infty$ denotes the $\ls_\infty$ norm. For any two vectors $\w_1$ and $\w_2$,  $\tri{\w_1, \w_2}$ denotes their inner product, and $\w_1 \odot \w_2$ denotes the vector generated by taking the Hadamard product of $\w_1$ and $\w_2$, i.e. $(\w_1 \odot \w_2)[j] = \w_1[j] \w_2[j]$ for $j \in [d]$. We denote by $\mb{1}_{d}$ a $d$-dimensional vector of all $1$s, and the notation $\bbI_d$ denotes the identity matrix in $d$-dimensions. The notation $\cN(0, \sigma^2)$ denotes Gaussian distribution with variance $\sigma^2$, and $\cB(p)$ denotes the Bernoulli distribution with mean $p$. 

For a function $f: \bbR^d \times \bbR$, we denote the gradient of $f$ at the point $\w \in \bbR^d$ by $\grad f(\w) \in \bbR^d$. The function $f$ is said to be $L$-Lipschitz if $f(\w_1) - f(\w_2) \leq L \nrm*{\w_1 - \w_2}$ for all $\w_1, \w_2$.

\subsection{Basic algorithmic results}  \label{app:basic_algorithms} 
The following convergence guarantee for SGD algorithm is well known in stochastic convex optimization literature, and is included here for the sake of completeness.  

\begin{theorem}[\citet{nemirovskij1983problem}]  
\label{thm:app_SGD_convergence} 
Let $w \in \bbR^d$ and $z \in \cZ$. Given an initial point $w_1 \in \bbR^d$, loss function $f(w; z)$ and a distribution $\cD$ such that: 
\begin{enumerate}[label=(\alph*), leftmargin=8mm]
\item The population loss  $F(\w) = \En_{z \sim \cD} \brk*{f(\w; z)}$ is $L$-Lipschitz and convex in $w$. 
\item For any $w$, $\En_{z \sim \cD} \brk*{\nrm*{\grad f(w; z) -\grad F(w)}} \leq \sigma^2$. 
\item The initial point $w_1$ satisfies $\nrm*{w_1 - w^*} \leq B$ where $w^* \in \argmin F(w)$. 
\end{enumerate}
	
Further, let $S \sim \cD^n$. Then, the point $\wh w_n^{\mathrm{SGD}}$ obtained by running SGD algorithm given in \pref{eq:SGD}, with step size $\eta = 1 / \sqrt{n}$ for $n$ steps using the dataset $S$ satisfies
 \begin{align*}
\En \brk*{F(\wh w_n^{\mathrm{SGD}}) - F^* } &\leq \frac{1}{\sqrt{n}} \prn*{ \sigma^2 + L^2 + B^2},  
\end{align*} 
where $F^* \ldef{} \min_{w} F(w)$. 
\end{theorem}

\begin{proof} Let $\crl{w_t}_{t \geq 1}$ denote the sequence of iterates generated by the SGD algorithm. We note that for any time $t \geq 1$, 
\begin{align*} 
\nrm*{\w_{t +1} - \w^*}^2_2 &=  \nrm*{\w_{t +1} - \w_{t} + \w_{t}  - \w^*}^2_2 \\
					&= \nrm*{\w_{t +1} - \w_{t} }_2^2 + \nrm*{\w_{t}  - \w^*}^2_2 + 2 \tri*{  \w_{t +1} - \w_{t},  \w_{t}  - \w^*} \\ 
					&= \nrm*{ - \eta \nabla f(\w_{t}; z_t)}_2^2 + \nrm*{\w_{t} - \w^*}_2^2 + 2\tri*{  - \eta \nabla f(\w_{t}; z_t), \w_{t} - \w^* }, 
\end{align*} where the last line follows from plugging in the SGD update rule that  $ \w_{t + 1} = \w_{t} - \eta \nabla f(\w_{t}; z_t)$. 

Rearranging the terms in the above, we get that 
\begin{align*} 
\tri*{\nabla f(\w_{t}; z_t), \w_{t} - \w^* } &\leq \frac{\eta}{2} \nrm*{ \nabla f(\w_{t}; z_t)}_2^2  + \frac{1}{2 \eta} \prn[\big]{ \nrm*{\w_{t} - \w^*}_2^2 - \nrm*{\w_{t+1}- \w^*}_2^2}. 
\intertext{Taking expectation on both the sides, while conditioning on the point $\w_t$, implies that} 
\tri*{\nabla F(\w_{t}), \w_{t} - \w^* } &\leq \frac{\eta}{2} \En \brk*{\nrm*{ \nabla f(\w_{t}; z_t)}_2^2 } + \frac{1}{2\eta} \prn[\big]{ \nrm*{\w_{t} - \w^*}_2^2 - \nrm*{\w_{t+1}- \w^*}_2^2} \\
&\leq \eta \En \brk*{\nrm*{ \nabla f(\w_{t}; z_t) - \grad F(\w_t)}_2^2} + \eta \nrm*{\grad F(\w_t)}_2^2  \\ 
& \qquad \qquad + \frac{1}{2\eta} \prn[\big]{ \nrm*{\w_{t} - \w^*}_2^2 - \nrm*{\w_{t+1}- \w^*}_2^2} \\
&\leq \eta(\sigma^2 + L^2) + \frac{1}{2\eta} \prn[\big]{ \nrm*{\w_{t} - \w^*}_2^2 - \nrm*{\w_{t+1}- \w^*}_2^2},
\end{align*}
where the inequality in the second line is given by the fact that $(a - b)^2 \leq 2 a^2 + 2b^2$ and the last line follows from using Assumption II (see \pref{eq:ass2}) which implies that $F(\w)$ is $L$-Lipschitz in $\w$ and that $ \En \brk{\nrm{ \nabla f(w; z_t) - \grad F(w)}_2^2} \leq \sigma^2$ for any $w$. Next, using convexity of the function $F$, we have that $F(\w^*) \geq F(\w_{t}) - \tri*{\nabla F(\w_{t}), \w_{t} - \w^*}$.  Thus 
\begin{align*} 
F(\w_{t}) - F^* &\leq  \eta(\sigma^2 + L^2) + \frac{1}{ 2\eta} \prn[\big]{ \nrm*{\w_{t} - \w^*}_2^2 - \nrm*{\w_{t + 1} - \w^*}_2^2}. 
\end{align*} 

Telescoping the above for $t$ from $1$ to $n$, we get that 
\begin{align*}
\sum_{t=1}^{n}\prn*{ F(\w_{t}) - F^* } &\leq \eta n (\sigma^2 + L^2)  + \frac{1}{2\eta} \prn[\big]{ \nrm*{\w_1 - \w^*}_2^2 - \nrm*{\w_{n + 1} - \w^*}_2^2} \\ 
					 &\leq \eta n (\sigma^2 + L^2) + \frac{1}{2 \eta} { \nrm*{\w_1 - \w^*}_2^2}. 
					 \end{align*} 
					Dividing both the sides by $n$, we get that 
					 \begin{align*} 
\frac{1}{n} \sum_{t=1}^{n}\prn*{ F(\w_{t}) - F^* }  &\leq  \eta (\sigma^2 + L^2) + \frac{1}{2 \eta n} { \nrm*{\w_1 - \w^*}_2^2}. 
\end{align*} 

An application of Jensen's inequality on the left hand side, implies that the point $\wh \w_n \ldef{} \frac{1}{n} \sum_{t=1}^n \w_t$ satisfies  
 \begin{align*}
\En \brk*{F(\wh \w_n) - F^* } &\leq \eta (\sigma^2 + L^2) + \frac{1}{2 \eta n} { \nrm*{\w_1 - \w^*}_2^2}. 
\end{align*}
Setting $\eta = \frac{1}{\sqrt{n}}$ and using the fact that $\nrm*{\w_1 - \w^*}_2 \leq B$ in the above, we get that  
 \begin{align*} 
\En \brk*{F(\wh \w_n) - F^* } &\leq \frac{1}{\sqrt{n}} \prn*{ \sigma^2 + L^2 + B^2}, 
\end{align*}
which is the desired claim. Dependence on the problem specific constants ($\sigma, L$ and $B$) in the above bound can be improved further with a different choice of the step size $\eta$; getting the optimal dependence on these constants, however, is not the focus of this work . 
\end{proof}

\subsection{Basic probability results}  \label{app:basic_probability} 
\begin{lemma}[Hoeffding's inequality] \label{lem:hoeffding}  
 Let $X_1, \ldots, X_n$ be independent random variables with values in the interval $[a, b]$, and the expected value $\En \brk*{X} = \mu$. Then, for every $t \geq 0$, 
\begin{align*}
\Pr \prn[\Big]{\abs[\big]{\frac{1}{n} \sum_{j=1}^{n} X_j - \mu} \geq t} \leq 2\exp \prn[\Big]{-\frac{2t^2n}{(b - a)^2}}. 
\end{align*}
\end{lemma}

\begin{lemma}  \label{lem:special_coordinate_exists1} Let $j^* \in [d]$. Let $X$ be a $\crl*{0, 1}^d$ valued random variable such that $X[j]$ is sampled independently from $\cB\prn*{p}$ for every $j \in [d]$.  Let the $\crl*{X_1, \ldots, X_n}$ denote n i.i.d.\ samples of the random variable $X$. If $$n \leq \log_{\frac{1}{1 - p}} \prn[\Big]{\frac{d - 1}{\ln(10)}},$$ then, with probability at least $0.9$, there exists a coordinate $\wh j \in [d]$ such that $\wh j \neq j^*$, and $X_i[\wh j] = 0$ for all $i \in [n]$. 
\end{lemma} 
\begin{proof} Let $E_j$ denote the event the coordinate $X_i[j] = 0$ for all $i \in [n]$. We note that for any $j \in [d]$,  
\begin{align*}  
\Pr \prn*{E_j} &= (1 - p)^n, 
\end{align*}
Further, let $E$ denote the event that there exists a coordinate $\wh j \in [d] \setminus \crl{j^*}$ such that $X_i[\wh j] = 0$ for all $i \in [n]$. Thus, 
\begin{align*}
	\Pr \prn*{E^c} &= \Pr \prn[\Big]{\bigcap_{j \in [d] \setminus \crl{j^*} } E^c_j} \overeq{\proman{1}} \prod_{j \in [d] \setminus \crl{j^*}} \Pr \prn{E^c_j} 
\end{align*}
where the equality in $\proman{1}$ follows from the fact that the events $\crl*{E_j}$ are mutually independent to each other. Plugging in the bound for $\Pr\prn{E_j^c}$, we get
\begin{align*} 
		\Pr \prn*{E^c} &= (1 - \prn*{1 - p}^n)^{d-1}.    
\end{align*} 
Using the fact that $1 - a \leq e^{-a}$ for $a \geq 0$, we get 
\begin{align*} 
			\Pr \prn*{E^c} &\leq e^{- (1-p)^{n} (d-1) } \leq 0.1, 
\end{align*} 
where the second inequality above holds for $n \leq \log_{\frac{1}{1 - p}} \prn*{\frac{d - 1}{\ln(10)}}$. 
\end{proof}

\begin{lemma}[Proposition 7.3.2, \cite{matouvsek2001probabilistic}] 
\label{lem:rademacher_anti_concen}
For $n$ even, let $X_1, \ldots, X_n$ be independent samples from $\cB(1/2)$. Then, for any $t \in [0, n/8]$, 
\begin{align*}
		\Pr \prn[\Big]{ \sum_{i=1}^n X_1 \geq \frac{n}{2} + t} \geq \frac{1}{15} e^{-16 t^2 / n}. 
\end{align*} 	
\end{lemma}
 
\begin{lemma}[\citet{panchenko2002some}] 
\label{lem:concentration_fast_rate}
Let $\cW$ denote a finite class of $k$ points $\crl*{\w_1, \ldots, \w_K}$, and let $f(\w; z)$ denote a loss function that is $L$-Lipschitz in the variable $\w$ for all $z \in \cZ$.  Further, let $S= \crl*{z_i}_{i=1}^n$ denote a dataset of $n$ samples, each drawn independently from some distribution $\cD$. Define the point $\wh \w_{S} = \argmin_{\w \in \cW}\sum_{i=1}^n \frac{1}{n} f(\w, z_i)$. Then, with probability at least $1 - \delta$ over the sampling of the set $S$, 
\begin{align*} 
F(\wh \w_S) \leq F^* + O\prn[\Big]{ \frac{L \log(K/\delta)}{n} + \sqrt{\frac{F^* L \log(K/\delta)}{n}}}, 
\end{align*}
where $F(\w) \ldef{} \En_{z \sim \cD} \brk*{f(\w; z)}$ and $F^* \ldef{} \min_{\w \in \cW} F(\w)$. 
\end{lemma}

\section{Missing proofs from \pref{sec:SGDimp}}    
%!TEX root=../paper.tex
Throughout this section, we assume that a data sample $z$ consists of $(x, y, \alpha)$, where $x \in \{0,1\}^d$, $y \in \crl{-1, 1}$ and $\alpha \in \{0,e_1,\ldots,e_d\}$. The loss function $\fA:\reals^d \times \cZ$ is given by: 
\begin{align} 
  \fA(\w; (x,y,\alpha)) = y \nrm*{(\w - \alpha) \odot x}. \label{eq:first_function} 
\end{align} 
We also assume that $d \geq \ln(10) 2^{n} + 1$.  Since $\fA$ is not differentiable when $w = 0$, we define the sub-gradient $\partial \fA(\w) = 0$ at the point $w= 0$. Further, for the rest of the section, whenever clear from the context, we will ignore the subscript $(\ref{eq:empfn_basic_cons_sco})$, and denote the loss function by $f(\w; z)$. Additionally we define the following distribution over the samples $z = (x, y, \alpha)$. 

\begin{definition} \label{def:violating_distribution_defn}
For parameters $\delta \in [0, \frac{1}{2}]$, $p \in [0, 1]$, and $a \in \{0,e_1,\ldots,e_d\}$, define the distribution $\cD(\delta, p, a)$ over $z = (x, y, \alpha)$ as follows:
\begin{align*} x \sim \cB(p)^{\otimes d}, \qquad y = 2r - 1 \quad \text{for} \quad r \sim \cB(\tfrac{1}{2} + \delta), \qquad \alpha = a.
\end{align*}
The components $x$, $y$ and $\alpha$ are sampled independently.
\end{definition}

\subsection{Supporting technical results}  

The following lemma provides some properties of the population loss under $\cD(\delta, p, a)$ given in \pref{def:violating_distribution_defn}. 
\begin{lemma} \label{lem:pop-loss-convexity} For any $\delta \in [0, \frac{1}{2}]$, $p \in (0, 1]$, and $a \in \{0,e_1,\ldots,e_d\}$, the population loss under $\fA$ when the data is sampled i.i.d.\ from $\cD(\delta, p, a)$ is convex in the variable $\w$. Furthermore, the population loss has a unique minimizer at $\w = a$ with $F(a) = 0$, and any $\epsilon$-suboptimal minimizer $\w$ of the population loss must satisfy $\|\w - a\| \leq \epsilon / 2 \delta p$. 
\end{lemma}
\begin{proof} Let $F$ be the population loss. Using the definition of $F$, we get 
\begin{align*} 
   F(\w) &= \En_{(x, y, \alpha) \sim \cD(\delta, p, a)} \brk*{y \nrm*{(\w - \alpha) \odot x}} \\
   		 &= \Pr(y = 1) \En_{x} \brk*{\nrm*{(\w - a) \odot x}} - \Pr(y=-1) \En_{x, \alpha} \brk*{\nrm*{(\w - a) \odot x}} \\
   		 &= (\tfrac{1}{2}+\delta) \En_{x} \brk*{\nrm*{(\w - a) \odot x}} - (\tfrac{1}{2}-\delta) \En_{x} \brk*{\nrm*{(\w - a) \odot x}} \\
   		 &= 2\delta\En_{x} \brk*{\nrm*{(\w - a) \odot x}}. 
\end{align*} 
Since, for any $x$ and $a$, the function $\nrm*{(\w - a) \odot x}$ is a convex function of $\w$, the above formula implies that $F$ is also a convex function of $\w$. Furthermore, $F$ is always non-negative and $F(a) = 0$. 

Now note that 
\begin{align*}
F(\w) = 2\delta \En_{x} \brk*{\nrm*{(\w - a) \odot x}} \geq 2\delta \nrm*{(\w - a) \odot \En_{x}[x]} = 2\delta p\nrm*{\w - a},
\end{align*}
where the second inequality follows from Jensen's inequality, and the last equality from the fact that $x \sim \cB(p)^d$. Now, if $\w$ is an $\epsilon$-suboptimal minimizer of $F$, then since $F(a) = 0$, the above bound implies that $\|\w - a\| \leq \epsilon / 2\delta p$. This also implies, in particular using $\epsilon = a$, that $a$ is a unique minimizer of $F$. 
\end{proof}

The next lemma establishes empirical properties of a dataset $S$ of size $n$ drawn from the distribution $\cD(1/10, 1/2, a)$ given in \pref{def:violating_distribution_defn}. 
\begin{lemma} \label{lem:datatset_properties_basic_sco} Let $j^* \in [d]$. Let  $S$ denote a dataset of size $n \leq \log_{2} \prn*{d/{\ln(10)}}$ sampled i.i.d.\ from a distribution $\cD(\frac{1}{10}, \frac{1}{2}, a)$ for some vector $a \in \crl{0, e_1, \ldots, e_d}$.  Then, with probability at least $0.9$ over the choice of the dataset $S$, there exists an index $\wh j \in [d]$ such that $\wh j \neq j^*$ and $x[\wh j]= 0$ for all $z \in S$ for which $y = 1$ and $x[\wh j]= 1$ for all $z \in S$ for which $y = -1$.  
\end{lemma}
\begin{proof} 
Let $S$ denote a set of $n$ samples drawn i.i.d.\ from $\cD(\frac{1}{10}, \frac{1}{2}, a)$. We define the sets $S^+$ and $S^-$ as follows: 
\begin{align*}
   S^+ &\ldef{} \crl*{z_i  \in S \mid{} y_i = +1},  \\ 
   S^- &\ldef{} \crl*{z_i \in S \mid{} y_i = -1}.  \numberthis \label{eq:set_plus_minus_definition2} 
\end{align*} 

Let $E_j$ denote the event that the coordinate $x[j] = 0$ for all $z =(x, y) \in S^+$, and $x[j] = 1$ for all $z =(x, y) \in S^-$. Since, for each sample, $x[j]$ is drawn independently from $\cB(1/2)$, we have that 
\begin{align*}
\Pr(E_j) &= \frac{1}{2^n}. 
\end{align*}

Next, let $E$ denote the event that there exists some $\wh j \in [d] \setminus \{j^*\}$ for which $x[\wh j] = 0$ for all $z =(x, y) \in S^+$, and $x[\wh j] = 1$ for all $z =(x, y) \in S^-$. We thus note that, 
\begin{align*} 
	\Pr \prn*{E^c} &= \Pr \prn[\big]{\bigcap_{j \in [d] \setminus \crl{j^*}} E^c_j} \overeq{\proman{1}} \prod_{j \in [d] \setminus \crl{j^*}} \Pr \prn[\big]{E^c_j}  
\end{align*}
where the equality in $\proman{1}$ follows from the fact that the events $\crl*{E_j}$ are mutually independent to each other. Plugging in the bound for $\Pr\prn{E_j^c}$, we get
\begin{align*} 
\Pr(E^c) &= \prn[\Big]{1 - \frac{1}{2^n}}^{d-1}. 
\end{align*}

Using the fact that $1 - a \leq e^{-a}$ for $a \geq 0$, we get
\begin{align*} 
			\Pr \prn*{E^c} &\leq e^{- \prn*{d-1}/2^n } \leq 0.1, 
\end{align*}
where the second inequality above holds for $n \leq \log_{2} \prn*{(d - 1)/{\ln(10)}}$. 
\end{proof} 

\subsection{Proof of \pref{thm:lowersco}}  \label{app:thm_lowersco_proof} 
We now have all the tools required to prove \pref{thm:lowersco}, which states for any regularization function $R(\w)$, there exists an instance of a SCO problem for which RERM fails to find $O(1)$-suboptimal solution, in expectation.  
\begin{proof}[{\bf Proof of \pref{thm:lowersco}}]   
In this proof, we will assume that $n \geq 300$, $d \geq \log(10)2^n + 1$ and the initial point $w_1 = 0$. Assume, for the sake of contradiction, that there exists a regularizer $R: \bbR^d \to \bbR$ such that for any SCO instance over $\bbR^d$ satisfying Assumption II with\footnote{The specific values $L, B = 1$ are used for convenience; the proof immediately yields the required SCO instances for arbitray values of $L$ and $B$ that are $O(1)$ in magnitude as a function of $n$.} $L, B = 1$ the expected suboptimality gap for RERM is at most $\epsilon = 1/ 20000$. Then, by Markov's inequality, with probability at least $0.9$ over the choice of sample set $S$, the suboptimality gap is at most $10\epsilon$.

Before delving into the proof, we first define some additional notation based on the regularization function $R(\cdot)$. For $j \in [d]$, define the points $\w^*_j$  such that 
\begin{align*} 
\w^*_{j}  \in \argmin_{\w \text{~s.t.~} {\nrm*{\w - e_j} \leq 100 \epsilon}} R(\w)  \numberthis \label{eq:w_jstar_defn}. 
\end{align*} and define the index $j^* \in [d]$ such that  $j^* \in \argmax_{j \in [d]} ~ R(\w^*_j).$ 

Now we will construct an instance of SCO in $d = \lceil2^n \ln(10) + 1\rceil$ dimensions. The instance will be based on the function $\fA$ given in \pref{eq:first_function}. The data distribution of interest is $\cD_1\ldef{} \cD(\frac{1}{10}, \frac{1}{2}, e_{j^*})$ (see \pref{def:violating_distribution_defn}) and suppose that the dataset $S = \crl*{z_i}_{i=1}^n$ is sampled i.i.d.\ from $\cD_1$. The population loss $F(w)$ corresponding to $\cD_1$ is given by 
\begin{align*} 
   F(\w) &= \En_{z \sim \cD_1} \brk*{\fA(\w; z)} = 0.2 \En_{x} \brk*{\nrm*{(\w-e_{j^*}) \odot x}}.  
\end{align*}
Clearly, $F(w)$ is convex in $w$ and so minimizing it is an instance of SCO. Furthermore, $\fA$ is $1$-Lipschitz, and the initial point $w_1 = 0$ satisfies $\nrm*{w_1 - w^*} \leq 1$. This implies that $\fA$ satisfies both Assumptions I and II given in (\ref{eq:ass1}) and (\ref{eq:ass2}) respectively. Additionally, note that $e_{j^*}$ is the unique minimizer of $F(\cdot)$, and as a consequence of  \pref{lem:pop-loss-convexity}, any $10\epsilon$-suboptimal minimizer $w'$ for $F(\cdot)$ must satisfy 
\begin{align*}
\nrm{w' - e_{j^*}} \leq 100 \epsilon.  \numberthis \label{eq:optimality_condition_a} 
\end{align*}

Using the dataset $S$, we define some additional sets as follows: 
\begin{enumerate}[label=$\bullet$, leftmargin=8mm]  
\item Define the set $S^+$ as the set of all the sample points in $S$ for which $y = +1$, i.e 
 $$S^+ \ldef{} \crl*{(x, y, \alpha) \in S \mid y = +1},$$
and define the set $S^- \ldef{} S \setminus S^+$. 
\item Define the set $U$ as the set of all the sample points in $S^+$ for which $x[j^*] = 1$, i.e.
$$U \ldef{} \crl*{(x, y, \alpha) \in S^+ \mid x[j^*] = 1}.$$ 
\item Similarly, define the set $V$ as the set of all the sample points in $S^-$ for which $x[j^*] = 1$, i.e. $$V \ldef{} \crl*{(x, y, \alpha) \in S^- \mid x[j^*] = 1}.$$
\end{enumerate} 
Next, define the event $E$ such that all of the following hold: 
\begin{enumerate}[label=$(\alph*)$, leftmargin=8mm]  
\item $\abs{S^-} \geq 7 n / 20$. 
\item $\abs{U} \leq  39n/100$. 
\item $\abs{V} \geq 7n/50$. 
\item There exists $\wh j$ such that $\wh j \neq j^*$ and $x[\wh j] = 0$ for all $z \in S^+$ and $x[\wh j] = 1$ for all $z \in S^-$.   
\item RERM with regularization $R(\cdot)$ and using the dataset $S$ returns an $10\epsilon$-suboptimal solution for the test loss $F(w)$. 
\end{enumerate} 
Using Hoeffding's inequality (\pref{lem:hoeffding}) and the fact that $\En[\abs{S^-}] = 2n/5$, $\En[\abs{U}] = 3n/10$, and $\En[\abs{V}] = n/5$, we get that parts (a), (b) and (c) hold simultaneously with probability at least $0.3$ for $n \geq 300$. Furthermore, by \pref{lem:datatset_properties_basic_sco}, part (d) holds with probability at least $0.9$ since $d \geq \ln(10) 2^{n} + 1$, and part (e) holds with probability $0.9$ by our assumption. Hence, the event $E$ occurs with probability at least $0.1$. In the following, we condition on the occurrence of the event $E$.

Consider the point $w^*_{\wh j}$ defined in \pref{eq:w_jstar_defn} corresponding to the coordinate $\wh j$ (that occurs in event $E$). By definition, we have that $\nrm{w^*_{\wh j} -e_{\wh j}} \leq 100 \epsilon$. Thus, we have
\begin{align*} 
\nrm{w^*_{\wh j} - e_{j^*}} \geq \nrm{e_{\wh j} - e_{j^*}} - \nrm{w^*_{\wh j} - e_{\wh j}} \geq \sqrt{2}  - 100 \epsilon > 100\epsilon. 
\end{align*} The first line above follows from Triangle inequality, and the second line holds because $\wh j \neq j^*$ and because $\epsilon = 1 / 20000$. As a consequence of the above bound and the condition in \pref{eq:optimality_condition_a}, we get that the point $w^*_{\wh j}$ is not an $10\epsilon$-suboptimal point for the population loss $F(\cdot)$, and thus would not be the solution of the RERM  algorithm (as the RERM solution is $10\epsilon$-suboptimal w.r.t $F(\cdot)$). Since, any RERM must satisfy condition \pref{eq:optimality_condition_a}, we have that 
\begin{align*}
\wh F(\w^*_{\wh j}) +  R(\w^*_{\wh j}) &> \min_{\w:\ \nrm*{\w - e_{j^*}} \leq 100\epsilon} \prn[\big]{\wh F(\w) +  R(\w)} \\ &\geq  \min_{\w:\ \nrm*{\w - e_{j^*}} \leq 100\epsilon} \wh F(\w) + \min_{\w:\ \nrm*{\w - e_{j^*}} \leq 100\epsilon} R(\w) \\ &\geq - 100\epsilon +  R(\w^*_{j^*}),  \numberthis \label{eq:RERM_proof_gen5_1}
\end{align*}
where $\wh F(w) = \frac{1}{n} \sum_{i=1}^n f(w; z_i)$ denotes the empirical loss on the dataset $S$, and the inequality in the last line follows from the definition of the point $w^*_{j^*}$ and by observing that $\wh F(w) \geq -100\epsilon$ for any $w$ for which $\nrm*{w - e_{j^*}} \leq 100\epsilon$ since $\wh F$ is $1$-Lipschitz and $\wh F(e_{j^*}) = 0$. For the left hand side, we note that    
\begin{align*} 
\wh F(\w^*_{\whj}) &= \frac{1}{n} \sum_{z \in S} y \nrm{(w^*_{\wh j} - e_{j^*} ) \odot x} \\  
 &\overleq{\proman{1}} 100\epsilon +  \frac{1}{n} \sum_{z \in S} y \nrm{(e_{\wh j} - e_{j^*} ) \odot x}  \\ 
 &\overleq{\proman{2}}  100\epsilon +  \frac{1}{n} \prn[\Big]{ \sum_{z \in S^+} \nrm{(e_{\wh j} - e_{j^*} ) \odot x} -  \sum_{z \in S^-} \nrm{(e_{\wh j} - e_{j^*} ) \odot x}} \\ 
 &\overleq{\proman{3}} 100\epsilon +  \frac{1}{n} \prn[\Big]{ \sum_{z \in S^+} \abs{x[j^*]} -  \sum_{z \in S^+} \sqrt{1 + (x[j^*])^2} } \\ 
  &\overleq{\proman{4}} 100\epsilon +  \frac{1}{n} \prn[\Big]{ \sum_{z \in U} 1 -  \sum_{z \in V} \sqrt{2}  - \sum_{z \in S^- \setminus V} 1} \\  
  &= 100\epsilon +  \frac{1}{n} \prn[\Big]{ \abs{U} - (\sqrt{2} - 1) \abs{V}  - \abs{S^-} } 
\end{align*} where the inequality $\proman{1}$ is due to the definition of the point $w^*_{\wh j}$, and because $\wh F$ is $1$-Lipschitz, the inequality in $\proman{2}$ follows from the definition of the sets $S^+$ and $S^-$, the inequality $\proman{3}$ holds due to the fact that $x[\wh j] = 0$ for all $(x, \alpha) \in S^+$, and $x[\wh j] = 1$ for all $x \in S^-$, and finally, the inequality $\proman{4}$ follows from the definition of the sets $U$ and $V$. Plugging the bounds on $\abs{U}$, $\abs{V}$  and $\abs{S^-}$ from the event $E$ defined above, we get that: 
\begin{align*}
\wh F(\w^*_{\whj}) 
&\leq 100\epsilon - \frac{3}{200}. 
\numberthis \label{eq:RERM_proof_gen6_1} 
\end{align*} 

Combining the bounds in \pref{eq:RERM_proof_gen5_1}  and \pref{eq:RERM_proof_gen6_1} and rearranging the terms, we get that 
\begin{align*}
200\epsilon &\geq \frac{3}{200} + R(\w^*_{j^*}) - R(\w^*_{\wh j})  \geq \frac{3}{200}. 
\end{align*} where the second inequality above holds because $j^* \in \argmax_{j \in [d]} R(w^*_j)$ (by definition). Thus, $\epsilon \geq 3 / 40000  > 1 / 20000$, a contradiction, as desired.\footnote{The constant in the lower bound for $\epsilon$ can be improved further via a tighter analysis for the sizes of the sets $\abs{U}, \abs{V}$ and $\abs{S^-}$ in the event $E$.}

Finally, note that since the function $\fA$ is $1-$Lipschitz, the initial point $w_1 = 0$ satisfies $\nrm*{w_1 - w^*} = \nrm*{w_1 - e_{j^*}} \leq 1$ and $F(w)$ is convex, due to \pref{thm:sgd_works_main}, SGD run with a step size of $1/\sqrt{n}$ for $n$ steps learns at a rate of $O(1/\sqrt{n})$. 
\end{proof} 

\subsection{Proof of \pref{corr:lowersco_lambda}} 
The following proof closely follows along the lines of the proof of \pref{thm:lowersco} above.   
\begin{proof} In this proof, we will assume that $n \geq 300$, $d \geq \log(10)2^n + 1$ and the initial point $w_1 = 0$. Assume, for the sake of contradiction, that there exists a regularizer $R: \bbR^d \to \bbR$ such that for any SCO instance over $\bbR^d$ satisfying Assumption II with $L, B = 1$, there exists a regularization parameter $\lambda$ such that  the expected suboptimality gap for the point $\w_{\mathrm{RERM}} = \argmin_{\w \in \W} F_S(\w)  + \lambda R(\w)$ is at most $\epsilon = 1/ 20000$. Then, by Markov's inequality, with probability at least $0.9$ over the choice of sample set $S$, the suboptimality gap is at most $10\epsilon$.

We next define the functions $f(w;z)$, the distribution $\cD_1$ and the population loss function $F(w)$ identical to the the corresponding quantities in the proof of \pref{thm:lowersco}. Furthermore, we also define the points $w_{j}$ for $j \in [d]$, the coordinates $j^*$ and the event $E$ identical to the corresponding definitions in the proof of \pref{thm:lowersco}. As we argued in the proof of \pref{thm:lowersco} above, we note that any $10\epsilon$-suboptimal minimizer $w'$ for $F(\cdot)$ must satisfy 
\begin{align*}
\nrm{w' - e_{j^*}} \leq 100 \epsilon.  \numberthis \label{eq:optimality_condition_a_2} 
\end{align*}

Thus, the point $w^*_{\wh j}$ (where $\wh j$ is defined in the event $E$) is not an $10\epsilon$-suboptimal point for the population loss $F(\cdot)$, and thus for any regularization parameter $\lambda$ (that can even depend on the dataset $S$) should not correspond to the RERM solution (as the point $\w_{\mathrm{RERM}}$ is $10\epsilon$-suboptimal with respect to $F(\cdot)$). Since, any RERM must satisfy condition \pref{eq:optimality_condition_a_2}, we have that  for any regularization parameter $\lambda$ that can even depend on the dataset $S$,  
\begin{align*}
\wh F(\w^*_{\wh j}) + \lambda R(\w^*_{\wh j}) &> \min_{\w:\ \nrm*{\w - e_{j^*}} \leq 100\epsilon} \prn[\big]{\wh F(\w) + \lambda R(\w)} \\ &\geq  \min_{\w:\ \nrm*{\w - e_{j^*}} \leq 100\epsilon} \wh F(\w) + \min_{\w:\ \nrm*{\w - e_{j^*}} \leq 100\epsilon} \lambda R(\w) \\ &\geq - 100\epsilon + \lambda R(\w^*_{j^*}),  \numberthis \label{eq:RERM_proof_gen5_1_1}
\end{align*} 
where the inequality in the last line follows from the definition of the point $w^*_{j^*}$ and by observing that $\wh F(w) \geq -100\epsilon$ for any $w$ for which $\nrm*{w - e_{j^*}} \leq 100\epsilon$ since $\wh F$ is $1$-Lipschitz and $\wh F(e_{j^*}) = 0$. For the left hand side, similar to the proof of \pref{thm:lowersco}, we upper bound     
\begin{align*} 
\wh F(\w^*_{\whj}) &\leq 100\epsilon - \frac{3}{200}. 
\numberthis \label{eq:RERM_proof_gen6_1_1} 
\end{align*} 

Combining the bounds in \pref{eq:RERM_proof_gen5_1_1}  and \pref{eq:RERM_proof_gen6_1_1} and rearranging the terms, we get that 
\begin{align*}
200\epsilon &\geq \frac{3}{200} + \lambda R(\w^*_{j^*}) - \lambda R(\w^*_{\wh j})   \geq \frac{3}{200}. 
\end{align*} where the second inequality above holds because $j^* \in \argmax_{j \in [d]} R(w^*_j)$ (by definition). Thus, $\epsilon \geq 3 / 40000  > 1 / 20000$, a contradiction, as desired. 

We remark that in the above proof the regularization parameter $\lambda$ can be arbitrary and can even depend on the dataset $S$, but $\lambda$ should not depend on $w$ as this will change the definition of the points $w^*_j$. Only the regularization function $R(\cdot)$ is allowed to depend on $w$. 
\end{proof}

\section{Missing proofs from \pref{sec:GD_fails}}  \label{app:gd_fails} 
\newcommand{\speta}{\eta}
\newcommand{\spT}{T}
\newcommand{\spoT}{\bar{T}} 
\newcommand{\spw}{\bar{w}}

\newcommand{\und}[1]{_{(#1)}} 
In this section, we first provide a learning problem for which, for any $\eta \in [1/n^2, 1)$ and $T \in [1, n^3)$, the point $\wgd_{\eta, T}$ returned by running GD algorithm with step size $\eta$ for $T$ time steps has the lower bound 
\begin{align*}
\En \brk{F(\wgd_{\eta, T})} - \inf_{\w \in \reals^d} F(\w) = \Omega\prn[\Big]{\frac{1}{\log^4(n)} \min\crl[\Big]{ \eta \sqrt{T} + \frac{1}{\eta T} + \frac{\eta T}{n},1}}. \numberthis \label{eq:app_GD_lower_bound_1} 
\end{align*} 

Then, we will provide a learning setting in which GD algorithm when run with step size $\eta \leq 1/64 n^{5/4}$ has the lower bound of $\Omega(1/n^{3/8})$ for all $T \geq 0$. Our final lower bound for GD algorithm, given in \pref{thm:GD_vs_SGD_comparison} then follows by considering the above two lower bound constructions together.

\subsection{Modification of \cite{amir2020gd} lower bound} 
\cite{amir2020gd} recently provided the following lower bound on the performance of GD algorithm. \begin{theorem}[Modification of Theorem 3.1, \cite{amir2020gd}]  
\label{thm:amir_lower_bound} 
Fix any  $n$,  $\bar{\eta}$ and $\spoT$. There exists a function $f(w; z)$ that is $4$-Lipschitz and convex in $w$, and a distribution $\cD$ over the instance space $\cZ$, such that for any $\eta \in [\bar{\eta}, \bar{\eta} \sqrt{3/2})$ and $T \in [\spoT, 2\spoT)$, the point $\wh w^{\text{GD}}[\eta, T]$ returned by running GD algorithm with a step size of $\eta$ for $T$ steps has excess risk 
\begin{align*} 
\En \brk{ F (\wgd[\eta, T])} - \inf_{\w \in \reals^d} F (\w) \ge \Omega\prn[\Big]{ \min\crl[\Big]{ \eta \sqrt{T} + \frac{1}{\eta T},1}},   \numberthis \label{eq:lower_bound_eqn_5} 
\end{align*}
where $F(w) \ldef{} \En_{z \sim \cD} \brk{f(w; z)}$. Additionally, there exists a point $w^* \in \argmin_{w} F(w)$ such that $\nrm{w^*} \leq 1$. 
\end{theorem} 
\begin{proof} We refer the reader to \cite{amir2020gd} for full details about the loss function construction and the lower bound proof, and discuss the modifications below: 

\begin{enumerate}[label=$\bullet$, leftmargin=8mm] 
\item   \cite{amir2020gd} provide the lower bound for a fixed $\eta$ and $T$. In particular, their loss function construction given in eqn-(16) (in their paper) consists of parameters $\gamma_1$, $\gamma_2$, $\gamma_3$, $\epsilon_1, \ldots, \epsilon_3$ and $d$ which are chosen depending  $\eta$ and $T$. However, a slight modification of these parameters easily extends the lower bound to hold for all $\eta \in [\bar{\eta}, \bar{\eta} \sqrt{3/2})$ and $T \in [\spoT, 2\spoT)$. The modified parameter values are: 
\begin{enumerate}[label=$\bullet$,  leftmargin=8mm]  
\item Set $d = \tfrac{2^n}{3 \bar{\eta}^2}$. 
\item Set $\gamma_1$ such that $\gamma_1(1 + 5 \bar{\eta} \spoT) \sqrt{d} \leq \tfrac{1}{64} \min\crl{\bar{\eta} \sqrt{\spoT}, \tfrac{1}{3}}$. 
\item Set $\gamma_2 = 2 \gamma_1 \bar{\eta} \spoT$ and $\gamma_3 = 1$.  
\item Set $0 < \epsilon_1 < \cdots < \epsilon_d < \tfrac{\gamma_1 \bar{\eta}}{2n}$. 
\end{enumerate} 

Their proof of the lower bound in Theorem 3.1 follows using Lemma 4.1, Theorem 6.1, Lemma 6.2 and Claim 6.3 (in their paper respectively).  We note that the above parameter setting satisfies the premise of Lemma 4.1 and Claim 6.3 for all $\eta \in [\bar{\eta}, \bar{\eta} \sqrt{3/2})$ and $T \in [\spoT, 2\spoT)$. Furthermore, it can be easily verified that the proofs of Theorem 6.1 and Lemma 6.2 also follow through with the above parameter setting for all $\eta \in [\bar{\eta}, \bar{\eta} \sqrt{3/2})$ and $T \in [\spoT, 2\spoT)$. Thus, the desired lower bound holds for all $\eta \in [\bar{\eta}, \bar{\eta} \sqrt{3/2})$ and $T \in [\spoT, 2\spoT)$. 

\item \cite{amir2020gd} consider GD with projection on the unit ball as their training algorithm. However, their lower bound also holds for GD without the projection step (as we consider in our paper; see \pref{eq:GD}). In fact, as pointed out in the proof of Lemma 4.1  (used to Prove Theorem 3.1) in their paper, the iterates $w_t$ generated by the GD algorithm never leave the unit ball. Thus, the projection step is never invoked, and GD with projection and GD without projection produce identical iterates. 
\end{enumerate}

The upper bound on $\nrm{w^*}$ also follows from the provided proof in  \cite{amir2020gd}. As they show in Lemma 4.1, all the GD iterates as well as the minimizer point $w^*$ lie a ball of unit radius. 

Finally, note that the loss function in \pref{thm:amir_lower_bound} is $4$-Lipschitz, and bounded over the unit ball which contains all the iterates generated by GD algorithm. Thus, a simple application of the Markov's inequality suggests that the lower bound in  \pref{eq:lower_bound_eqn_5} also holds with constant probability. However, for our purposes, the in-expectation result suffices. 
\end{proof}

The loss function construction in \pref{thm:amir_lower_bound}  depends on $\bar{\eta}$ and $\spoT$, and thus the lower bound above only holds when the GD algorithm is run with step size $\eta \in [\bar{\eta}, \bar{\eta} \sqrt{3/2})$ and for $T \in [\spoT, 2\spoT)$. In the following, we combine together multiple such lower bound instances to get an anytime and any stepsize guarantee. 

\begin{theorem} 
\label{thm:amir_lower_bound_modified}
Fix any $n \geq 200$.  There exists a function $f(w; z)$ that is $1$-Lipschitz and convex in $w$, and a distribution $\cD$ over $\cZ$ such that, for any $\eta' \in [1/n^2, 1)$ and $T' \in [1, n^3)$, the point $\wh w^{\text{GD}}[\eta', T']$ returned by running GD algorithm with a step size of $\eta'$ for $T'$ steps satisfies the lower bound  
\begin{align*} 
\En \brk{ F (\wgd[\eta', T'])} - \inf_{\w \in \reals^d} F (\w) \ge \Omega\prn[\bigg]{\frac{1}{\log^4(n)} \min\crl[\Big]{ \eta' \sqrt{T'} + \frac{1}{\eta' T'},1}},  \numberthis \label{eq:anytime_lb1} 
\end{align*}
 where $F(w) \ldef{} \En_{z \sim \cD} \brk*{f(w; z)}$. Additionally, there exists a minimizer $w^* \in \argmin_{w} F(w)$ such that $\nrm{w^*} = O(1)$. 
\end{theorem}  
\begin{proof} Set $\gamma \ldef{} \sqrt{3/2}$. We consider a discretization of the intervals $[1/ n^3, 1)$ for the step sizes (we take a slightly larger interval that the domain $[1/n^2, 1)$ of the step size) and $[1, n^3)$ for the time steps respectively. Define the set 
\begin{align*} 
& \cN \ldef{} \crl[\Big]{\frac{1}{n^3}, \frac{\gamma}{n^3}, \frac{\gamma^2}{n^3}, \cdots, \frac{\gamma^{\ceil*{3 \log (n) / \log (\gamma)}}}{n^3}} 
\intertext{of step sizes such that for any $\eta' \in [1/n^3, 1)$, there exists an ${\speta} \in \cN$ that satisfies ${\speta} \leq \eta' < \gamma {\speta}$. Similarly, define the set }
&\cT \ldef{} \crl[\big]{1, 2, 4, \ldots, 2^{\ceil*{3\log(n)}}}
\end{align*} 
of time steps such that for any $T' \in [1, n^3)$, there exists a ${\spT} \in \cT$ that satisfies ${\spT} \leq T' < 2 {\spT}$. Further, define $M = \abs{\cN} \abs{\cT}$. Clearly, $M =  \ceil*{3 \log (n) / \log (\gamma)} \cdot \ceil*{3 \log(n)}$ and for $n \geq 20$ satisfies the bound $40 \log^2(n) \leq M \leq 80 \log^2(n)$. 

In the following, we first define the component function $f_{\speta, \spT}$ for every $\speta \in \cN$ and $\spT \in \cT$. We then define the loss function $f$ and show that it is convex and Lipschitz in the corresponding optimization variable. Finally, we show the lower bound for GD for this loss function $f$ for any step size $\eta \in [1/n^2, 1)$ and time steps $T \in [1, n^3)$. 

\paragraph{Component functions.} For any $\speta \in \cN$ and $\spT \in \cT$, let $\spw_{\speta, \spT}$, $z_{\speta, \spT}$,  $f_{\speta, \spT}$ and $\cD_{\speta, \spT}$ denote the optimization variable,  data instance, loss function and the corresponding data distribution in the lower bound construction in \pref{thm:amir_lower_bound} where $\bar{\eta}$ and $\bar{T}$ are set as $\speta$ and $\spT$ respectively. We note that: 
\begin{enumerate}[label=(\alph*), leftmargin=8mm]
\item For any $z_{\speta, \spT}$, the function $f_{\speta, \spT}(\spw_{\speta, \spT}; z_{\speta, \spT})$ is $4$-Lipschitz and convex in $\spw_{\speta, \spT}$.
\item For any $\eta'' \in [\speta, \gamma \speta)$ and $T'' \in [\spT, 2\spT)$, the output point\footnote{We use the notation $\wgd_{\speta, \spT}[\eta'', T'']$ to denote the value of the variable $\spw_{\speta, \spT}$ computed by running GD algorithm with the step size $\eta''$ for $T''$ time steps. The subscripts $\speta, \spT$ are used to denote the fact that the corresponding variables are associated with the loss function construction in \pref{thm:amir_lower_bound} where $\bar{\eta}$ and $\bar{T}$ are set as $\speta$ and $\spT$ respectively.}  $\wgd_{\speta, \spT}[\eta'', T'']$ returned by running GD algorithm with a step size of $\eta''$ for $T''$ steps has excess risk 
\begin{align*} 
\En \brk[\big]{ F_{\speta, \spT} \prn[\big]{\wgd_{\speta, \spT}[\eta'', T'']}} - \inf_{\spw_{\speta, \spT}} F_{\speta, \spT} (\spw_{\speta, \spT}) \ge \Omega\prn[\Big]{ \min\crl[\Big]{ \eta'' \sqrt{T''} + \frac{1}{\eta'' T''},1}}, \numberthis \label{eq:lower_bound_eqn_mix_1} 
\end{align*} 
where the population loss $F_{\speta, \spT} (\spw_{\speta, \spT}) \ldef{} \En_{z_{\speta, \spT} \sim \cD_{\speta, \spT}} \brk{f_{\speta, \spT}(\spw_{\speta, \spT}; z_{\speta, \spT})}$. 
\item There exists a point $\spw_{\speta, \spT}^* \in \argmin F_{\speta, \spT}(\spw_{\speta, \spT})$ such that $\nrm{\spw^*_{\speta, \spT}} \leq 1$. 
\end{enumerate}

\paragraph{Lower bound construction.} We now present our lower bound construction: 
\begin{enumerate}[label=$\bullet$, leftmargin=8mm] 
\item \textit{Optimization variable}: For any $\speta$ and $\spT$, define  $w_{\speta, \spT} \ldef{} \spw_{\speta, \spT} / \log(n)$. The optimization variable $w$ is defined as the concatenation of the variables $\prn{w_{\speta, \spT}}_{\speta \in \cN, \spT \in \cT}$. 
\item \textit{Data instance}: $z$ is defined as the concatenation of the data instances $\prn{z_{\speta, \spT}}_{\speta \in \cN, \spT \in \cT}$. 
\item \textit{Data distribution}: $\cD$ is defined as the cross product of the distributions $\prn{\cD_{\speta, \spT}}_{\speta \in \cN, \spT \in \cT}$. Thus, for any $\speta \in \cN$ and $\spT \in \cT$, the component $z_{\speta, \spT}$ is sampled independent from $\cD_{\speta, \spT}$. 
\item \textit{Loss function}: is defined as 
\begin{align*} 
f(w; z) 
&= \frac{1}{M} \sum_{\speta \in \cN, \spT \in \cT} f_{\speta, \spT}\prn*{\spw_{\speta, \spT}; z_{\speta, \spT}},  \numberthis \label{eq:grid_loss_function}
\end{align*}
where recall that $\spw_{\speta, \spT} = \w_{\speta, \spT} \log(n)$. Additionally, we define the population loss~$F(w) \ldef{} \En_{z \sim \cD} \brk*{f(w; z)}$.  
\end{enumerate} 

\paragraph{$f$ is convex and $1$-Lipschitz.} Since, for any $\speta \in \cN$ and $\spT \in \cT$, the function $f_{\speta, \spT}(\spw_{\speta, \spT}; z_{\speta, \spT})$ is convex in $\spw_{\speta, \spT}$ for every $z_{\speta, \spT}$, and since $w_{\speta, \spT} = \spw_{\speta, \spT} / \log(n)$, we immediately get that the function $f(w;z)$ is also convex in $w$ for every $z$. Furthermore, for any $w$, $w'$ and $z$, we have
\begin{align*} 
\abs*{f(w; z) - f(w'; z)} &\overleq{\proman{1}} \frac{1}{M} \sum_{\speta \in \cN, \spT \in \cT}  \abs*{f_{\speta, \spT}(\spw_{\speta, \spT}; z_{\speta, \spT}) - f_{\speta, \spT}(\spw'_{\speta, \spT}; z_{\speta, \spT})} \\
&\overleq{\proman{2}}  \frac{4 \log(n)}{M} \sum_{\speta \in \cN, \spT \in \cT}   \nrm{ w_{\speta, \spT} - w'_{\speta, \spT}} \\ 
&=  \frac{4 \log(n)}{M} \sum_{\speta \in \cN, \spT \in \cT}   \sqrt{\nrm{ w_{\speta, \spT} - w'_{\speta, \spT}}^2} \\
&\overleq{\proman{3}}  4 \log(n)    \sqrt{\frac{1}{M} \sum_{\speta \in \cN, \spT \in \cT}  \nrm{ w_{\speta, \spT} - w'_{\speta, \spT}}^2} \\
&=     \frac{4 \log(n)}{\sqrt{M}}  \sqrt{ \sum_{\speta \in \cN, \spT \in \cT}  \nrm{ w_{\speta, \spT} - w'_{\speta, \spT}}^2} \\ 
&=     \frac{4 \log(n)}{\sqrt{M}}  \sqrt{ \sum_{\speta \in \cN, \spT \in \cT}  \nrm{w_{\speta, \spT} - w'_{\speta, \spT}}^2} \\ 
&\overeq{\proman{4}}     \frac{4 \log(n)}{\sqrt{M}}  \sqrt{\nrm{ w - w'}^2} \\ 
&\overleq{\proman{5}} \nrm*{w - w'}, 
\end{align*} where the inequality $\proman{1}$ follows from Triangle inequality and the inequality $\proman{2}$ holds because $f_{\speta, \spT}$ is $4$-Lipschitz in the variable $\spw_{\speta, \spT}$ for every $\speta \in \cN$ and $\spT \in \cT$, and by using the fact that $\spw_{\speta, \spT} = w_{\speta, \spT} \log(n)$. The inequality $\proman{3}$ above follows from an application of Jensen's inequality and using the concavity of square root. Furthermore, the equality $\proman{4}$ holds by the construction of the variable $w$ as concatenation of the variables $\prn{w_{\speta, \spT}}_{\speta \in \cN, \spT \in \cT}$. Finally, the inequality $\proman{5}$ follows from the fact that $M \geq 16 \log^2(n)$ for $n \geq 4$. Thus, the function $f(w;z)$ is  $1$-Lipschitz in $w$ for every $z$. 

\paragraph{Bound on the minimizer.} Since the components $\prn{w_{\speta, \spT}}_{\speta \in \cN, \spT \in \cT}$ do not interact with each other in the loss function $f$, any point $w^* \in \argmin F(w)$ satisfies: 
\begin{align*}
\nrm*{w^*} &= \sqrt{\sum_{\speta \in \cN, \spT \in \cT} \nrm{w^*_{\speta, \spT}}^2} = \frac{1}{\log(n)} \sqrt{\sum_{\speta \in \cN, \spT \in \cT} \nrm{\spw^*_{\speta, \spT}}^2}, 
\end{align*}
where $\spw^*_{\speta, \spT}$ denote the corresponding minimizers of the population loss $F_{\speta, \spT}$. Due to \pref{thm:amir_lower_bound}, we have that for any $\speta$ and $\spT$, there exists a $ \spw^*_{\speta, \spT}$ that satisfies $\nrm{\spw^*_{\speta, \spT}} \leq 1$. Plugging these in the above, we get that there exists a $w^*$ for which  
\begin{align*}
\nrm*{w^*} &\leq \frac{1}{\log(n)} \sqrt{M} \leq 10,  
\end{align*} where the second inequality follows by using the fact that $M \leq 100 \log^2(n)$ for $n \geq 20$. 

\paragraph{Lower bound proof.} We next provide the desired lower bound for the loss function in \pref{eq:grid_loss_function}. First, note that when running GD update on the variable $w$, with a step size of $\eta'$ and using the loss function $f(w; z)$, each of the component variables $\spw_{\speta, \spT}$ are updated as if independently performing GD  on the function $f_{\speta, \spT}$ but with the step size of $\eta' / M$. 

Now, suppose that we run GD on the loss function $f$ with step size  $\eta' \in [1/n^2, 1)$ and for $T' \in [1, n^3)$ steps. For $n \geq 20$, the step size $\wt \eta \ldef{} \eta' / M \leq \eta / {40 \log^2(n)}$ with which each component is updated clearly satisfies $\wt \eta \in [1/n^3, 1]$. Thus, by construction of the sets $\cN$ and $\cT$, there exists some $\bar{\eta} \in \cN$ and $\bar{T} \in \cT$ such that $\bar{\eta} \leq \wt \eta' < \gamma \bar{\eta}$ and $\bar{T} \leq T' < 2 \bar{T}$. Thus, due to \pref{eq:grid_loss_function}, we have that for the component function corresponding to $(\bar{\eta}, \bar{T})$, the point $\wgd_{\bar{\eta}, \bar{T}}[\wt \eta, T']$ returned after running GD with step size $\wt \eta$ for $T'$ steps satisfies
\begin{align*} 
\En \brk[\big]{ F_{\bar{\eta}, \bar{T}} \prn[\big]{\wgd_{\bar{\eta}, \bar{T}}[\wt \eta, T']}} - \inf_{\spw_{\bar{\eta}, \bar{T}}} F_{\bar{\eta}, \bar{T}} (\spw_{\bar{\eta}, \bar{T}}) &\ge \Omega\prn[\Big]{ \min\crl[\Big]{ \wt \eta \sqrt{T'} + \frac{1}{\wt \eta ~ T'},1}}   \\ 
&\geq \Omega\prn[\Big]{  \frac{1}{M} \min\crl[\Big]{ \eta' \sqrt{T'} + \frac{1}{\eta' T'},1}},  \numberthis \label{eq:lower_bound_eqn_mix_6} 
\end{align*} 
where the last line holds for $\wt \eta = \eta' / M$.

Our desired lower bound follows immediately from \pref{eq:lower_bound_eqn_mix_6}. Let $\wgd[\eta', T']$ be the output of running GD algorithm on the function $f$ with step size $\eta'$ and for $T'$ steps. We have that 
\begin{align*} 
 \En \brk*{ F\prn[\big]{\wgd[\eta', T']}} - \min_{w} F(w) &=  \frac{1}{M} \sum_{\speta \in \cN, \spT \in \cT} \prn[\Big]{\En \brk[\big]{ F_{\speta, \spT} \prn[\big]{\wgd_{\speta, \spT}[\wt \eta, T']}} - \inf_{\w_{\speta, \spT}} F_{\speta, \spT} (\w_{\speta, \spT}) } \\
 &\geq \frac{1}{M} \prn[\Big]{ \En \brk[\big]{ F_{\bar{\eta}, \bar{T}} \prn[\big]{\wgd_{\bar{\eta}, \bar{T}}[\wt \eta, T']}} - \inf_{\w_{\bar{\eta}, \bar{T}}} F_{\bar{\eta}, \bar{T}} (\w_{\bar{\eta}, \bar{T}})} \\ 
 &=  \Omega\prn[\Big]{  \frac{1}{M^2} \min\crl[\Big]{ \eta' \sqrt{T'} + \frac{1}{\eta' T'},1}}, 
\end{align*}
where the equality in the first line holds because the variables $\crl{w_{\speta, \spT}}$ do not interact with each other, the inequality in the second line follows by ignoring rest of the terms which are all guaranteed to be positive, and finally, the last line follows by plugging in \pref{eq:lower_bound_eqn_mix_6}. Using the fact that $M \geq 40 \log^2(n)$ in the above, we get that for any  $\eta' \in [1/n^2, 1)$ and for $T' \in [1, n^3)$, the point $\wgd[{\eta', T'}]$ satisfies 
\begin{align*}
\En \brk*{ F(\wgd[\eta', T'])} - \min_{w} F(w)  &=  \Omega\prn[\Big]{  \frac{1}{\log^4(n)} \min\crl[\Big]{ \eta' \sqrt{T'} + \frac{1}{\eta' T'},1}}. 
\end{align*} 
\end{proof} 

\subsection{Lower bound of $\eta T / n$} 
The lower bound in \pref{eq:anytime_lb1} already matches the first two terms in our desired lower bound in \pref{eq:app_GD_lower_bound_1}. In the following, we provide a function for which GD algorithm has expected suboptimality of $\eta T / n$. 

\begin{lemma} 
\label{lem:GD_lower_bound_ours}
Fix any $n$. Let $w \in \bbR$ denote the optimization variable and $z$  denote a data sample from the instance space $\cZ = \crl*{-1, 1}$.  There exists a $2$-Lipschitz function $f(w; z)$ and a distribution $\cD$ over $\cZ$ such that: 
\begin{enumerate}[label=$(\alph*)$, leftmargin=8mm] 
\item The population loss $F(w) \ldef{} \En_{z \sim \cD}\brk{f(w; z)}$ is convex in $w$. Furthermore, $w^* = 0$ is the unique minimizer of the population loss $F(w)$. \item For any $\eta$ and $T$, the point $\wgd[\eta, T]$ returned by running GD with a step size of $\eta$ for $T$ steps satisfies  
\begin{align*} 
	\En \brk*{F(\wgd[\eta, T])} - \inf_{w \in \bbR} F(w) \geq \Omega\prn[\Big]{\frac{\eta T}{n}}. 
\end{align*} 
\end{enumerate} 
\end{lemma} 
\begin{proof} For $z \in \crl{-1, 1}$ and $w \in \bbR$, define the instance loss function $f$ as 
\begin{align*} 
	f(w; z) \ldef{} \prn[\big]{\frac{1}{4 \sqrt{n}} + z} \abs*{w}. 
\end{align*} 
Define the distribution $\cD$ such that $z = +1$ or $z=-1$ with probability $1/2$ each. Clearly, 	$f(w; z)$ is $2$-Lispchitz w.r.t. $w$ for any $z \in \cZ$. Furthermore, for any $w \in \bbR$, 
\begin{align*} 
	F(w) = \En_{z \sim \cD} \brk*{ f(w; z) } &= \En_{z} \brk[\Big]{ \prn[\big] { \frac{1}{4 \sqrt{n}} + z} \abs*{w} } = \frac{1}{4 \sqrt{n}} \abs{w}, 
\end{align*} where the last equality holds because $\En \brk*{z} = 0$. Thus, $F(w)$ is convex in $w$ and $w^* = 0$ is the unique minimizer of the population loss $F(w)$. This proves part-(a).

We now provide a lower bound for GD algorithm. Let $S = \crl*{z_i}_{i=1}^n$ denote a dataset of size $n$ sampled i.i.d.\ from $\cD$. The update rule for GD algorithm from \pref{eq:GD} implies that 
\begin{align*}
w^\text{GD}_{t + 1} &\leftarrow w^\text{GD}_{t} - \eta \cdot \text{sign}\prn{{w^\text{GD}_{t}}} \cdot \prn[\big]{\frac{1}{4\sqrt{n}} + \frac{1}{n} \sum_{i=1}^n z_i},  \numberthis \label{eq:GD_lower_bound_update_rule} 
\end{align*} 
and finally the returned point is given by $\wgd[\eta, T] = \frac{1}{n} \sum_{t=1}^T w^{\text{GD}}_t$. 

For $i \in [n]$, define the random variables $y_i = (1 - z_i) / 2$. Note that $y_i \sim \cB(1/2)$, and thus \pref{lem:rademacher_anti_concen} implies that 
\begin{align*}
\sum_{i=1}^n y_i \geq \frac{n}{2} + \frac{\sqrt{n}}{4}, 
\end{align*} with probability at least $1 / 15 e$. Rearranging the terms, we get that $ \sum_{i=1}^{n} z_i \leq - \sqrt{n} / 2$ with probability at least $1 / 15 e$.  Plugging this in \pref{eq:GD_lower_bound_update_rule}, we have that with probability at least $1 / 15 e$, for all $t \geq 0$, 
\begin{align*}
w^\text{GD}_{t + 1} &\geq w^\text{GD}_{t} + \frac{\eta}{4 \sqrt{n}}  \text{sign}\prn{w^\text{GD}_{t}}. 
\end{align*} 

Without loss of generality, assume that $w_1 > 0$. In this case, the above update rule implies that 
\begin{align*}
w_t &\geq w_1 + \frac{t \eta}{4 \sqrt{n}} 
\end{align*}
for all $t \geq 0$, which further implies that $$\wgd[\eta, T] = \frac{1}{T} \sum_{t=1}^T w_t \geq w_1  + \frac{\eta (T - 1)}{8 \sqrt{n}}. $$ Thus, 
\begin{align*}
F\prn{\wgd[\eta, T]} - \inf_{w} F(w) &= F\prn{\wgd[\eta, T]} \geq \frac{w_1}{4 \sqrt{n}} + \frac{\eta (T-1)}{32 n}, 
\end{align*} 
giving us the desired lower bound on the performance guarantee of the returned point $\wgd[\eta, T]$. The final in-expectation statement follows by observing that the above holds with probability at least $1/15e$. The proof follows similarly when $w_1 \leq 0$.  \end{proof} 

We next prove the lower bound in \pref{eq:app_GD_lower_bound_1} which follows by combining the lower bound construction from \pref{thm:amir_lower_bound} and \pref{lem:GD_lower_bound_ours}.   
\begin{theorem} 
\label{thm:GD_fails_lb} 
 Fix any $n \geq 200$. There exists a $3$-Lipschitz function $f(w; z)$ and a distribution $\cD$ over $z$, such that:  
\begin{enumerate}[label=$(\alph*)$, leftmargin=8mm] 
\item The population loss $F(w) \ldef{} \En_{z \sim \cD}\brk{f(w; z)}$ is convex in $w$. Furthermore, there exists a $w^* \in \argmin_{w} F(w)$ such that $\nrm{w^*} = O(1)$. 
\item For any $\eta \in [1/n^2, 1)$ and $T \geq 1$, the point $\wgd[\eta, T]$ returned by running GD with a step size of $\eta$ for $T$ steps has excess risk  
\begin{align*} 
\En \brk*{F(\wgd[\eta, T])} - \inf_{\w \in \reals^d} F(\w) = \Omega\prn[\Big]{ \frac{1}{\log^4(n)} \min\crl[\Big]{ \eta \sqrt{T} + \frac{1}{\eta T} + \frac{\eta T}{n},1}}. \numberthis \label{eq:app_GD_lower_bound} 
\end{align*} 
\end{enumerate} 
\end{theorem} 
\begin{proof} 
We first define some additional notation. Let $w\und{1}$, $z\und{1}$, $f\und{1}$ and $\cD\und{1}$ denote the optimization variable, the data sample, the instance loss and the distribution over $z\und{1}$ corresponding to the lower bound construction in \pref{thm:amir_lower_bound_modified}. Additionally, let $F\und{1}(w\und{1})$ denote the corresponding population loss under the distribution $D\und{1}$. We note that the function $f\und{1}$ is $1$-Lipschitz in $w \und {1}$ for any $z\und{1}$. Furthermore,\pref{thm:amir_lower_bound_modified} implies that $F\und{1}(w\und{1})$ is convex in $w\und{1}$, there exists a minimizer $w\und{1}^* \in \argmin_{w\und{1}} F\und{1}(w\und{1})$ such that $\nrm{w \und{1}^*} = O(1)$, and that for any $\eta \in [1/n^2, 1)$ and $T \in [1, n^3)$, the point $\wgd_{(1)}[\eta, T]$ returned by running GD algorithm with step size $\eta$ for $T$ time steps satisfies
\begin{align*}
\En \brk[\big]{ F\und{1} (\wgd_{(1)}[\eta, T])} - \inf_{\w\und{1}} F (\w\und{1}) = \Omega\prn[\Big]{\frac{1}{\log^4(n)} \min\crl[\Big]{ \eta \sqrt{T} + \frac{1}{\eta T},1}}.  \numberthis \label{eq:part1_lb1} 
\end{align*}

Similarly, let $w\und{2}$, $z\und{2}$, $f\und{2}$ and $\cD\und{2}$ denote the corresponding quantities for the lower bound construction in \pref{lem:GD_lower_bound_ours}, and let  $F\und{2}(w\und{2})$  denote the corresponding population loss under the distribution $\cD\und{2}$. We note that the function $f\und{2}$ is $2$-Lipschitz in $w \und {2}$ for any $z\und{2}$. Furthermore, \pref{lem:GD_lower_bound_ours} implies that $F\und{2}(w\und{2})$ is convex in $w\und{2}$ with  $w^* = 0$ being the unique minimizer, and that for any $\eta$ and $T$ the point $\wgd_{(2)}[\eta, T]$ returned by running GD algorithm with step size $\eta$ for $T$ time steps satisfies
\begin{align*}
\En \brk[\big]{ F\und{2} (\wgd_{(2)}[\eta, T])} - \inf_{\w\und{2}} F (\w\und{2}) = \Omega \prn[\Big]{\frac{\eta T}{n}}.  \numberthis \label{eq:part1_lb2} 
\end{align*}

Our desired lower bound follows by combining the lower bound constructions from \pref{thm:amir_lower_bound} and \pref{lem:GD_lower_bound_ours} respectively.\paragraph{Lower bound construction.}  Consider the following learning setting: 
\begin{enumerate}[label=$\bullet$, leftmargin=8mm] 
\item \textit{Optimization variable}: $w$ is defined as the concatenation of the variables $\prn{w\und{1}, w\und{2}}$.
\item \textit{Data instance}: $z$ is defined as the concatenation of the data instances $\prn{z\und{1}, z\und{2}}$. 
\item \textit{Data distribution}: $\cD$ is defined as $\cD\und{1} \times \cD\und{2}$, i.e. $z\und{1}$ and $z\und{2}$ are sampled independently from $\cD\und{1}$ and $\cD\und{2}$ respectively.  
\item \textit{Loss function}: is defined as 
\begin{align*}
f(w; z)  \ldef{} f\und{1}(w\und{1}; z\und{1}) + f\und{2}(w\und{2}; z\und{2}), 
\end{align*} 
Additionally, define the population loss $F(w) \ldef{} \En_{z \sim \cD} \brk*{f(w; z)}$. 
\end{enumerate}

Since, $f\und{1}$ is $1$-Lipschitz in $w\und{1}$ and  $f\und{2}$ is $2$-Lipschitz in $w\und{2}$, we have that the function $f$ defined above is $3$-Lipschitz in $w$. Furthermore, the population loss 
\begin{align*}
F(w) = F\und{1}(w\und{1}) + F\und{2}(w\und{2}),
\end{align*}
is convex in $w$ as both $F\und{1}(w\und{1})$ and $F\und{2}(w\und{2})$ are convex functions. Furthermore, since the components $\prn{w\und{1}, w\und{2}}$ do not interact with each other in the function $f$, we have that there exists a $w^* \in \argmin_{w} F(w)$ such that: 
\begin{align*}
\nrm{w^*} \leq \nrm{w^*\und{1}} + \nrm{w^*\und{2}} = O(1),  
\end{align*} where $w^*\und{1}$ denotes a minimizer of $F\und{1}$ with $\nrm{w^*\und{1}} = O(1)$ and $w^*\und{1} = 0$ denotes the unique minimizer of $F\und{2}$.

\paragraph{GD lower bound.} From the construction of the function $f$, we note that the variables $w\und{1}$ and $w\und{2}$ are updated independent to each other by GD algorithm. Thus, for any $\eta \in [1/n^2, 1)$ and $T \in [1, n^3)$, the point $\wgd[{\eta, T}]$ returned by running GD algorithm on the function $f$ with step size $\eta$ for $T$ time steps satisfies 
\begin{align*}
\En \brk*{F\prn[\big]{\wgd[{\eta, T}]}} - \min_{w} F(w)  &= \En \brk[\big]{F\und{1}\prn[\big]{\wgd_{(1)}[\eta, T]} + F\prn[\big]{\wgd_{(2)}[\eta, T]}} - \min_{w\und{1}, w\und{2}} F\und{1}(w\und{1})  + F\und{2}(w\und{2})  \\ 
&=   \En \brk[\big]{ F\und{1}\prn[\big]{\wgd_{(1)}[{\eta, T}]} }   - \min_{w\und{1}} F\und{1}(w\und{1}) +  \En \brk[\big]{ F\prn[\big]{\wgd_{(2)} [{\eta, T}]} } - \min_{w\und{2}} F(w\und{2})  \\
&\overeq{\proman{1}} \Omega\prn[\Big]{\frac{1}{\log^4(n)} \min\crl[\Big]{ \eta \sqrt{T} + \frac{1}{\eta T},1}} + \Omega\prn[\Big]{\frac{\eta T}{n}} \\
&= \Omega\prn[\Big]{\frac{1}{\log^4(n)} \min\crl[\Big]{ \eta \sqrt{T} + \frac{1}{\eta T} + \frac{\eta T}{n},1}},
\end{align*}
where the lower bound in $\proman{1}$ follows from combining the lower bounds in \pref{eq:part1_lb1}  and \pref{eq:part1_lb2}. 

Finally, we note that when $\eta \in [1/n^2, 1)$ and $T \geq n^3$, we have that 
\begin{align*}
\En \brk*{F(\wgd[{\eta, T}])} - \min_{w} F(w)  &\geq  \En \brk[\big]{F(\wgd_{(2)}[\eta, T])} - \min_{\und{2}} F\und{2}(w\und{2})  \\ 
&=  \Omega\prn[\Big]{\frac{\eta T}{n}} \\
&= \Omega(1), 
\end{align*} where the second line follows by using the lower bound in \pref{eq:part1_lb2}, and the last line holds for $T > n^3$ because $\eta \geq 1/n^2$. Thus, the desired lower bound holds for all $\eta \in [1/n^2 , 1]$ and $T \geq 1$. 
\end{proof} 

\subsection{Lower bound for small step size ($\eta < 1/64 n^{5/4}$)}  

In the following, we provide a learning setting for which GD run with step size $\eta < 1/64 n^{5/4}$ has lower bound of $\Omega(1/n^{3/8})$. 
 
\begin{lemma}
\label{lem:small_step_size_lb}
 Let $w \in \bbR$ denote the optimization variable, and $z$ denote a data sample. There exists a function $f(w; z)$ and a distribution $\cD$ over $z$ such that: 
\begin{enumerate}[label=(\alph*), leftmargin=8mm]  
\item The population loss $F(w) \ldef{} \En_{z \sim \cD} \brk*{f(w; z)}$ is $1$-Lipschitz and convex in $w$. Furthermore, there exists a point $w^* \in \argmin_{w} F(w)$ such that $\nrm*{w^*} = 1$. 
\item The variance of the gradient is bounded, i.e. $\En_{z \sim \cD} \brk{\nrm*{\grad f(w; z) - \grad F(w)}^2} \leq 1$ for any $w \in \bbR$. 
\item If $\eta < 1/ 64 n^{5/4}$, then for any $T > 0$, the point $\wgd_T$ returned by running GD algorithm with step size $\eta$ for $T$ steps satisfies 
\begin{align*} 
	\En \brk{ F({{\widehat{w}}^{\text{GD}}}_T)} - \min_{w \in \bbR} F(w) = \Omega\prn[\Big]{\frac{1}{n^{3/8}}}. \numberthis \label{eq:SGD_upper_bound_fn} 
\end{align*}  
\end{enumerate} 
\end{lemma} 
\begin{proof} Before delving into the construction of the function $f$, we first define some  auxiliary functions and  notation. Define the kink function $h(w)$ as 
\begin{align*} 
h(w) \ldef{} \begin{cases} 0 & \text{if} \quad w < 0 \\   
- n^{5/8} w & \text{if} \quad  0 \leq w < \frac{1}{64 n^{5/4}} \\  
~~~ n^{5/8} w - \frac{2}{64 n^{5/8}} & \text{if} \quad \frac{1}{ 64 n^{5/4}} \leq w \leq \frac{3}{ 64 n^{5/4}} \\  
- n^{5/8} w + \frac{4}{64 n^{5/8}}  & \text{if} \quad \frac{3}{64 n^{5/4}} < w   \leq \frac{1}{16 n^{5/4}} \\ 
0 & \text{if} \quad \frac{1}{16 n^{5/4}} \leq w 
\end{cases},  \numberthis \label{eq:kink_definition} 
\end{align*} and the corresponding gradients $\grad h(w)$ as 
\begin{align*}
\grad h(w) \ldef{} \begin{cases} 0 & \text{if} \quad w < 0 \\ 
- n^{5/8} & \text{if} \quad  0 \leq w < \frac{1}{64 n^{5/4}} \\  
n^{5/8}& \text{if} \quad \frac{1}{ 64 n^{5/4}} \leq w \leq \frac{3}{ 64 n^{5/4}} \\ 
- n^{5/8}  & \text{if} \quad \frac{3}{64 n^{5/4}} < w   \leq \frac{1}{16 n^{5/4}} \\ 
0 & \text{if} \quad \frac{1}{16 n^{5/4}} \leq w  
\end{cases}.  \numberthis \label{eq:kink_grad_definition} 
\end{align*} 

Additionally, define the set 
\begin{align*}
H \ldef{} \crl*{\frac{1}{4} + \frac{1}{8 n^{5/4}},   \frac{1}{4} + \frac{2}{8 n^{5/4}}, \ldots, \frac{3}{4}}, 
\end{align*}
where the set $H$ has $4 n^{5/4}$ numbers from the interval $[1/4, 3/4]$ spaced at a distance of $1 / 8n^{5/4}$. 

We now present our learning setting:  
\begin{enumerate}[label=$\bullet$, leftmargin=8mm] 
\item \textit{Data sample:} $z$ consists of the tuple $(\beta, y)$ where $\beta \in H$ and $y \in \crl{-1, +1}$. 

\item \textit{Data Distribution:} $\cD$ over the instances $z = (\beta, y)$ is defined such that 
\begin{align*} 
	\beta \sim \text{Uniform}(H) \qquad \text{and} \qquad y \sim \text{Uniform}(\crl{-1, 1}),  \numberthis \label{eq:step_lb_dist} 
\end{align*} where $\beta$ and $y$ are sampled independent of each other. 

\item \textit{Loss function:} is defined as 
\begin{align*}
	f(w; z) &\ldef{} \frac{1}{n^{3/8}} \max\crl*{-w, -1} + y \cdot {h(w + \beta)},  \numberthis \label{eq:step_lb_func} 
\end{align*} 

Additionally, define the population loss $F(w) \ldef{} \En_{z \sim \cD} \brk*{f(w; z)}$. 
\end{enumerate} 

We next show the desired statements for this learning setting: 
\begin{enumerate}[label=$(\alph*)$, leftmargin=8mm]  
\item For the distribution $\cD$ defined in \pref{eq:step_lb_dist}, since $\En \brk*{y} = 0$ and $y$ is sampled independent of $\beta$, we have that the population loss  
\begin{align*} 
F(w) = \En_{z \sim \cD} \brk*{f(w; z)} = \frac{1}{n^{3/8}} \max \crl{-w, -1}.  
\end{align*} 

Clearly, $F(w)$ is $1$-Lipschitz and convex is $w$. Additionally, the point $w^* = 1$ is a minimizer of $F(w)$. 

\item We next bound variance of the stochastic gradient. For any $w \in \bbR$, we have 
\begin{align*}
\En_{z \sim \cD} \brk{\abs*{\grad f(w; z) - \grad F(w)}^2} &= \En_{(\beta, y)} \brk*{ \abs{ y \grad h(w + \beta) }^2} \\ 
&= \En_{\beta} \brk*{ \abs{ \grad h(w + \beta) }^2}  \\ 
&=  \En_{\beta} \brk[\Big]{ \indic\crl[\Big]{\beta \in \brk[\big]{w - \frac{1}{16 n^{5/4}}, w}} \cdot n^{5/4}} \\ 
&= n^{5/4} \Pr \prn[\big]{\beta \in \brk[\big]{w - \frac{1}{16 n^{5/4}}, w}}  
\end{align*} where the equality in the first line holds because $y \in \crl*{-1, 1}$, and the second line follows from the construction of the function $h$ which implies that $\abs{\grad h(w + \beta)} \leq n^{5/8}$ for any $w$ and $\beta$ (see \pref{eq:kink_grad_definition}). Using the fact that $\beta \sim \text{Uniform}(B)$ in the above, we get that  
\begin{align*} 
\En_{z \sim \cD} \brk{\abs*{\grad f(w; z) - \grad F(w)}^2} &\leq 1/4. 
\end{align*} 

\item We next show that GD algorithm when run with the step size of $\eta \leq 1/ 64 n^{5/4}$ fails to converge to a good solution. 

Define $\beta_{(1)}, \ldots, \beta_{(n)}$ such that $\beta_{(j)}$ denotes the $j$th smallest item in the set $\crl*{\beta_1, \ldots, \beta_n}$, and define $y_{(j)}$ as the corresponding $y$ variable for the random variable $\beta_{(j)}$. An application of \pref{lem:technical_lemma_4} implies that with probability at least $1 - 2 / n^{1/4}$, there exists a $\wh j \leq \ceil{{\log(n)}/{2 n ^{1/4}}}$ such that: 
\begin{enumerate}[label=$(\alph*)$, leftmargin=8mm]
\item[$\proman{1}$] $\beta_{(1)} \neq \beta_{(2)} \neq \beta_{(3)} \neq \cdots \neq \beta_{(\wh j + 1)}$, and 
\item[$\proman{2}$] $y_{(\wh j)} = +1$. 
\end{enumerate} 

In the following, we condition on the occurrence of the above two events. The key idea of the proof is that at the level of the empirical loss, the first $\wh j$ kinks (that arise as a result of the stochastic function $h$) would be isolated from each other due to event-$\proman{1}$. Thus, the norm of the gradient of the empirical loss would be bounded by $2 / n^{3/8}$ for all points before $\beta_{\wh j}$. At the same time, since $y_{(\wh j)} = +1$ from event-$\proman{2}$, the empirical loss would be flat in a small interval around $\beta_{\wh j}$. As we show in the following, when GD is run on the empirical loss with step size $\eta \leq 1/64 n^{5/4}$, some GD iterate will lie in this small flat region, and after that GD will fail to make progress because the gradient is $0$; hence outputting a bad solution. On the other hand, GD / SGD run with a large step size e.g. $1/\sqrt{n}$ will easily jump over these kinks and converge to a good solution. We illustrate this intuition in \pref{fig:emp_loss}, and provide the proof below. 

\begin{figure}[ht]
\begin{flushright}
{\includegraphics[width=1.1\textwidth]{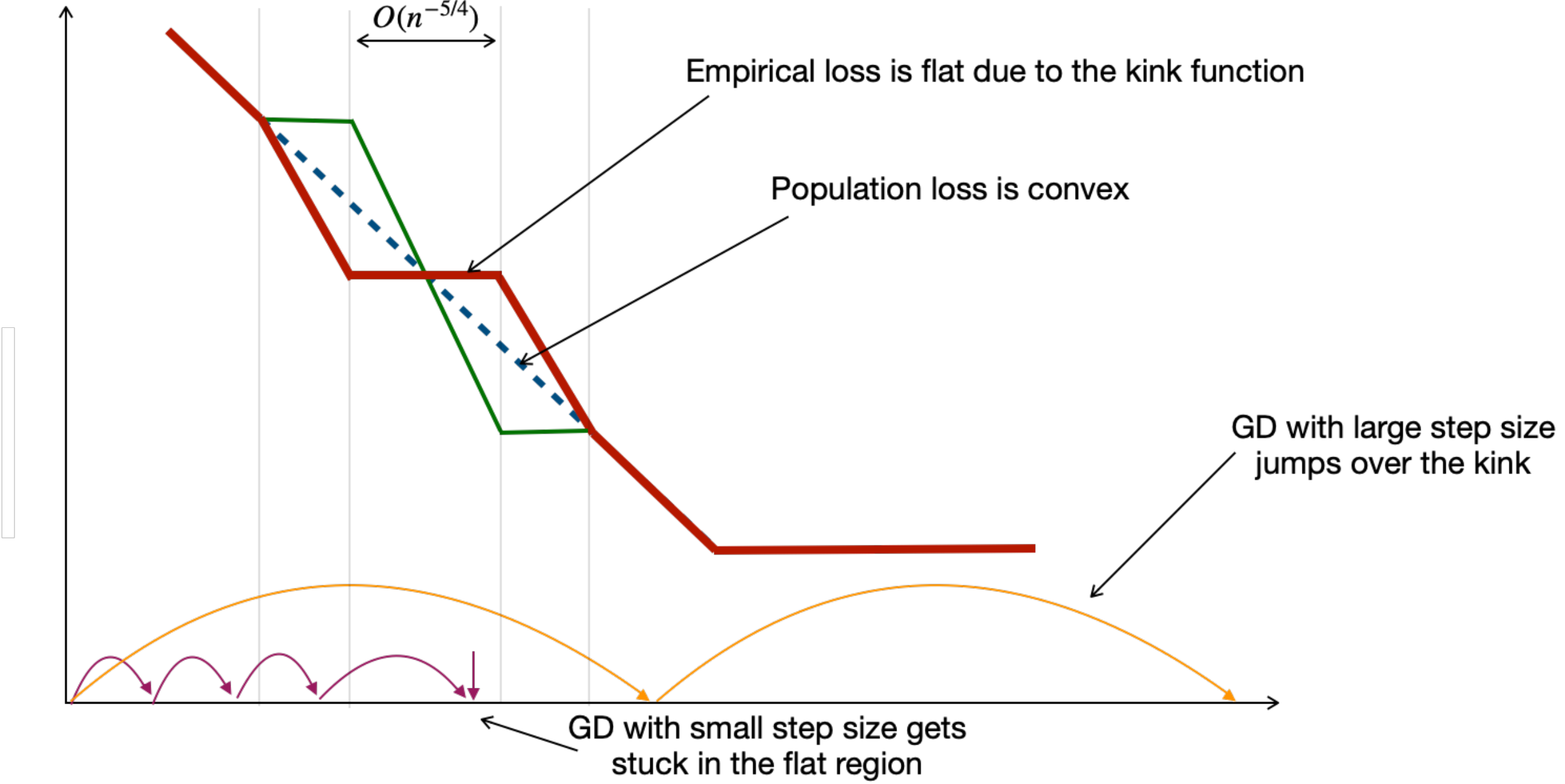}}
\end{flushright} 

\caption{(Picture not drawn according to scale) The solid red line shows the empirical loss induced by the kink function when $y = +1$, the solid green line shows the empirical loss when $y = -1$, and the dotted blue line shows the convex population loss. The empirical loss when $y = +1$ has gradient $0$ in a region of width $1/32 n^{5/4}$. Gradient descent with step size smaller than $1/64 n^{5/4}$, shown in the bottom, will get stuck in this flat region and thus fail to find a good solution. On the other hand, gradient descent with large step size will jump over the kink and find a good solution.} 
\label{fig:emp_loss}  	
\end{figure} 

Recall that the empirical loss on the dataset $S$ is given by: 
\begin{align*} 
	\wh F(w) = \frac{1}{n^{3/8}} \max \crl{-w, -1} + \frac{1}{n} \sum_{j=1}^n y_{j} \prn*{h(w + \beta_j)} 
\end{align*} 
Note that for any $\beta \in H$, the set of $w$ for which $h(w + \beta)$ is non-zeros is given by the interval $[\beta, \beta + 1/16n^{5/4}]$. Furthermore, any two numbers in the set $H$ are at least $1/8n^{5/4}$ apart from each other. Thus, the event-$\proman{1}$ above implies that non-zero parts of the functions $\crl{h(w + \beta_{(1)}), \ldots, h(w + \beta_{(\wh j)})}$, i.e. the first $\wh j$ kinks do not overlap with each other. This implies that for $w <  \beta_{(\wh j + 1)}$, 
\begin{align*}
\abs{\grad \wh F(w)} &\leq \frac{1}{n^{3/8}} + \frac{1}{n} \sum_{j=1}^n \abs{\grad {h(w + \beta_j)}} \\ &\leq \frac{1}{n^{3/8}} + \frac{n^{5/8}}{n} \leq \frac{2}{n^{3/8}},  \numberthis \label{eq:gradient_lipschitz}  
\end{align*}
where the first inequality in the above follows from Triangle inequality and the second inequality holds because at most one of the $n$ terms in the stochastic component $\sum_{j=1}^n y_{j} \prn*{h(w + \beta_j)}$ above is non-zero for $w <  \beta_{(\wh j + 1)}$. Thus, $\wh F(w)$ is $1$-Lipschitz for $w \in [0, \beta_{(\wh j + 1)})$. 

Next, the event-$\proman{2}$ above implies that $y_{(\wh j)} = 1$. Define the interval $\wt \cW \ldef{} (\beta_{(\wh j)} + 1/64 n^{5/4}, \beta_{(\wh j)} + 3/64 n^{5/4})$ and note that for any $w \in \wt \cW$, 
\begin{align*}
\grad \wh F(w) &= - \frac{1}{n^{3/8}} + \frac{1}{n} \grad h(w + \beta_{(\wh j)}) \\ 
 &= - \frac{1}{n^{3/8}}  + \frac{1}{n} \cdot n^{5/8} = 0,  \numberthis \label{eq:gradient_flat} 
\end{align*} where the first equality holds because $y_{(\wh j)} = +1$ and using the event-$\proman{1}$ above. This implies that the empirical loss $\wh F(w)$ has $0$ gradient for $w \in \wt \cW$, and thus GD algorithm will stop updating if any iterate reaches $\wt \cW$.

In the following, we will show that for $\eta \leq 1/ 64 n^{5/4}$, GD algorithm is bound to get stuck in the interval $\wt \cW$, and will thus fail to find a good solution. Consider the dynamics of the GD algorithm:
\begin{align*} 
w_{t + 1} \leftarrow w_{t} - \eta \grad \wh F(w_{t}), 
\end{align*} 
where the initial point  $w_1 = 0$. Suppose there exists some time $\tau$ for which  $w_{\tau} > \beta_{(\wh j + 1)}$, and let $t_0 > 0$ denote the smallest such time. Thus, for any $t < t_0$ and $\eta \leq 1/64 n^{5/4}$, we have that 
\begin{align*}
\abs{w_{t + 1} - w_{t}} &= \eta \abs{ \grad \wh F(w)} \leq \eta \leq 1/64 n^{5/4}, 
\end{align*}
where the first inequality holds due to \pref{eq:gradient_lipschitz}.  This implies that any two consecutive iterates produced by the GD algorithm for $t < t_0$ are at most $1/64 n^{5/4}$ apart from each other. However, note that the interval $\wt \cW  \subseteq  [0, \beta_{(\wh j + 1)})$ and has width of $2/64 n^{5/4}$. Thus, there exists some time $t' \leq t_0$ for which $w_{t'}$ will lie in the set $\wt \cW$. However, recall that $\grad F(w) = 0$ for any $w \in \wt \cW$ as shown in \pref{eq:gradient_flat}. Thus, once $w_{t'} \in \wt \cW$, GD will not update any further implying that for all $t > t'$, $w_{t} = w_{t'}.$ This shows via contradiction that no such time $\tau$ exists for which $w_{\tau} > \beta_{(\wh j + 1)}$. 

Hence for $\eta < 1/64 n^{5/4}$, all iterates $\crl{w_t}_{t \geq 0}$ generated by the GD algorithm will lie in the set $[0, \beta_{(\wh j + 1)})$. Thus, for any $T > 1$, the returned point $\wgd_T$ satisfies
\begin{align*}
\wgd_T \leq \beta_{(\wh j + 1)} \leq \frac{3}{4}, 
\end{align*} where the second inequality holds because $\beta_{(\wh j + 1)} \in H$ and is thus smaller than $3/4$. Thus, 
\begin{align*}
 F({{\widehat{w}}^{\text{GD}}}_T) - \min_{w \in \bbR} F(w) \geq \frac{1}{4 n^{3/8}}. \numberthis \label{eq:lower_bound_flat_final} 
\end{align*}

Since, the events $\proman{1}$ and $\proman{2}$ occur simultaneously with probability at least $1 - 2 / n^{1/4}$, we have that for any $T \geq 0$, 
\begin{align*} 
	\En \brk{ F({{\widehat{w}}^{\text{GD}}}_T)} - \min_{w \in \bbR} F(w) = \Omega\prn[\big]{\frac{1}{n^{3/8}}}. 
\end{align*} 
\end{enumerate} 
\end{proof}

\begin{lemma} 
\label{lem:technical_lemma_4}  
Suppose $(\beta_1, y_1), \ldots, (\beta_n, y_n)$ be $n$ samples drawn independently from $\textit{Uniform}(H) \times \text{Uniform}(\crl{-1, 1})$, where the set 
$H \ldef{} \crl{\tfrac{1}{4} + \tfrac{1}{8 n^{5/4}},   \tfrac{1}{4} + \tfrac{2}{8 n^{5/4}}, \ldots, \tfrac{3}{4}}$. Further, define $\beta_{(j)}$ to denote the $j$th smallest number in the set $\crl*{\beta_1, \ldots, \beta_n}$, and define $y_{(j)}$ as the corresponding $y$ variable for $\beta_{(j)}$. Then, with probability at least $1 - 2 / n^{1/4}$, there exists a $\wh j \leq  \ceil{{\log(n)}/{2 n ^{1/4}}}$ such that: 
\begin{enumerate}[label=$(\alph*)$, leftmargin=8mm]
\item[$\proman{1}$] $\beta_{(1)} \neq \beta_{(2)} \neq \beta_{(3)} \neq \cdots \neq \beta_{(\wh j + 1)}$, and 
\item[$\proman{2}$] $y_{(\wh j)} = +1$. 
\end{enumerate} 
\end{lemma}
\begin{proof} Let $k = \ceil{{\log(n)}/{2 n ^{1/4}}}$. We give the proof by translating our problem into a ``ball and bins" problem. 

Let the set $H$ denote $m = 4 n^{5/4}$ distinct bins, and let there be $n$ distinct balls. Each of the $n$ balls are tossed independently into one of the $m$ bins drawn uniformly at random. For any $i, j$, the event where the ball $j$ is tossed into the bin $i$ corresponds to the event that $\beta_j = \tfrac{1}{4} + \frac{i}{8 n^{5/4}}$. In addition to this, the balls are such that they change color after they are tossed into a bin. In particular, after being tossed into a bin, a ball changes color to either red or blue with equal probability. For any $i, j$, the event where the ball $j$ takes red color corresponds to the event that $y_j = 1$; Similarly, blue color corresponds to $y = -1$. Thus, in the balls and bins model, the position of the $n$ balls in the bins and their corresponding colors after tossing reveals the value of the random variables $(\beta_1, y_1), \ldots, (\beta_n, y_n)$.

In the following, we first show that with probability at least $1 - \frac{2}{n^{1/4}}$, the first $k$ bins will have at most one ball each, and there will be some bin amongst the first $k$ bins that contains a red color ball. Define $E_{i}$ to denote the event that the bin $i$ gets more than one ball. Thus, 
\begin{align*}
	\Pr(E_i) &= 1 - \Pr (\text{Bin $i$ has $0$ or $1$ ball in it}) \\
					&= 1 - \prn[\Big]{1 - \frac{1}{n^{5/4}}}^n -  n \prn[\Big]{1 - \frac{1}{n^{5/4}}}^{n-1} \frac{1}{n^{5/4}} \\ 					
					&\leq  1 - \prn[\Big]{1 - \frac{1}{n^{1/4}}} - \frac{1}{n^{1/4}}\prn[\Big]{1 - \frac{n - 1}{n^{5/4}}} \\
					&= \frac{n - 1}{n^{3/2}} \leq \frac{1}{\sqrt{n}}  \numberthis \label{eq:E_i_bound_bnb}
\end{align*} where the first inequality in the above follows from the fact that $(1- \alpha)^n \leq 1 - \alpha n$ for any $\alpha \geq -1$ and $n \geq 1$. Let $A_k$ denote the event that there exists some bin among the first $k$ bins that has more than one ball. Taking a union bound, we get   
\begin{align*}
\Pr \prn{A_k} = \Pr \prn{\cup_{i=1}^{k} E_i}  &\leq \sum_{i=1}^{k} \Pr \prn*{E_i}  \leq \frac{k}{\sqrt{n}}. 
\end{align*}

Next, let $B_k$ denote the event that there is no red ball in the first $k$ bins after $n$ tosses. When we throw a ball, the probability that it will fall in the first $k$ bins is given by ${k}/{n^{5/4}}$, and the probability that it takes the color red is ${1}/{2}$. Since the bin chosen and the final color are independent of each other for each ball, the probability that a ball falls into the first $k$ bins and is of red color is given by $k/2n^{5/4}$. Furthermore, since the balls are thrown independent to each other, the probability that if we throw $n$ balls, no ball falls into the first $k$ bins that takes the red color is given by 
\begin{align*}
\Pr \prn*{B_k} &= \prn[\Big]{1 - \frac{k}{2 n^{5/4}}}^n. 
\end{align*}

Finally, let us define $C_k$ to denote the event that there is at most one ball per bin amongst the first $k$ bins, and that there exists a bin amongst the first $k$ bins with a red ball in it. By the definition of the events $A_k$ and $B_k$, we note that 
\begin{align*}
\Pr \prn*{C_k} &= \Pr \prn*{A_k^{c} \cap B_k^{c}} \\
						&= 1 - \Pr \prn*{A_k \cup B_k}  \\ 
						  &\geq 1 - \Pr(A_k) - \Pr(B_k) \\ 
						  &\geq 1 - \frac{k}{\sqrt{n}} - \prn[\Big]{1 - \frac{k}{2n ^{5/4}}}^n \\ 
						  &\geq 1 - \frac{k}{\sqrt{n}} - e^{- k/ 2 n^{1/4}}, 
\end{align*} where the first inequality in the above follows from the union bound, the second inequality holds by plugging in the corresponding bounds for $P(A_k)$ and $P(B_k)$, and the last line is due to the fact that $(1 + \alpha)^n \leq e^{\alpha n}$ for any $\alpha$. Plugging in the value of $k = \ceil{{\log(n)}/{2 n ^{1/4}}}$ in the above, we get that 
\begin{align*} 
\Pr \prn*{C_k} &\geq 1 - \frac{2}{n^{1/4}}. 
\end{align*} 
 
Thus, with probability at least $1 - 2 / n^{1/4}$, the first $k$ bins have at most one ball each, and there exists some bin $\wh j \in [k]$ that contains a red ball; when this happens, the corresponding random variables  $(\beta_1, y_1), \ldots, (\beta_n, y_n)$ satisfy: 
\begin{enumerate}[label=$(\alph*)$, leftmargin=8mm] 
\item $\beta_{(1)} \neq \beta_{(2)} \neq \beta_{(3)} \neq \cdots \neq \beta_{(\wh j + 1)}$, and 
\item $y_{(\wh j)} = +1$. 
\end{enumerate} 
\end{proof} 

\subsection{Proof of \pref{thm:GD_vs_SGD_comparison}} 
In this section, we combine the lower bound constructions in \pref{thm:GD_fails_lb} and \pref{lem:small_step_size_lb} to provide an instance of a SCO problem on which, for any step size $\eta$ and time steps $T$, GD has the lower bound of $\Omega(1/n^{5/12})$. 

\begin{proof}[{Proof of \pref{thm:GD_vs_SGD_comparison}} ] Throughout the proof, we assume that $n \geq 200$ and the initial point $w_1 = 0$. 

We first define some additional notation: Let $w\und{1}$, $z\und{1}$, $f\und{1}$ and $\cD\und{1}$ denote the optimization variable, the data sample, the instance loss and the distribution over $z\und{1}$ corresponding to the lower bound construction in \pref{thm:GD_fails_lb}. The statement of  \pref{thm:GD_fails_lb} implies that:  
\begin{enumerate}[label=(\alph*), leftmargin=8mm] 
\item $f\und{1}(w\und{1}; z\und{1})$ is $3$-Lipschitz in the variable $w\und{1}$ for any $z\und{1}$. 
\item The population loss $F\und{1}(w\und{1}) \ldef \En_{z\und{1} \sim \cD\und{1}} \brk{f(w\und{1}; z\und{1})}$ is $3$-Lipschitz and convex in $w\und{1}$. Additionally, there exists a point $w^*\und{1} \in \argmin_{w\und{1}} F\und{1}(w\und{1})$ such that $\nrm{w^*\und{1}} = O(1)$. 
\item For any  $\eta \in [1/n^2, 1)$ and $T \geq 1$,  the point $\wgd_{(1)}[\eta, T]$ returned by running GD algorithm with a step size of $\eta$ for $T$ time steps satisfies: 
\begin{align*} 
	\En \brk[\big]{F\und{1}\prn[\big]{\wgd_{(1)} [\eta, T]}} - \min_{w\und{1}} F\und{1}(w\und{1}) = \Omega\prn[\Big]{ \frac{1}{\log^4(n)} \min\crl[\Big]{ \eta \sqrt{T} + \frac{1}{\eta T} + \frac{\eta T}{n},1}}. \numberthis \label{eq:GD_final_lb1}  
\end{align*} 
\end{enumerate} 

Similarly, let $w\und{2}$, $z\und{2}$, $f\und{2}$ and $\cD\und{2}$ denote the optimization variable, the data sample, the instance loss and the distribution over $z\und{2}$ corresponding to the lower bound construction in \pref{lem:small_step_size_lb}. The statement of \pref{lem:small_step_size_lb} implies that:  
\begin{enumerate}[label=(\alph*), leftmargin=8mm] 
\item The population loss $F\und{2}(w\und{2}) \ldef{} \En_{z\und{2} \sim \cD\und{2}} \brk*{f(w\und{2}; z\und{2})}$ is $1$-Lipschitz and convex in $w\und{2}$.  Additionally, there exists a point $w^*\und{2} \in \argmin_{w\und{2}} F(w\und{2})$ such that $\nrm{w^*\und{2}} \leq 1$. 
\item The variance of the gradient is bounded, i.e.  for any $w\und{2}$,  
\begin{align*}
\En_{z\und{2} \sim \cD\und{2}} \brk*{\nrm{\grad f\und{2}(w\und{2}; z\und{2}) - \grad F\und{2}(w\und{2})}^2} \leq 1.  \numberthis \label{eq:GD_final_lb2}  
\end{align*}
\item If $\eta < 1/ 64 n^{5/4}$, then for any $T > 0$, the point $\wgd_{(2)}[ \eta, T]$ returned by running GD algorithm with step size $\eta$ for $T$ steps satisfies 
\begin{align*} 
	\En \brk{ F\und{2}\prn[\big]{\wgd_{(2)} [\eta, T]}} - \min_{w\und{2}} F\und{2}(w\und{2}) = \Omega\prn[\Big]{\frac{1}{n^{3/8}}}. \numberthis \label{eq:GD_final_lb8} 
\end{align*}  \end{enumerate} 

\paragraph{Lower bound construction.}  Consider the following learning setting: 
\begin{enumerate}[label=$\bullet$, leftmargin=8mm] 
\item \textit{Optimization variable}: $w$ is defined as the concatenation of the variables $\prn{w\und{1}, w\und{2}}$. 
\item \textit{Data instance}: $z$ is defined as the concatenation of the data instances $\prn{z\und{1}, z\und{2}}$. 
\item \textit{Data distribution}: $\cD$ is defined as ${\cD\und{1} \times \cD\und{2}}$, i.e. $z\und{1}$ and $z\und{2}$ are sampled independently from $\cD\und{1}$ and $\cD\und{2}$ respectively.  
\item \textit{Loss function}: is defined as 
\begin{align*}
f(w; z)  \ldef{} f\und{1}(w\und{1}; z\und{1}) + f\und{2}(w\und{2}; z\und{2}), 
\end{align*} 
Additionally, define the population loss $F(w) \ldef{} \En_{z \sim \cD} \brk*{f(w; z)} = F\und{1}(w\und{1}) + F\und{2}(w\und{2})$.  
\end{enumerate}

\paragraph{Excess risk guarantee for SGD.} We first show that the above learning setting is in SCO. Note that the population loss 
\begin{align*}
F(w) = F\und{1}(w\und{1}) + F\und{2}(w\und{2}). 
\end{align*}

Clearly, $F(w)$ is convex in $w$ since the functions $F\und{1}(w\und{1})$ and $ F\und{2}(w\und{2})$ are convex in $w\und{1}$ and $w\und{2}$ respectively. Next, note that for any $w$ and $w'$, 
\begin{align*}
\abs{F(w) - F(w')} &\overleq{\proman{1}} \abs{F\und{1}(w\und{1}) - F\und{1}(w'\und{1})}  +  \abs{F\und{2}(w\und{2}) - F\und{2}(w'\und{2})} \\
&\overleq{\proman{2}}  3 \nrm{ w\und{1} - w'\und{1}} + \nrm{w\und{2} - w'\und{2}} \\
&\leq 3 \prn[\big]{\sqrt{\nrm{ w\und{1} - w'\und{1}}^2} + \sqrt{\nrm{w\und{2} - w'\und{2}}^2}} \\
&\overleq{\proman{3}} 3 \sqrt{2} \cdot \sqrt{\nrm{ w\und{1} - w'\und{1}}^2 +  \nrm{w\und{2} - w'\und{2}}^2} \\
&\overeq{\proman{4}} 3 \sqrt{2} \cdot \sqrt{\nrm*{w - w'}^2} \\ 
&=3 \sqrt{2} \nrm*{w - w'}, 
\end{align*} where the inequality $\proman{1}$ above follows from Triangle inequality and the inequality $\proman{2}$ is given by the fact that $F\und{1}$ is $3$-Lipschitz in $w\und{1}$ and that $F\und{2}$ is $1$-Lipschitz in $w\und{2}$. The inequality in $\proman{3}$ follows from an application of Jensen's inequality and using concavity of square-root. Finally, the equality in $\proman{4}$ is by construction of the variable $w$ as concatenation of the variables $\prn{w\und{1}, w\und{2}}$. Thus, the population loss $F(w)$ is $3 \sqrt{2}$-Lipschitz in $w$. 

We next show a bound on the variance of the gradient. Note that for any $w$, 
\begin{align*}
\hspace{1.5cm} & \hspace{-1.5cm} \En_{z \sim \cD} \brk*{\nrm{\grad f(w; z) - \grad F(w)}^2} \\ 
&= \En_{z \sim \cD} \brk*{\nrm{\grad f\und{1}(w\und{1}; z\und{1}) + \grad f\und{2}(w\und{2}; z\und{2}) - \grad F\und{1}(w\und{1}) - \grad F\und{2}(w \und{2})}^2} \\ 
&\overleq{\proman{1}}  2 \En_{z\und{1} \sim \cD\und{1}} \brk*{\nrm{\grad f\und{1}(w\und{1}; z\und{1}) - \grad F\und{1}(w\und{1})}^2}  \\
& \qquad \qquad\qquad  +  2\En_{z\und{2} \sim \cD\und{2}} \brk*{\nrm{\grad f\und{2}(w\und{2}; z\und{2})  - \grad F\und{2}(w \und{2})}^2} \\ 
&\overleq{\proman{2}} 72 +  2\En_{z\und{2} \sim \cD\und{2}} \brk*{\nrm{\grad f\und{2}(w\und{2}; z\und{2})  - \grad F\und{2}(w \und{2})}^2} \overleq{\proman{3}} 74, 
\end{align*} where the inequality  $\proman{1}$ above follows from the fact that $(a + b)^2 \leq 2a^2 + 2 b^2$ for any $a, b$. The inequality $\proman{2}$ holds because the function $f\und{1}(w\und{1}; z\und{1})$ is $3$-Lipchitz in $w\und{1}$ for any $z\und{1}$, and because $F\und{1}(w\und{1})$ is $3$-lipschitz in $w\und{1}$. Finally, the inequality $\proman{3}$ is due to the bound in \pref{eq:GD_final_lb2}. 

\par We next show that there exists a $w^* \in \argmin_{w} F(w)$ such that $\nrm*{w^*} = O(1)$. Since the components $\prn{w\und{1}, w\und{2}}$ do not interact with each other in the above construction of the loss function $f$, we note that the point $w^* = (w^*\und{1}, w^*\und{2})$ is a minimizer of $F(w) = F\und{1}(w\und{1}) + F\und{1}(w\und{1})$. This point $w^*$ satisfies 
\begin{align*}
\nrm{w^*} \leq \nrm{w^*\und{1}} + \nrm{w^*\und{2}} = O(1),  
\end{align*} where we used the fact that $\nrm{w^*\und{1}} = O(1)$ and $\nrm{w^*\und{2}} = O(1)$. 

Combining the above derived properties, we get that:
\begin{enumerate}[label=(\alph*), leftmargin=8mm] 
\item The population loss $F(w)$  is $3 \sqrt{2}$-Lipschitz and convex in $w$. 
\item There exists a point $w^* \in \argmin_{w} F(w)$ such that $\nrm*{w^*} = O(1)$. 
\item For any $w$, the gradient variance $\En_{z \sim \cD} \brk*{\nrm{\grad f(w; z) - \grad F(w)}^2}  \leq 72.$ 
\end{enumerate} 

Thus, as a consequence of \pref{thm:app_SGD_convergence}, we get that running SGD with step size $\eta = 1/\sqrt{n}$ and initialialization $w_1 = 0$, returns the point $\wh w^{\mathrm{SGD}}_n$ that satisfies 
\begin{align*}
\En \brk*{F(\wh w^{\mathrm{SGD}}_n)} - \min_{w} F(w) = O\prn[\Big]{\frac{1}{\sqrt{n}}}. \numberthis \label{eq:GD_final_lb3}
\end{align*}

\paragraph{Lower bound for GD.} We next show that  GD algorithm fails to match the performance guarantee of SGD in \pref{eq:GD_final_lb3}, for any step size $\eta$ and time step $T$. 

Let $\wgd[{\eta, T}]$ denote the point returned by running GD algorithm on the function $f$ with step size $\eta$ for $T$ steps. Since the components $w\und{1}$ and $w\und{2}$ do not interact with each other in the GD update step due to the construction of the function $f$, we have that  
\begin{align*}
\En \brk*{F(\wgd[{\eta, T}])} - \min_w F(w) &= \En \brk{F\und{1}(\wgd_{(1)} [\eta, T])} - \min_{w\und{1}} F\und{1}(w\und{1})  \\ 
& \qquad \qquad \qquad + \En \brk{F\und{2}(\wgd_{(2)} [\eta, T])} - \min_{w\und{2}} F\und{2}(w\und{2}) \numberthis \label{eq:GD_final_lb4} 
\end{align*} 

The key idea behind the lower bound for GD is that first components has a lower bound of $\Omega(1/n^{5/12})$ when $\eta$ is larger than $1/64n^{5/4}$. Specifically, in order to improve the excess risk bound over the rate of $1/n^{5/12}$ w.r.t. the variable $w\und{1}$, we need $\eta$ to be smaller than $1/64n^{5/4}$. However, any choice of $\eta < 1 / 64n^{5/4}$ fails to find a good solution w.r.t. the component $w\und{2}$.  We formalize this intuition by considering the two cases (a) $\eta < 1 / 64n^{5/4}$, and (b) $\eta \geq 1 / 64n^{5/4}$ separately below. 
\begin{enumerate}[label=$\bullet$, leftmargin=8mm]
\item \textbf{Case 1: $\mb{\eta < 1 / 64n^{5/4}}$.} Using the fact that $\En \brk{F\und{1}(\wgd_{(1)}[\eta, T])} - \min_{w\und{1}} F\und{1}(w\und{1})  \geq 0$ in \pref{eq:GD_final_lb4}, we get that 
\begin{align*}
\En \brk*{F(\wgd[{\eta, T}])} - \min_w F(w) &\geq \En \brk[\big]{F\und{2}\prn[\big]{\wgd_{(2)}[{ \eta, T}]}} - \min_{w\und{2}} F\und{2}(w\und{2}) \\ 
&= \Omega\prn[\Big]{\frac{1}{n^{3/8}}},  \numberthis \label{eq:GD_final_lb12}
\end{align*} where the inequality in the second line above is due to the lower bound in \pref{eq:GD_final_lb8} which holds for all $\eta < 1 / 64n^{5/4}$ and $T \geq 1$. 

\item  \textbf{Case 2: $\mb{\eta \geq 1 / 64n^{5/4}}$.} Using the fact that $\En \brk{F\und{2}\prn{\wgd_{(2)}[{ \eta, T}]}} - \min_{w\und{2}} F\und{2}(w\und{2}) \geq 0$ in \pref{eq:GD_final_lb4}, we get that 
\begin{align*} 
	\En \brk*{F(\wgd_{\eta, T})} - \min_w F(w) &\geq \En \brk[\big]{F\und{1}\prn[\big]{\wgd_{(1)}[{ \eta, T}]}} - \min_{w\und{1}} F\und{1}(w\und{1}).   \numberthis \label{eq:GD_lower_bound_fn3}
\end{align*}

The lower bound in \pref{eq:GD_final_lb1} suggests that for $\eta \in [1/n^2, 1)$ and $T \geq 1$, 
\begin{align*}
\En \brk[\big]{F\und{1}\prn[\big]{\wgd_{(1)}[{ \eta, T}]}} - \min_{w\und{1}} F\und{1}(w\und{1})	&= \Omega\prn[\Big]{ \frac{1}{\log^4(n)} \min\crl[\Big]{ \eta \sqrt{T} + \frac{1}{\eta T} + \frac{\eta T}{n},1}} \\
	&= \Omega\prn[\Big]{ \frac{1}{\log^4(n)} \min\crl[\Big]{\eta \sqrt{T} + \frac{1}{2\eta T} + \frac{1}{2\eta T} + \frac{\eta T}{n}, 1}} \\
	&\overeq{\proman{1}}  \Omega\prn[\Big]{ \frac{1}{\log^4(n)} \min\crl[\Big]{\eta \sqrt{T} + \frac{1}{2\eta T} + \frac{1}{\sqrt{2n}}, 1}} \numberthis \label{eq:GD_lower_bound_fn2} \\ 
	&\overeq{\proman{2}}  \Omega\prn[\Big]{ \frac{1}{\log^4(n)} \min\crl[\Big]{\eta^{1/3} + \frac{1}{\sqrt{n}}, 1}} 
\end{align*}
 where $\proman{1}$ follows from an application of the AM-GM inequality for the last two terms, and $\proman{2}$ holds by setting $T = 1/\eta^{4/3}$ which minimizes the expression in \pref{eq:GD_lower_bound_fn2}. Plugging the above lower bound in \pref{eq:GD_lower_bound_fn3}, we get that  
\begin{align*}
	\En \brk*{F\prn[\big]{\wgd[{\eta, T}]}} - \min_w F(w) &\geq \Omega\prn[\Big]{ \frac{1}{\log^4(n)} \min\crl[\Big]{\eta^{1/3} + \frac{1}{\sqrt{n}}, 1}}. 
\end{align*}

Finally, using the fact that $\eta \geq 1/ {64 n^{5/4}}$ in the above bound, we get  
\begin{align*}
	\En \brk*{F\prn[\big]{\wgd[{\eta, T}]}} - \min_w F(w) &\geq \Omega\prn[\Big]{ \frac{1}{\log^4(n)} \min\crl[\Big]{\frac{1}{n^{5/12}} + \frac{1}{\sqrt{n}}, 1}}. \numberthis \label{eq:GD_final_lb11}
\end{align*} 
\end{enumerate}

Combining the lower bound from \pref{eq:GD_final_lb12} and  \pref{eq:GD_final_lb11} for the two cases above, we get that for all $\eta \geq 0$ and $T \geq 1$, the point  $\wgd[{\eta, T}]$ returned by running GD algorithm on the function $f$ with step size $\eta$ for $T$ steps satisfies: 
\begin{align*}
	\En \brk*{F\prn[\big]{\wgd[{\eta, T}]}} - \min_w F(w) &\geq \Omega\prn[\Big]{ \frac{1}{\log^4(n)} \min\crl[\Big]{\frac{1}{n^{5/12}}, 1}}. 
\end{align*}
\end{proof}

\section{Missing proofs from \pref{sec:multi_pass}} \label{app:multi_pass} 
%!TEX root=../paper.tex

The pseudocode for multi-pass SGD algorithm given in \pref{alg:multi_pass_SGD} is slightly different from the description of the algorithm given at the beginning of \pref{sec:multi_pass}. In particular, at the start of every epoch, \pref{alg:multi_pass_SGD} uses the following projection operation:
\[\Pi_{w_1, B}(w) = \begin{cases}
   w & \text{ if }\|w - w_1\| \leq B \\
   w_1 + \frac{B}{\|w-w_1\|}(w-w_1) & \text{ otherwise.}
\end{cases}\]
This ensures that the iterate at the start of every epoch has bounded norm. Rest of the algorithm is the same as in the description in the main body.   
 
\begin{algorithm}[h]  
\caption{Multi-pass SGD algorithm}  
\begin{algorithmic}[1]
   \Require Dataset $S = \crl{z_i}_{i=1}^{n}$, number of passes $k$, initial point $w_1$. 
   \State Define $\newn \ldef{} n / 2$, $S_1 \ldef{} \crl*{z_i}_{i=1}^{\newn}$ and $S_2 \ldef{} S \setminus S_1$ 
   \State Initialize $\wh \cW \leftarrow \emptyset$ 
   \For{$j = 1, \ldots, k$} \Comment{Multiple passes} 
   \State $w_{\newn(j-1)  +1} \leftarrow \Pi_{w_1, B}(w_{\newn(j-1)  +1})$
   \State $\eta_j \leftarrow \frac{1}{\sqrt{n j}}$ 
   \For{$i = 1, 3, \ldots, \newn$} 
	\State $w_{\newn(j-1)+i+1} \leftarrow w_{\newn(j-1)+i} - \eta_{j} \grad f(w_{\newn(j-1) +i}; z_i)$. 
	\EndFor
	\State $\wh w_j \leftarrow \frac{1}{\newn j} \sum_{t=1}^{\newn j} w_j$ 
	\State $\wh \cW \leftarrow \wh \cW \cup \crl{\wh w_j}$ 
   \EndFor 
   \State Return the point $\wh w^{\text{MP}} \in \argmin_{w \in \wh \cW} F_{S_2}(W) \ldef{} \frac{1}{\newn} \sum_{i=1}^{\newn} f(w; z_{\newn + 1})$  \Comment{Validation} 
\end{algorithmic} 
\label{alg:multi_pass_SGD}  
\end{algorithm} 

\subsection{Proof of \pref{prop:multi_pass_works}} 
\begin{proof}  Assume that $f$ is bounded by $M$. Since $S_2$ is independent of $S_1$, we note $S_2$ is also independent of the set of points $\wh \cW$. Hoeffding's inequality (\pref{lem:hoeffding}) thus implies that, with probability at least $1 - \delta$, for all $w \in \wh \cW$, 
\begin{align*}
\abs{F_{S_2}(\w) - F(\w)} \leq M \sqrt{\frac{\log(2 k/\delta)}{n} }. 
\end{align*} 

Thus, the returned point $\wh w^{\mathrm{MP}} \in \argmin_{w \in \wh \cW} F_{S_2}(w)$ satisfies: 
\begin{align*}
F(\wh w^{\mathrm{MP}}) &\leq \min_{w \in \wh \cW} F(w) + 2M \sqrt{\frac{\log(2k /\delta)}{n} } \\
&\leq F(\wh w_1) + 2M \sqrt{\frac{\log(2k /\delta)}{n} }.  
\end{align*}

Observing that the point $\wh w_1$ denotes $\wh w^{\mathrm{SGD}}_{n/2}$, and converting the above high probability statement into an in-expectation result (since $f$ is bounded by $M$) gives us the desired statement.  
\end{proof}

\subsection{Proof of \pref{thm:multipass}}

For the upper bound of part (b) in \pref{thm:multipass}, we need the following slight generalization of \pref{thm:sgd_works_main}.

\begin{lemma}\label{lem:multipass_sgd_technical} 
Consider any SCO problem and initial point $w_1$ satisfying Assumption I in \pref{eq:ass1} and Assumption II in \pref{eq:ass2}. Suppose starting from the point $w_1$, we run SGD algorithm with step size $\eta$ for $n$ steps. Then the average iterate $\widehat{\w}^{\mathrm{SGD}}_{n} \ldef{} \frac{1}{n} \sum_{i=1}^n w_i$ enjoys the bound
$$ 
\En_{S}[F(\widehat{\w}^{\mathrm{SGD}}_{n})] - \inf_{w} F(\w) \le \eta (\sigma^2 + L^2) + \frac{1}{2 \eta n} \nrm*{w_1 - w^*}^2,$$ 
for any point $w^* \in \argmin_{w} F(w)$. 
\end{lemma} 
\begin{proof} The proof follows exactly along the lines of the proof of \pref{thm:app_SGD_convergence} given on page~\pageref{thm:app_SGD_convergence}. 
\end{proof} 

\begin{theorem}\label{thm:sgd_works_multipass}
Consider any SCO problem satisfying Assumption I in \pref{eq:ass1} and Assumption II in \pref{eq:ass2}. Suppose we run the following variant of SGD: for some integer $k \geq 1$, we run \pref{alg:multi_pass_SGD} for $nk$ steps with the only change being that fresh samples $z_t$ are used in each update step (line 5) instead of reusing samples. Then the average iterate $\widehat{\w}^{\mathrm{SGD}}_{nk}$ enjoys the bound
$$ 
\En_{S}[F(\widehat{\w}^{\mathrm{SGD}}_{nk})] - \inf_{w:\ \|w - w_1\| \leq B} F(\w) \le 2(B^2 + L^2 + \sigma^2) \sqrt{nk}.  
$$ 
\end{theorem}
\begin{proof}
   The proof is completely standard, we just include it here for completeness. Let $w^* \in \bbR^d$ be any point such that $\|w_1 - w^*\| \leq B$. We use $\|w_t - w^*\|^2$ as a potential function. Within epoch $j$, we have
   \[\|w_{t+1} - w^*\|^2 = \|w_t - w^*\|^2 - 2\eta_j \langle \nabla f(w_t; z_t), (w_t - w^*)\rangle + \eta_j^2\|\nabla f(w_t; z_t)\|^2.\]
   Since $\|\nabla f(w_t; z_t)\|^2 \leq 2L^2 + 2\sigma^2$, by rearranging the above, taking expectation over $z_t$ conditioned on $w_t$, and using convexity of $F$ we have
   \[F(w_t) - F(w^*) \leq \frac{\|w_t - w^*\|^2 - \En[\|w_{t+1} - w^*\|^2 | w_t]}{2\eta_j} + \eta_j (L^2 + \sigma^2).\]
   Now taking expectation over the randomness in $w_t$, and summing up the inequality for all iterations in epoch $j$, we get 
   \[\sum_{t=n(j-1)+1}^{nj} F(w_t) - F(w^*) \leq \frac{\En[\|w_{n(j-1)+1} - w^*\|^2 - \En[\|w'_{nj+1} - w^*\|^2]}{2\eta_j} + \eta_j (L^2 + \sigma^2)n,\]
   where we use the notation $w'_{nj+1}$ to denote the iterate computed by SGD before the $\Pi_{w_1, B}$ projection to generate $w_{nj+1}$. Summing over all the $k$ epochs, we have
   \begin{align*}
   \sum_{t=1}^{nk} F(w_t) - F(w^*) & \leq \frac{\|w_1 - w^*\|^2}{2\eta_1} + \sum_{j=1}^{k-1}\prn[\bigg]{\frac{\En[\|w_{nj+1} - w^*\|^2]}{2\eta_{j+1}}- \frac{\En[\|w'_{nj+1} - w^*\|^2]}{2\eta_j}} \\ 
   & - \frac{\En[\|w_{nk+1} - w^*\|^2]}{2\eta_k} + \sum_{j=1}^k \eta_j (L^2 + \sigma^2)n\\
   & \overleq{\proman{1}} \frac{\|w_1 - w^*\|^2}{2\eta_1} + \sum_{j=1}^{k-1}\prn[\Big]{\frac{1}{2\eta_{j+1}} - \frac{1}{2\eta_j}}\En[\|w_{nj+1} - w^*\|^2]  + \sum_{j=1}^k \eta_j (L^2 + \sigma^2)n \\ 
   & \overleq{\proman{2}} \frac{2B^2}{\eta_k} + \sum_{j=1}^k \eta_j (L^2 + \sigma^2)n \\
   & \overleq{\proman{3}} 2(B^2 + L^2 + \sigma^2) \sqrt{nk}. 
   \end{align*}
   Here, $\proman{1}$ follows since $\|w_{nj+1} - w^*\|^2 \leq \|w'_{nj+1} - w^*\|^2$ since the $\Pi_{w_1, B}$ projection can only reduce the distance to $w^*$, and $\proman{2}$ follows since $\|w_{nj+1} - w^*\|^2 \leq 4B^2$ for all $j = 0, 1, \ldots, k-1$, and telescoping, and finally $\proman{3}$ follows by plugging in $\eta_j = 1 / \sqrt{nj}$. 

   Finally, the stated bound on $F(\widehat{\w}^{\mathrm{SGD}}_{nk})$ follows via an application of Jensen's inequality to the convex function $F$.  The dependence on problem specific constants ($\sigma, L$ and $B$) in the above bound can be improved further with a different choice of the step size $\eta$; getting the optimal dependence on these constants, however, is not the focus of this work. 
\end{proof}

We now prove \pref{thm:multipass}. To begin, we first define the instance space, the loss function, and the data distribution. 

\paragraph{Instance space and loss function.} We define the instance space 
$\cZ = \{0,1\}^{kd} \times \{\pm 1\}^k \times \{0,e_1,\ldots,e_d\}^k$. That is, each instance $z \in \cZ$ can be written as a 3-tuple $z = (x,y,\alpha)$ where $x \in \{0,1\}^{kd}$, $y \in \{\pm1\}^k$, and $\alpha \in \{0,e_1,\ldots,e_d\}^k$. For each $s \in [k]$, we define $x_s, y_s, \alpha_s$ to be the $s$-th parts of $x, y, \alpha$ respectively when these vectors are split into $k$ contiguous equal sized blocks of sizes $d, 1, 1$ respectively. Define the function $\fD: \bbR^{d+ n + 2} \times \cZ \to \bbR$ on the variables $u \in \bbR, v \in \bbR^{n+1}, \tau \in \bbR^d$ and instance $z = (x, y, \alpha) \in \cZ$ as follows. First, define the intervals $I_1 = (-\infty, \frac{1}{k}]$, $I_s = (\frac{(s-1)}{k}, \frac{s}{k}]$ for $s = 2, 3, \ldots, k-1$, and $I_k = (\frac{(k-1)}{k},\infty)$. Then define 
\vspace{-0.1em} 
\begin{align} 
  \fD((u,v,\tau); z) &:= f_N(v) + \sum_{s=1}^{k} \!\indic[u \in I_s] \fA(\tau; (x_s, y_s, \alpha_s)) - \tfrac{2}{\sqrt{kn}}\min\{u, 1\} + c_1,\tag{C} \label{eq:multipass-fn}
\end{align} 
where the function $f_N: \bbR^{n+1} \to \bbR$ is defined as 
\[f_N(v) := (\tfrac{\sqrt{n+1}}{1+\sqrt{n+1}})\max_{i \in [n+1]} v_i + (\tfrac{1}{2+2\sqrt{n+1}})\|v\|^2,\]
and the constant $c_1 = \frac{2}{\sqrt{kn}} + \frac{1}{2 + 2 \sqrt{n + 1}}$. Finally, we define the variable $w$ to denote the tuple $(u, v, \tau)$. We also assume that $n \geq 300$ and $d \geq \log(10)2^n + 1$. 

For the purpose of defining the SGD update on the function $\fD(\cdot;z)$ for any given $z \in \cZ$, we make the following convention for defining subgradients: when $u \in I_s$ for some $s \in [k]$, we use a subgradient of 
\[f_N(v) + \fA(\tau; (x_s, y_s, \alpha_s)) - \tfrac{2}{\sqrt{kn}}\min\{u, 1\}.\] It is easy to check that with the above convention, the norm of the subgradient is always $O(1)$.  

\paragraph{Input distribution.} The samples $z \in \cZ$ are drawn by sampling $\langle (x_s, y_s, \alpha_s) \rangle_{s=1}^k \sim \cD(\frac{1}{10}, \frac{1}{2}, e_{j^*})^{\otimes k}$, for some $j^* \in [d]$ that will be defined in the proof (see \pref{def:violating_distribution_defn} for the definition of $\cD$). 

\begin{proof}[{\bf Proof of \pref{thm:multipass}}]  
Let $\cD'$ denote the distribution specified over the instance space $\cZ$. The exact choice of $j^*$ will be specified later in this proof. Note that due to the indicators in \eqref{eq:multipass-fn}, the population loss under $\fD$ when the data are sampled from $\cD'$ can be written as 
\begin{equation} \label{eq:F_D-pop-loss}
F(w) = \En_{z \sim \cD'}[\fD((u, v, \tau); z')] =  f_N(v) + \En_{z\sim \cD'}[\fA(\tau; z)] - \tfrac{2}{\sqrt{kn}}\min\{u, 1\}. 
\end{equation} 
Note that by \pref{lem:pop-loss-convexity}, $\En_{z\sim \cD'}[\fA(\tau; z)]$ is a convex function of $\tau$. Hence, $F$ is convex, $1$-Lispchitz and $\tau = e_{j^*}$ denotes its unique minimizer. Also note that $F$ nicely separates out as convex functions of $u, v, \tau$ and hence it is minimized by minimizing the component functions separately. In particular, it is easy to check by computing the subgradient that the optima are
\[u = 1,\qquad v = -\tfrac{1}{\sqrt{n+1}}\mathbf{1}, \qquad w = e_{j^*},\]
where $\mathbf{1}$ is the all 1's vector. The corresponding optimal values are $-\frac{2}{\sqrt{kn}}$, $-\frac{1}{2+2\sqrt{n+1}}$, and $0$. This implies that $F^* = \min_{w} F(w) = 0$. More importantly, the suboptimality gap also decomposes as the sum of suboptimality gaps for the three functions, a fact we will use repeatedly in our analysis. 

Finally, we will assume that the initial value of the variables $u, v, \tau$ is $0$. 
\paragraph{Proof of part-(a).}  Since $\nrm{\grad f(u,v,\tau; z)} \leq O(1)$ for any $z \in \cZ$ and $u, v, \tau$, and the optimal values of $u, v, \tau$ are $O(1)$ in magnitude, and $F(\omega)$ is convex in $\omega = (u, v, \tau)$, \pref{thm:sgd_works_main} implies that single pass SGD run for $n$ steps with a step size of $\frac{1}{\sqrt{n}}$ will obtain $O(\frac{1}{\sqrt{n}})$ suboptimality gap in expectation. 

As for the lower bound, note that the function $f_N$ is exactly the same one constructed in \citep[Section 3.2.1]{nesterov}. There, it is shown (in Theorem 3.2.1) that a certain class of gradient based methods (including GD/SGD) have suboptimality at least $\Omega(1/\sqrt{n})$ after $n$ iterations, which completes the proof of part-(a) of \pref{thm:multipass}.

\paragraph{Proof of part-(b).} In the following, we give convergence guarantee for multi-pass SGD algorithm that does $k$ passes over the dataset $S_1$ of $n/2$ samples (see \pref{alg:multi_pass_SGD}).  First, note that the deterministic components of the function $\fD$, viz. $f_N$ and $- \tfrac{2}{\sqrt{kn}}\min\{u, 1\}$, are unaffected by the randomized component in the iterates produced by SGD update. Since these deterministic components are $1$-Lipschitz, their optimal values are $O(1)$ in magnitude and the corresponding population losses are convex, we conclude via \pref{thm:sgd_works_multipass} that the multi-pass SGD, which is equivalent to GD, attains a suboptimality gap of $O(\frac{1}{\sqrt{kn}})$ on these components. Specifically, the points $\wh u_k$ and $\wh v_k$ returned after the $k$-th pass satisfy 
\begin{align*}
f_N(\wh v_k) - \tfrac{2}{\sqrt{kn}}\min\{\wh u_k, 1\}  - \min_{v, u} \prn*{f_N(v) - \tfrac{2}{\sqrt{kn}}\min\{u, 1 \} } \leq O(\frac{1}{\sqrt{nk}}). 
\end{align*}

Coming to the randomized component of $\fD$, we note that as long as $u < 1$, the gradient of $\fD$ with respect to $u$ is always $- 2/ {\sqrt{kn}}$. Thus, $u$ keeps monotonically increasing at each step of SGD with an increment equal to the step size times $2 / {\sqrt{kn}}$. Suppose we run $k$ pass SGD with step size set to $1 / {\sqrt{kn}}$ and $u$ starting at $0$, where in each pass we take $n/2$ steps of SGD using the dataset $S_1$. It is easy to see that for all $s \in [k]$, the value of $u$ stays in $I_s$ in the $s$-th pass, and traverses to the next interval $I_{s+1}$ as soon as the $(s+1)$-th pass starts. Thus, within each pass $s \in [k]$ and for all $i \in [n]$, multi-pass SGD encounters a fresh i.i.d.\ sample $(x_{i,s}, y_{i,s}, \alpha_{i,s}) \sim \cD$ for every update. This is thus equivalent to running SGD with $kn/2$ such i.i.d.\ samples over the first $kn/2$ iterations. An application of \pref{thm:sgd_works_multipass} (where we set $n$ to be $n/2$) implies that the suboptimality gap of the iterate $\wh \tau_k$ generated after the $k$th-pass of SGD algorithm on the randomized component of $\fD$ is
\begin{align*}
\En_{z\sim \cD}[\fA(\wh \tau_k; z)] - \min_{w} \En_{z\sim \cD}[\fA(\w; z)] &\leq O\prn[\Big]{\frac{1}{\sqrt{nk}}}. 
\end{align*}
Taking the above two bounds together, we get that the point $\wh w_k = (\wh u_k, \wh v_k, \wh \tau_k)$ satisfies
\begin{align*}
F(\wh w_k) = F(\wh w_k) - \min_{w} F(w) \leq O\prn[\Big]{\frac{1}{\sqrt{nk}}},  \numberthis \label{eq:multipass_bound1} 
\end{align*} where the equality in the first line above follows from using the fact that $ \min_{w} F(w) = 0$ by construction. 

Finally, note that the returned point $\wmp \in \argmin_{w \in \wh \cW} F_{S_2}(W)$ where $\cW = \crl*{\wh w_1, \ldots, \wh w_k}$. \pref{lem:concentration_fast_rate} thus implies that the point $\wmp$ satisfies
\begin{align*} 
F(\wmp) &\leq \min_{j \in [k]} F(\wh w_j) + O\prn[\Big]{ \frac{L \log(k/\delta)}{n} + \sqrt{\frac{ \min_{j \in [k]} F(\wh w_j)\,  L \log(k/\delta)}{n}}} \\
&\leq F(\wh w_k) + O\prn[\Big]{ \frac{L \log(k/\delta)}{n} + \sqrt{\frac{ F(\wh w_k)\,  L \log(k/\delta)}{n}}} \\
&\leq \frac{1}{\sqrt{nk}} + O\prn[\Big]{ \frac{\log(k/\delta)}{n} + \frac{1}{n} \sqrt{\frac{\log(k/\delta)}{k}}}, 
\end{align*} where $L$-denotes the Lipschitz constant for the function $f$ and the inequality in the last line follows by plugging in the bound on $F(\wh w_k)$ from \pref{eq:multipass_bound1}  and using the fact that $L = O(1)$. Finally, observing that $\min_{w} F(w) = 0$, the above bound implies that 
\begin{align*}
F(\wmp)  - \min_{w} F(w) \leq O\prn[\Big]{\frac{1}{\sqrt{nk}}} 
\end{align*} for $k = o(n)$; proving the desired claim. 

\paragraph{Proof of part-(c).} The proof follows exactly along the lines of the proof of \pref{thm:lowersco} in \pref{app:thm_lowersco_proof}. Recall that $n \geq 300$ and $d \geq \log(10)2^n + 1$. 

Assume, for the sake of contradiction, that there exists a regularizer $R: \bbR \times \bbR^{n +1} \times \bbR^d \to \bbR$ such that for any distribution $\cD = \cD(\frac{1}{10}, \frac{1}{2}, e_{j^*})$ (see \pref{def:violating_distribution_defn}) for generating the components of $z$, the expected suboptimality gap for RERM is at most $\epsilon = 1 / {20000}$. Then, by Markov's inequality, with probability at least $0.9$ over the choice of sample set $S$, the suboptimality gap is at most $10\epsilon$.

First, since the population loss separates out nicely in terms of losses for $u, v, \tau$ in \eqref{eq:F_D-pop-loss}, we conclude that if $(u', v', \tau')$ is a $10\epsilon$-suboptimal minimizer for the population loss when components of $z$ are drawn from $\cD$, then $u'$, $v'$, and $\tau'$ must be individually $10\epsilon$-suboptimal minimizers for the functions $-\tfrac{1}{\sqrt{kn}}\min\{u, 1\}$, $f_N(v)$ and $\En_{z\sim \cD}[\fA(\tau; z)]$ respectively. Additionally, from \pref{lem:pop-loss-convexity}, we must have that $\|\tau' - e_{j^*}\| \leq 100\epsilon$. With this insight, for any $j \in [d]$, define the set 
\begin{align*} 
G_j &:= \left\{(u, v, \tau):\ -\tfrac{1}{\sqrt{kn}}\min\{u, 1\} \leq -\tfrac{1}{\sqrt{kn}} + 10\epsilon,\ f_N(v) \leq -\tfrac{1}{2+2\sqrt{n+1}} + 10\epsilon,\ \|\tau\!-\!e_j\| \leq 100\epsilon\right\}.
\end{align*}

This set covers all possible $10\epsilon$-suboptimal minimizers of the population loss. Also, for convenience, we use the notation \[ \fullvariable = (u, v, \tau).\]

As in the proof of \pref{thm:lowersco}, define the points $\fullvariable_j^*$ for $j \in \{0, 1, \ldots, d\}$ to be
\begin{align*} 
\fullvariable^*_{j}  \in \argmin_{\fullvariable \in G_j} R(\fullvariable)
\end{align*} 
Now we are ready to define $j^*$: let $j^* \in [d]$ be any element of $\argmax_{j \in [d]} ~ R(\fullvariable^*_j)$. The proof now follows similarly to the proof of \pref{thm:lowersco}. In particular, in the following, we condition on the occurrence of the event $E$ defined in the proof of \pref{thm:lowersco}. 

Now define $c := -\frac{1}{\sqrt{kn}} - \frac{1}{2+2\sqrt{n+1}}$. It is easy to see that for any $\fullvariable = (u, v, \tau) \in G_j$ for any $j \in [d]$, we have
\begin{equation} \label{eq:u-v-suboptimality}
  c \leq -\tfrac{1}{\sqrt{kn}}\min\{u, 1\} + f_N(v) \leq c + 20\epsilon.
\end{equation}

Next, consider the point $\fullvariable^*_{\wh j}$ defined in \pref{eq:w_jstar_defn} for the special coordinate $\wh j$. Reasoning similarly to the proof of \pref{thm:lowersco}, we conclude that that $\fullvariable^*_{\wh j}$ cannot be an $\epsilon$-suboptimal minimizer of $F$, and thus $\fullvariable_{\wh j}^*$ can not be a minimizer of the regularized empirical risk (as all RERM solutions are $10\epsilon$-suboptimal w.r.t. the population loss $F\,$). Thus, we must have   
\begin{align*} 
\wh F(\fullvariable^*_{\wh j}) + R(\fullvariable^*_{\wh j}) &> \min_{\fullvariable \in G_{j^*}} \prn*{\wh F(\fullvariable) + R(\fullvariable)} \\ &\geq \min_{\fullvariable \in G_{j^*}} \wh F(\fullvariable) + \min_{\fullvariable \in G_{j^*}} R(\fullvariable) \\ &\geq  R(\fullvariable^*_{j^*}) + c - 100\epsilon, \numberthis \label{eq:multipass_rerm_fails_5}
\end{align*} 
where $\wh F$ denotes the empirical loss on $S$,, and the last inequality above follows from \eqref{eq:u-v-suboptimality} and due to the fact that: if $\fullvariable^* = (u^*, v^*, \tau^*)$ is the minimizer of $\wh F(\fullvariable)$ over $G_{j^*}$, and $s \in [k]$ is the index such that $u^* \in I_s$, then the function $\tau \mapsto \frac{1}{n}\sum_{i=1}^n y_{i,s}\|\tau \odot x_{i,s}\|$ is 1-Lipschitz and takes the value $0$ at $\tau = e_{j^*}$. 

On the other hand, if $\fullvariable^*_{\wh j} = (\wh u, \wh v, \wh \tau)$ and $s \in [k]$ is the index such that $\wh u \in I_s$, then using \eqref{eq:u-v-suboptimality}, we have 
\begin{align*}
\wh F(\fullvariable^*_{\wh j}) &\leq \tfrac{1}{n}\textstyle\sum_{i=1}^n y_{i,s}\|(\wh \tau - e_{j^*}) \odot x_{i,s}\| -\frac{1}{\sqrt{kn}}\min\{\wh u, 1\} + f_N(\wh v) \\ 
&\leq \tfrac{1}{n}\textstyle\sum_{i=1}^n y_{i,s}\|(\wh \tau - e_{j^*}) \odot x_{i,s}\| + c + 20\epsilon. \numberthis \label{eq:multipass_rerm_fails_6} 
\end{align*}
Now, since we are conditioning on the occurrence of the event $E$, using the same chain of inequalities leading to \pref{eq:RERM_proof_gen6_1}, we conclude that 
\begin{equation} \label{eq:multipass_rerm_fails_7}
   \tfrac{1}{n}\textstyle\sum_{i=1}^n y_{i,s}\|(\wh \tau - e_{j^*}) \odot x_{i,s}\| \leq 100\epsilon - \frac{3}{200}. 
\end{equation}
Combining \pref{eq:multipass_rerm_fails_5}, \pref{eq:multipass_rerm_fails_6}, and \pref{eq:multipass_rerm_fails_7} and rearranging the terms, we get  
\begin{align*}
220\epsilon \geq R(\fullvariable^*_{j^*}) -  R(\fullvariable^*_{\wh j}) + \frac{3}{200} \geq \frac{3}{200}. \numberthis \label{eq:multipass_rerm_doesnt_work3} 
\end{align*} 
where the second inequality above holds because $j^* \in \argmax_{j \in [d]} R(\fullvariable^*_j)$ (by definition). Thus, $\epsilon \geq 3 / 44000 > 1 / 20000$, a contradiction, as desired.
\end{proof}

\section{Missing proofs from \pref{sec:dist_free_rerm}} 
%!TEX root=../paper.tex
Throughout this section, we assume that a data sample $z$ consists of $(x, \alpha)$, where $x \in \{0,1\}^d$ and $\alpha \in \{0,e_1,\ldots,e_d\}$. The loss function $f_{(\ref{eq:empfn_basic_cons})}:\reals^d \times \cZ$ is given by: 
\begin{align} 
	f_{(\ref{eq:empfn_basic_cons})}(\w; z) = & \frac{1}{2} \nrm*{\prn*{\w - \alpha} \odot x}^2 - \frac{c_n}{2} \nrm*{\w - \alpha}^2  +  \max\{1, \nrm*{\w}^4\},   \label{eq:second_function} 
\end{align} 
where $c_n \ldef{}  n^{-(\frac{1}{4} - \gamma)}$ for some $\gamma > 0$. We will also assume that $c_n \leq \frac{1}{4}$. Furthermore, since $f_{(\ref{eq:empfn_basic_cons})}(\w; z)$ is not differentiable when $\|\w\| = 1$, we make the following convention to define the sub-gradient:
\begin{align*}
   \partial f(\w; z) &=  {(\w - \alpha) \odot x - c_n(\w - \alpha)} + 4 \indicator{\|\w\| > 1} \nrm*{\w}^2\w. \numberthis \label{eq:subgradient-definition}
\end{align*}
 
 Whenever clear from the context in the rest of the section, we will ignore the subscript $(\ref{eq:empfn_basic_cons})$, and denote the loss function by $f(w; z)$.  Additionally we define the following distribution over the samples $z = (x, \alpha)$. 

\begin{definition} \label{def:violating_distribution_defn_a2}
For parameters $\delta \in [0, \frac{1}{2}]$, $c_n \in [0, 1]$, and $v \in \crl{0, e_1, \ldots, e_d}$, define the distribution $\bar{\cD}(\delta, c_n, v)$ over $z = (x, \alpha)$ as follows: 
\begin{align*} x \sim \cB\prn*{c_n + \delta}^{\otimes d} \qquad \text{and} \qquad \alpha = v. 
\end{align*} 
The components $x$ and $\alpha$ are sampled independently. 
\end{definition}

\subsection{Supporting technical results} 
Before delving into the proofs, we first establish some technical properties of the empirical loss $f(\w; z)$ and the population loss $F(\w)$. The following lemma states that for any distribution $\cD$, the minimizers of the population loss $F(\w) \ldef{} \En_{z \sim \cD} \brk*{f(\w; z)}$ are bounded in norm. 

\begin{lemma} \label{lem:minimizer-bound} 
Suppose $c_n \leq \frac{1}{4}$. Then for any data distribution $\cD$, the minimizer of $F(\w) \ldef{} \En_{z \sim \cD} \brk*{f(\w; z)}$ has norm at most $1$. 
\end{lemma} 	
\begin{proof} 
Note that for any $\w$ such that $\|\w\| > 1$, we have 
\[f(\w; z) \geq - c_n(\|\w\|^2 + \|\alpha\|^2) + \|\w\|^4 \geq -c_n\|\w\|^2 - c_n + \|\w\|^4 > 1 - 2c_n.\]
Thus, $F(\w) > 1 - 2c_n$.
On the other hand, $f(0; z) = \frac{1}{2}\|\alpha \odot x\|^2 - \frac{c_n}{2}\|\alpha\|^2 \leq \frac{1-c_n}{2}$. Hence, $F(0) \leq \frac{1-c_n}{2} \leq 1 - 2c_n$ since $c_n \leq \frac{1}{4}$. This implies that such a $\w$ is not a minimizer of $F(\w)$. 
\end{proof}

We show now that the single stochastic gradient descent update keeps the iterates bounded as long as the learning rate $\eta$ is chosen small enough. 

\begin{lemma} \label{lem:bounded-iterates}
Suppose that the learning rate $\eta \leq \frac{1}{100}$ and the point $\w$ satisfies $\nrm*{\w} \leq 2.5$. Let $\w^+ = \w - \eta \cdot \partial f(\w; z)$ for an arbitrary data point $z$ in the support of the data distribution. Then $\nrm*{\w^+} \leq 2.5$.
\end{lemma} 
\begin{proof} We prove the lemma via a case analysis:
\begin{enumerate}
\item \textbf{Case 1: $\mb{\nrm*{\w} < 2}$}. Using \pref{eq:subgradient-definition} and the Triangle inequality, we get that 
\begin{align*}
\nrm{\partial f(\w; z)} &\leq \nrm*{(\w - \alpha) \odot x} +  \nrm*{c_n(\w - \alpha)} + 4 \nrm{\w}^3  \\ 
									 &\leq (1 + c_n) \prn*{\nrm{\w} + \nrm{ \alpha}}  + 4 \nrm{\w}^3 \leq  36, \numberthis \label{eq:sgd_success_gl3} 
\end{align*} where the last inequality follows from the fact that the iterate $\w$ satisfies $\nrm*{\w} \leq 2$, and that the parameter $c_n \leq \frac{1}{4}$ and $\nrm{\alpha} \leq 1$. 

Now, for the gradient descent update rule $w_{t+1} = \w_{t} - \eta \partial f\prn*{\w; z}$, an application of the triangle inequality implies that  
\begin{align*}
\nrm{\w_{t+1}} &= \nrm*{\w - \eta \partial f(\w; z)}  \leq \nrm*{ \w}  + \eta \nrm*{\partial f(\w; z)}. \numberthis \label{eq:sgd_success_gl1} 
\end{align*}

Plugging in the bounds on $\nrm*{\partial f(\w; z)}$ derived above, we get that 
\begin{align*}
\nrm{\w^{+}} \leq  2 + 36 \eta \leq 2.5   \numberthis \label{eq:sgd_success_gl6} 
\end{align*}
since $\eta \leq \frac{1}{100}$. 

\item \textbf{Case 2: $\mb{2 \leq \nrm*{\w} \leq 2.5}$.} We start by observing that the new iterate $\w^{+}$ satisfies 
\begin{align*}
 \nrm{\w^{+}}^2 &= \nrm*{\w_{t} - \eta \partial f(\w; z)}^2  \\ 
 							   &= \nrm*{\w}^2 + \eta^2 \nrm*{\partial f(\w; z)}^2 - 2 \eta \tri*{\w, \partial f(\w; z)}. \numberthis \label{eq:sgd_success_gl2} 
\end{align*} 
Reasoning as in case 1 above, using the fact that $\nrm*{\w} \leq 2.5$, we have $\|\partial f(\w; z)\| \leq 70$. Furthermore, 
\begin{align*}
\tri*{\w, \partial f(\w; z)}  &= \tri{\w, {(\w - \alpha) \odot x - c_n(\w - \alpha)} + 4\nrm*{\w}^2\w } \\ 
&= 4 \nrm*{\w}^4 + \tri*{\w, \w \odot x} - \tri*{\w, \alpha \odot x + c_n \prn*{\w - \alpha}} \\  
&\overgeq{\proman{1}} 4 \nrm*{\w}^4  - \tri*{\w, \alpha \odot x + c_n \prn*{\w - \alpha}} \\ 
&\overgeq{\proman{2}} 4 \nrm*{\w}^4  - \nrm*{\w} \prn*{\prn*{1 + c_n}\nrm*{\alpha} + c_n \nrm*{\w} }  
\end{align*} where the inequality in $\proman{1}$ follows from the fact that $\tri*{\w, w \odot x} \geq 0$, the inequality in $\proman{2}$ is given by an application of Cauchy-Schwarz inequality followed by Triangle inequality. Next, using the fact that $c_n \leq \frac{1}{4}$ and $\nrm*{\alpha} \leq 1$, we get 
\begin{align*} 
\tri*{\w, \partial f(\w; z)} &\geq 4 \nrm{\w}^4  - \nrm*{\w} \prn[\Big]{\frac{5}{4} + \frac{1}{4} \nrm*{\w} }  \geq 60, \numberthis \label{eq:sgd_success_gl4}  
\end{align*}
 where the last inequality holds as the polynomial $f(a) \ldef{} 4a^4 - a \prn*{ 1.25 + 0.25 a}$ is an increasing function of $a$ over the domain $a \in [2, \infty)$. 

Plugging the bound $\|\partial f(\w; z)\| \leq 70$ and and \pref{eq:sgd_success_gl4} in \pref{eq:sgd_success_gl2}, we have
\begin{align*}
 \nrm{\w^{+}}^2 &\leq \nrm*{\w}^2 + 4900 \eta^2 - 60 \eta \leq \nrm*{\w}^2,
\end{align*} 
since $\eta \leq \frac{1}{100}$.
Thus, $\nrm{\w^{+}} \leq \nrm*{\w} \leq 2.5$.
\end{enumerate}

Thus, in either case, $\nrm{\w^{+}} \leq 2.5$, completing the proof.
\end{proof} 

\begin{corollary}
Suppose that the learning rate $\eta \leq 1/100$ and the initial point $\w_1$ be such that $\nrm*{\w_1} \leq 2.5$. Then the iterates obtained when running either full gradient descent on the empirical risk, or by running SGD, have norm bounded by $2.5$ at all times $t \geq 0$.
\end{corollary} 
\begin{proof}  
The iterates obtained during the running of any of the gradient descent variants described in the statement of the lemma can be seen as convex combinations of single sample gradient descent updates. Thus, the bound on the norm follows immediately from \pref{lem:bounded-iterates} via the convexity of the norm.
\end{proof} 

The following lemma shows that the loss function $f(\w; z)$ is Lipschitz in the domain of interest, i.e. the ball of radius 2.5. This is the region where all iterates produced by gradient based algorithms and the global minimizers of the population loss (for any distribution $\cD$) are located in. 
\begin{lemma} \label{lem:lipschitzness}
In the ball of radius $2.5$ around $0$, and for any data point $z$, the function $\w \mapsto f(\w; z)$ is $70$-Lipschitz.
\end{lemma} 
\begin{proof}
For $\w$ such that $\nrm*{\w} \leq 2.5$, using \pref{eq:subgradient-definition} and the Triangle inequality, we get that 
\begin{align*} 
\nrm*{\partial f(\w; z)} &\leq \nrm*{(\w - \alpha) \odot x} +  \nrm*{c_n(\w - \alpha)} + 4 \nrm*{\w}^3  \\ 
                            &\leq (1 + c_n) \prn*{\nrm*{\w} + \nrm*{ \alpha}}  + 4 \nrm*{\w}^3 \leq  70. 
\end{align*} \end{proof} 

\subsection{Proof of  \pref{prop:general_learning_lb}} 

\begin{proof}
The proof is rather simple. To show this statement, we consider two distributions $\cD_1$ and $\cD_2$ on the instance space. What we will show is that if a learning algorithm succeeds with a rate any better than $c_n$ on one distribution, then it has to have a worse rate on the other distribution. Thus, we will conclude that any learning algorithm cannot obtain a rate better than $c_n$. 

Without further delay, let us define the two distributions we will use in this proof. The first distribution $\cD_1$ is given by: 
\begin{align*} 
x \sim \cB\prn*{\frac{1}{2}}^{\otimes d} ~~\textrm{and}~~~ \alpha = 0 
\end{align*} 
and the second distribution we consider is $\cD_2$ given as follows, first, we draw $\tilde{j} \sim \mathrm{Unif}[d]$, next we set $x[\tilde{j}] = 0$ deterministically, finally, on all other coordinates $i \ne \tilde{j}$, $x[\tilde{j}] \sim \mathrm{unif}\{0,1\}$. We also set $\alpha = 0$ deterministically.

Now the key observation is the following. Since $d > 2^n$, when we draw $n$ samples from distribution $\cD_1$ with constant probability there is a coordinate $\hat{j}$ such that $x_t[\hat{j}] = 0$ for all $t \in [n]$. Further, $\hat{j}$ can be any one of the $d$ coordinates with equal probability. However, on the other hand, if we draw $n$ samples from distribution $\cD_2$, then by definition of $\tilde{j}$, we have that $x_t[\tilde{j}] = 0$ for any $t \in [n]$. The main observation is that the learning algorithm is agnostic to the distribution on instances. Now since a draw of $n$ samples from $\cD_1$ has a coordinate $\hat{j}$ that to the algorithm is indistinguishable from coordinate $\tilde{j}$ when $n$ samples are drawn from $\cD_2$, the algorithm cannot identify which of the two distributions the sample if from. 

Hence, with constant probability both samples from $\cD_1$ and $\cD_2$ will be indistinguishable. However note that for distribution $\cD_1$ we have 
$$ 
F_1(\w) = \En_{z \sim \cD_1} f(\w,z) = \frac{1}{2}\prn[\Big]{\frac{1}{2} - c_n} \|\w\|^2 +   \max\{1, \nrm*{\w}^4\},  
$$
and for distribution $\cD_2$,
$$
F_2(\w) = \En_{z \sim \cD_2} f(\w,z) = \frac{1}{2}\prn[\Big]{\frac{1}{2} - c_n} \|\w_{[d] \setminus \{\tilde{j}\}}\|^2  - \frac{c_n}{2} |\w[\tilde{j}]|^2 +   \max\{1, \nrm*{\w}^4\}.
$$
Hence notice that for any $\w$,
\begin{align*}
F_1(\w) - \inf_{\w}F_1(\w) &= \frac{1}{2}\prn[\Big]{\frac{1}{2} - c_n} \|\w\|^2 +   \max\{1, \nrm*{\w}^4\} - 1\\
&\ge \frac{1}{2}\prn[\Big]{\frac{1}{2} - c_n} \|\w\|^2 \ge \frac{1}{2}\prn[\Big]{\frac{1}{2} - c_n} |\w[\tilde{j}]|^2
\end{align*}
and 
\begin{align*}
F_2(\w) - \inf_{\w}F_2(\w) &= \frac{1}{2}\prn[\Big]{\frac{1}{2} - c_n} \|\w_{[d] \setminus \{\tilde{j}\}}\|^2  - \frac{c_n}{2} |\w[\tilde{j}]|^2 +   \max\{1, \nrm*{\w}^4\} +  \frac{c_n}{2}  - 1\\
&\ge \frac{1}{2}\prn[\Big]{\frac{1}{2} - c_n} \|\w_{[d] \setminus \{\tilde{j}\}}\|^2  - \frac{c_n}{2} |\w[\tilde{j}]|^2  +  \frac{c_n}{2}  \\
&\ge   - \frac{c_n}{2} |\w[\tilde{j}]|^2  + \frac{c_n}{2} 
\end{align*} 

Now as mentioned before, they key observation is that with constant probability we get a sample using which we cant distinguish between whether we got samples from $\cD_1$ or from $\cD_2$. Hence in this case, if we want to obtain a good suboptimality, we need to pick a common solution $\w$ for which both $F_1(\w) - \inf_{\w} F_1(\w)$, and $F_2(\w) - \inf_{\w} F_2(\w)$ are small. However, note that if we want a $\w$ for which $F_2(\w) - \inf_{\w} F_2(\w) \le c_n/4$, then it must be the case that 
$$
\frac{c_n}{4} \ge - \frac{c_n}{2} |\w[\tilde{j}]|^2  + \frac{c_n}{2}  
$$
and hence, it must be the case that, $|\w[\tilde{j}]| \ge \frac{1}{\sqrt{2}}$. However, for such a $\w$ from the above we clearly have that
$$
F_1(\w) - \inf_{\w}F_1(\w) \ge \frac{1}{2}\prn[\Big]{\frac{1}{2} - c_n} |\w[\tilde{j}]|^2 \ge \frac{1}{4}\prn[\Big]{\frac{1}{2} - c_n} \ge \frac{1}{9}
$$
(as long as $c_n = o(1)$). Thus we can conclude that no learning algorithm can attain a rate better than $c_n/4$. 
\end{proof}

\subsection{Proof of \pref{thm:sgd_learns}}

We prove part-(a) and part-(b) separately below. The following proof of \pref{thm:sgd_learns}-(a) is similar to the proof of  \pref{thm:lowersco} given in \pref{app:thm_lowersco_proof}. 

\begin{proof}[{\bf Proof of \pref{thm:sgd_learns}-(a)}] 
In the following proof, we will assume that $n$ is large so that $n c_n^2 \geq 200$ and that $d \geq \ln(10) (1 - c_n)^{-n} + 1$.  
 
Assume, for the sake of contradiction, that there exists a regularizer $R: \bbR^d \to \bbR$ such that for any distribution $\cD \in \convexclass$ the expected suboptimality gap for the RERM solution is at most $\epsilon / 10$. Then, by Markov's inequality, with probability at least $0.9$ over the choice of sample set $S$, the suboptimality gap is at most $\epsilon$. We will show that $\epsilon$ must be greater than $c_n^2/{3200}$ for some distribution in the class $\convexclass$, hence proving the desired claim. 

We first define additional notation. Set $\delta = c_n / 10$ and define $\wt \epsilon \ldef{}  \epsilon / \delta$. For the regularization function $R(\cdot)$, define the points $\w^*_j$ for $j \in [d]$ such that 
\begin{align*} 
\w^*_{j}  \in \argmin_{\w \text{~s.t.~} {\nrm*{\w - e_j}^2 \leq \tepsilon}} R(\w)  \numberthis \label{eq:w_jstar_defn2}. 
\end{align*} and define the index $j^* \in [d]$ such that  $j^* \in \argmax_{j \in [d]} ~ R(\w^*_j).$ We are now ready to prove the desired claim. 

Consider the data distribution $\cD_1\ldef{} \bar{\cD}(\delta, c_n, e_{j^*})$ (see \pref{def:violating_distribution_defn_a2}) and suppose that the dataset $S = \crl*{z_i}_{i=1}^n$ is sampled i.i.d.\ from $\cD_1$. The population loss $F(w)$ corresponding to $\cD_1$ is given by 
\begin{align*} 
   F(\w) &= \En_{z \sim \cD_1} \brk*{f(\w; z)} = \frac{\delta}{2} \nrm*{\w-e_{j^*}}^2 + \max \crl{1, \nrm*{\w}^4}. 
\end{align*}
Clearly, $F(w)$ is convex in $w$ and thus the distribution $\cD_1 \in \convexclass$. Furthermore, $e_{j^*}$ is the unique minimizer of $F(\cdot)$, and any $\epsilon$-suboptimal minimizer $w'$ for $F(\cdot)$ must satisfy 
\begin{align*}
\nrm{w' - e_{j^*}}^2 \leq \frac{2\epsilon}{\delta} = 2\wt \epsilon \leq \frac{1}{5}.   \numberthis \label{eq:optimality_condition} 
\end{align*}
To see the above, note that if $w'$ is an $\epsilon$-suboptimal minimizer of $F(\w)$, then $F(\w) \leq F(e_{j^*}) + \epsilon = 1+\epsilon$. However, we also have that $F(\w') \geq \delta \nrm*{\w- e_{j^*}}^2 + 1$. Taking the two inequalities together, and using the fact that $\tepsilon \leq 1/10$, we get the desired bound on $\nrm*{\w - e_{j^*}}$. 

\par Next, define the event $E$ such that the following hold: 
\begin{enumerate}[label=$(\alph*)$] 
\item For the coordinate $j^*$, we have $\sum_{z \in S} x[j^*] \leq  n(c_n + 2 \delta)$. 
\item There exists $\wh j$ such that $\wh j \neq j^*$ and $x[\wh j] = 0$ for all $z \in S$.  
\item RERM (with regularization $R(\cdot)$) using the dataset $S$ returns an $\epsilon$-suboptimal solution for the test loss $F(w)$. 
\end{enumerate} 
Since $x[j^*] \sim \cB(c_n + \delta)$, Hoeffding's inequality (\pref{lem:hoeffding}) implies that the event (a) above occurs with probability at least $0.8$ for $n \geq 2 / \delta^2 = 200 / c_n^2$. Furthermore, \pref{lem:special_coordinate_exists1} gives us that the event (b) above occurs with probability at least $0.9$ for $d \geq \ln(10) (1 - c_n)^{-n}$. Finally, the assumed performance guarantee for RERM with regularization $R(\cdot)$ implies that the event (c) above occurs with probability at least $0.9$. Thus, the event $E$ occurs with probability at least $0.6$. In the following, we condition on the occurrence of the event $E$. 

\par Consider the point $w^*_{\wh j}$, defined in \pref{eq:w_jstar_defn2} corresponding to the coordinate $\wh j$ (that occurs in event $E$). By definition, we have that $\nrm{w^*_{\wh j} -e_{\wh j}}^2 \leq \tepsilon$, and thus 
\begin{align*}
\nrm{w^*_{\wh j} - e_{j^*}}^2 &\geq \frac{1}{2}\nrm{e_{\wh j} - e_{j^*}}^2 - \nrm{w^*_{\wh j} - e_{\wh j}}^2 \\ 
&\geq 1 - \tepsilon \geq \frac{1}{4}, 
\end{align*} where the first line above follows from the identity that $(a + b)^2 \leq 2a^2 + 2b^2$ for any $a, b > 0$, and the second line holds because $\wh j \neq j^*$ and because $\wt \epsilon \leq 1/10$. As a consequence of the above bound and the condition in \pref{eq:optimality_condition}, we get that the point $w^*_{\wh j}$ is not an $\epsilon$-suboptimal point for the population loss $F(\cdot)$, and thus would not be the solution of the RERM  algorithm (as the RERM solution is $\epsilon$-suboptimal w.r.t $F(w)$). Since, any RERM must satisfy condition \pref{eq:optimality_condition}, we have that

\begin{align*} 
\wh F(\w^*_{\wh j}) + R(\w^*_{\wh j}) &> \min_{\w:\ \nrm*{\w - e_{j^*}}^2 \leq \tepsilon} \prn[\big]{\wh F(\w) + R(\w)} \\
&\geq \min_{\w:\ \nrm*{\w - e_{j^*}}^2 \leq \tepsilon} \wh F(\w) + \min_{\w:\ \nrm*{\w - e_{j^*}}^2 \leq \tepsilon} R(\w) \\   
&\overgeq{\proman{1}} \min_{\w:\ \nrm*{\w - e_{j^*}}^2 \leq \tepsilon} \prn[\Big]{-\frac{c_n}{2}\nrm*{\w - e_{j^*}}^2 + \max\crl{1, \nrm*{w}^4}} + R(\w^*_{j^*}) \\ 
&\geq \min_{\w:\ \nrm*{\w - e_{j^*}}^2 \leq \tepsilon} \prn[\Big]{-\frac{c_n}{2}\nrm*{\w - e_{j^*}}^2 + 1} + R(\w^*_{j^*}) \\ 
&= -\frac{c_n}{2} \cdot \tepsilon + 1 + R(\w^*_{j^*}) \\ 
&\overgeq{\proman{2}} - \frac{c_n}{20} + 1 + R(\w^*_{j^*}) \numberthis \label{eq:RERM_proof_gen5} 
\end{align*} where $\wh F(w) \ldef{} \frac{1}{n} \sum_{i=1}^n f(w; z_i)$ denotes the empirical loss on the dataset $S$, the inequality $\proman{1}$ follows by ignoring non-negative terms in the empirical loss $\wh F(w)$, and the inequality $\proman{2}$ is due to the fact that $\tepsilon \leq 1/10$. For the left hand side, we note that 
\begin{align*} 
\wh F(\w^*_{\whj}) &= \frac{1}{2n}\sum_{(x, \alpha) \in S}\|(\w^*_{\wh j}-e_{j^*})\odot x\|^2 - \frac{c_n}{2}\|(\w^*_{\wh j}-e_{j^*})\|^2 + \max\crl{1, \nrm{w^*_{\wh j}}^4} \\ 
&\overleq{\proman{1}} \wt \epsilon +  \frac{1}{2n}\sum_{(x, \alpha) \in S}\|(e_{\wh j}-e_{j^*})\odot x\|^2 - \frac{c_n}{2}\|(e_{\wh j}-e_{j^*})\|^2 + \max\crl{1, (1 + \tepsilon)^4} \\
&\overleq{\proman{2}} 16 \wt \epsilon +  \frac{1}{2n}\sum_{(x, \alpha) \in S}\|(e_{\wh j}-e_{j^*})\odot x\|^2 - \frac{c_n}{2}\|(e_{\wh j}-e_{j^*})\|^2 + 1 \\
&\overleq{\proman{3}} 16 \wt \epsilon + \frac{1}{2} (c_n + 2 \delta) - \frac{\sqrt{2}c_n}{2} + 1,  \numberthis \label{eq:RERM_proof_gen6} 
\end{align*} where the inequality $\proman{1}$ above holds because $\nrm{w^*_{\wh j} - e_{\wh j}} \leq \tepsilon$ (by definition) and because $\nrm{e_{\wh j}} = 1$. The inequality in $\proman{2}$ follows from the fact that $(1 + a)^4 \leq 1 + 15 a$ for $a < 1$. Finally, the inequality $\proman{3}$ is due to the fact that $x[\wh j] = 0$ for all $(x, \alpha) \in S$ and because $\sum_{(x, \alpha) \in S} x[j^*] \leq n(c_n + 2\delta)$ due to the conditioning on the event $E$.  

Combining the bounds in \pref{eq:RERM_proof_gen5}  and \pref{eq:RERM_proof_gen6}, plugging in $\delta = c_n / 10$, and rearranging the terms, we get that 
\begin{align*}
16 \wt \epsilon &\geq \frac{c_n}{20} + R(\w^*_{j^*}) - R(\w^*_{\wh j}) \geq \frac{c_n}{20},  
\end{align*} where the second inequality above holds because $j^* \in \argmax_{j \in [d]} R(w^*_j)$ (by definition). Since, $\tepsilon = \epsilon / \delta$, we conclude that 
\begin{align*}
\epsilon &\geq \delta \cdot \frac{c_n}{320} \geq \frac{c_n^2}{3200}. 
\end{align*}

The above suggests that for data distribution $\cD_1$, RERM algorithm with the regularization function $R(w)$ will suffer an excess risk of at least $\Omega(c_n^2)$ (with probability at least $0.9$). Finally note that $R(w)$ could be any arbitrary function in the above proof, and thus we have the desired claim for any regularization function $R(w)$. 
\end{proof} 

We now prove \pref{thm:sgd_learns}-(b). In the following, we give convergence guarantee for SGD algorithm given in \pref{eq:SGD} when run with step size $\eta = 1/20 \sqrt{n} < 1/100$ using the dataset $S = \crl{z_i}_{i=1}^n$ drawn i.i.d.\ from a distribution $\cD$. We further assume that assume that the initial point $\w_1 = 0$, and thus $\nrm*{\w_1} \leq 2.5$. Note that having initial weights to be bounded is typical when learning with non-convex losses, for eg. in deep learning.

\begin{proof}[{\bf Proof of \pref{thm:sgd_learns}}-(b)]  We consider the two cases, when $D \in \convexclass$ and when $D \notin \convexclass$, separately below: 
\paragraph{Case 1: When $\cD \in \convexclass$ .} In this case, we note the population loss $F(\w) \ldef{} \En_{z \sim \cD} \brk*{f(\w; z)}$ is convex in $\w$ by the definition of the set $\convexclass$ in \pref{eq:convdis}.  Further, since $\eta = 1/20 \sqrt{n} < 1/100$,  \pref{lem:bounded-iterates} and \pref{lem:lipschitzness} imply that $f$ is $70$-Lipschitz. Finally, due to \pref{lem:minimizer-bound}, the initial point $w_1 = 0$ satisfies $\nrm*{w_1 - w^*} \leq 2.5$. Thus, we satisfy both  Assumption I  and Assumption II in \pref{eq:ass1} and \pref{eq:ass2} respectively, and an application of \pref{thm:sgd_works_main} implies that the point $\wh w^{\mathrm{SGD}}_n$ enjoys the performance guarantee 
\begin{align*}
\En \brk*{F(\wh w^{\mathrm{SGD}}_n) - \w^* \mid E} \leq O\prn[\Big]{\frac{1}{\sqrt{n}}}. 
\end{align*} 

\paragraph{Case 2: When $\cD \notin \convexclass$ .} 
In order to prove the performance guarantee of SGD algorithm in this case, we split up the loss function $f(\w; z)$ into convex and non-convex parts $g$ and $\wt g$, defined as: 
\[g(\w; z) := \frac{1}{2}\|(\w-\alpha)\odot x\|^2 + \max\{1, \|\w\|^4\} \qquad \text{ and } \qquad \wt g(\w; z) := -c_n\|w-\alpha\|^2.\]
Further, we define the functions $G(\w)$ and $\wt{G}(\w)$  to denote their respective population counterparts, i.e. $G(\w) = \En_{z \sim \cD}[g(\w; z)]$ and $\wt G(\w) = \En_{z \sim \cD}[\wt g(\w; z)]$. Clearly, 
\begin{align*}
F(\w) = \En_{z \sim \cD} \brk*{f(\w; z)} = G(\w) + \wt G(\w).  \numberthis \label{eq:poploss_decomposition} 
\end{align*} 

Let $\w^*$ be a minimizer of $F(\w; \cD)$. By \pref{lem:minimizer-bound}, we have $\|\w^*\| \leq 1$. The folltowing chain of arguments follows along the lines of proof of SGD in Case 1 above (see the proof of \pref{thm:app_SGD_convergence} on page~\pageref{thm:app_SGD_convergence}).  

Let the sequence of iterates generated by SGD algorithm be given by $\crl*{\w_t}_{t=1}^T$.  We start by observing that for any $t \geq 0$,  
\begin{align*}
\nrm*{\w_{t+1} - \w^*}^2_2 &=  \nrm*{\w_{t+1} - \w_{t} + \w_{t}  - \w^*}^2_2 \\
					&= \nrm*{\w_{t+1} - \w_{t} }_2^2 + \nrm*{\w_{t}  - \w^*}^2_2 + 2 \tri*{ \w_{t+1} - \w_{t},  \w_{t}  - \w^*} \\ 
					&= \nrm*{ - \eta \nabla f_w(\w_{t}; z_t)}_2^2 + \nrm*{\w_{t} - \w^*}_2^2 + 2\tri*{  - \eta \nabla f_w(\w_{t}; z_t), \w_{t} - \w^* },  
\intertext{where the last line follows from plugging in the update step $ w_{t+1} = \w_{t} - \eta \nabla f_w(\w_{t}; z_t)$. Rearranging the terms in the above, we get that }
\tri*{\nabla f(\w_{t}; z_t), \w_{t} - \w^* } &\leq \frac{\eta}{2} \nrm*{ \nabla f(\w_{t}; z_t)}_2^2  + \frac{1}{2\eta} \prn[\big]{ \nrm*{\w_{t} - \w^*}_2^2 - \nrm*{\w_{t+1}- \w^*}_2^2} \\
&\leq 2450 \eta   + \frac{1}{2\eta} \prn[\big] { \nrm*{\w_{t} - \w^*}_2^2 - \nrm*{\w_{t+1}- \w^*}_2^2},\numberthis \label{eq:sgd_upper_bound_1}
\end{align*}
where the second inequality in the above follows from the bound on the Lipschitz constant of the function $f$ when the iterates stay bounded in a ball of radius 2.5 around $0$ (see \pref{lem:lipschitzness}). We split the left hand side as: 
\begin{align*}
\tri*{\nabla f(\w_{t}; z_t), \w_{t} - \w^* }  &= \tri*{\nabla g(\w_{t}; z_t), \w_{t} - \w^* }  +  \tri*{\nabla \wt g(\w_{t}; z_t), \w_{t} - \w^* }. 
\end{align*}
This implies that 
\begin{align*}
\tri*{\nabla g(\w_{t}; z_t), \w_{t} - \w^* }  &\leq  2450 \eta   + \frac{1}{2\eta} \prn[\big]{ \nrm*{\w_{t} - \w^*}_2^2 - \nrm*{\w_{t+1}- \w^*}_2^2}  + \tri*{\nabla \wt g(\w_{t}; z_t), \w^* - \w_{t}}. 
\end{align*} 
Taking expectation on both the sides with respect to the data sample $z_t$, while conditioning on the point $\w_t$ and the occurrence of the event $E$, we get  
\begin{align*}
\tri*{\nabla G(\w_{t}), \w_{t} - \w^* }  &\leq  2450\eta   + \frac{1}{2\eta} \En \brk[\big]{ \nrm*{\w_{t} - \w^*}_2^2 - \nrm*{\w_{t+1}- \w^*}_2^2}  + \tri{\nabla \wt G(\w_{t}), \w^* - \w_{t}}. 
\end{align*} 

Next, note that the function $G(\w)$ is convex (by definition). This implies that $G(\w^*) \geq G(\w_{t}) - \tri*{\nabla G(\w_{t}), \w_{t} - \w^*}$, plugging which in the above relation gives us 
\begin{align*} 
G(\w_{t}) - G(\w^*) &\leq 2450\eta   + \frac{1}{2\eta} \En \brk[\big]{ \nrm*{\w_{t} - \w^*}_2^2 - \nrm*{\w_{t+1}- \w^*}_2^2}  + \tri{\nabla \wt G(\w_{t}), \w^* - \w_{t}}.
\intertext{Telescoping the above for $t$ from $1$ to $n$, we get that}   
\sum_{t=1}^{n}\prn*{ G(\w_{t}) - G(\w^*) } &\leq 2450\eta n + \frac{1}{2\eta} \prn[\big]{ \nrm*{\w_1 - \w^*}_2^2 - \nrm*{\w_{n + 1} - \w^*}_2^2} + \sum_{t=1}^n  \tri{\nabla \wt G(\w_{t}), \w^* - \w_{t}} \\ 
					 &\leq 2450\eta n + \frac{\nrm*{\w_1 - \w^*}_2^2}{2\eta} + \sum_{t=1}^n  \nrm{\nabla \wt G(\w_{t})} \nrm*{\w^* - \w_{t}}, 
\end{align*} where the inequality in the second line follows by ignoring negative terms, and through an application of Cauchy-Schwarz inequality. Since, $\|\w_1\| = 0$ and $\|\w^*\| \leq 1$ (due to \pref{lem:minimizer-bound}), we have that $\|\w_1 - \w^*\| \leq 1$. Setting $\eta = 1/20 \sqrt{n}$, using the bound  $\|\w_1 - \w^*\| \leq 1$ and by an application of Jensen's inequality in the left hand side, we get that the point $\wh w^{\mathrm{SGD}}_n \ldef{} \frac{1}{n}\sum_{t=1}^n \w_t$ satisfies 
\begin{align*} 
 G(\wh w^{\mathrm{SGD}}_n) - G(\w^*)  &\leq \frac{245}{\sqrt{n}} + \frac{1}{n} \sum_{t=1}^{n}  \nrm{\nabla \wt G(\w_{t})} \nrm*{\w^* - \w_{t}} \numberthis  \label{eq:projected_gradient_descent_guarantee}. 
\end{align*} 
Furthermore, by \pref{lem:bounded-iterates}, since $\eta = \frac{1}{20 \sqrt{n}} \leq \frac{1}{100}$ and $\nrm*{\w_1} \leq 2.5$ (by construction), we have that $\nrm*{w_t} \leq 2.5$ for all $t$. Thus, $\nrm*{\w^* - \w_{t}} \leq 3.5$ for all $t$. Plugging this bound in \pref{eq:projected_gradient_descent_guarantee}, we get that 
\begin{align*}
 G(\wh w^{\mathrm{SGD}}_n) - G(\w^*)  &\leq \frac{245}{\sqrt{n}} + 3.5 \max_{t \leq n} ~ \nrm{ \grad \wt G(\w_t)}. 
\end{align*} 
Next, note that from the definition of the function $\wt G(\w_t)$, we have that 
\[\nrm{ \grad \wt G(\w_t)} = c_n\nrm*{\w_t-\alpha} \leq 3.5 c_n,\] and thus
\begin{align*}
 G(\wh w^{\mathrm{SGD}}_n) - G(\w^*)  &\leq \frac{245}{\sqrt{n}} + 3.5 c_n.  \numberthis \label{eq:poploss_decomposition2} 
\end{align*} 
Finally, using the relation \pref{eq:poploss_decomposition} and taking expectations on both the sides, we have that 
\begin{align*} 
	\En \brk*{ F(\wh w^{\mathrm{SGD}}_n) - F(\w^*) \mid E} &= \En \brk*{G(\wh w^{\mathrm{SGD}}_n) - G(\w^*)} + \En \brk{\wt G(\wh w^{\mathrm{SGD}}_n) - \wt G(\w^*)} \\
	&\leq \frac{245}{\sqrt{n}} + 3.5 c_n + c_n\En_{z}[\w^* - \alpha\|^2] \\ 
   &\leq  \frac{245}{\sqrt{n}} + 6c_n, \numberthis \label{eq:desired_sgd_guarantee}  
\end{align*} 
where the inequality in the second line is due to \pref{eq:poploss_decomposition2} and by using the definition of the function $\wt G(\w)$ and by ignoring negative terms. The last line holds because $\|\w^* - \alpha\| \leq \nrm*{\w^*} + \nrm*{\alpha} \leq 2$. 
\end{proof}

%!TEX root=../paper.tex
\section{Missing proofs from \pref{sec:nn_connections}}

\subsection{$\alpha$-Linearizable functions} In the following, provide the proof of the performance guarantee for SGD for $\alpha$-Linearizable functions. 
\begin{proof}[Proof of \pref{thm:sgd_works_linearizable}] \label{proof:sgd_works_linearizable_proof} The proof follows along the lines of the proof of \pref{thm:app_SGD_convergence} on page~\pageref{app:basic_algorithms}. Let $\crl{w_t}_{t \geq 1}$ denote the sequence of iterates generated by the SGD algorithm. We note that for any $w^* \in \argmin_{w} F(w)$ and time $t \geq 1$, 
\begin{align*} 
\nrm*{\w_{t +1} - \w^*}^2_2 &=  \nrm*{\w_{t +1} - \w_{t} + \w_{t}  - \w^*}^2_2 \\
					&= \nrm*{\w_{t +1} - \w_{t} }_2^2 + \nrm*{\w_{t}  - \w^*}^2_2 + 2 \tri*{  \w_{t +1} - \w_{t},  \w_{t}  - \w^*} \\ 
					&= \nrm*{ - \eta \nabla f(\w_{t}; z_t)}_2^2 + \nrm*{\w_{t} - \w^*}_2^2 + 2\tri*{  - \eta \nabla f(\w_{t}; z_t), \w_{t} - \w^* }, 
\end{align*} where the last line follows from plugging in the SGD update rule that  $ \w_{t + 1} = \w_{t} - \eta \nabla f(\w_{t}; z_t)$. 

Rearranging the terms in the above, we get that 
\begin{align*} 
\tri*{\nabla f(\w_{t}; z_t), \w_{t} - \w^* } &\leq \frac{\eta}{2} \nrm*{ \nabla f(\w_{t}; z_t)}_2^2  + \frac{1}{2 \eta} \prn[\big]{ \nrm*{\w_{t} - \w^*}_2^2 - \nrm*{\w_{t+1}- \w^*}_2^2}. 
\intertext{Taking expectation on both the sides, while conditioning on the point $\w_t$, implies that} 
\tri*{\nabla F(\w_{t}), \w_{t} - \w^* } &\leq \frac{\eta}{2} \En \brk[\big]{\nrm*{ \nabla f(\w_{t}; z_t)}_2^2 } + \frac{1}{2\eta} \prn[\big]{ \nrm*{\w_{t} - \w^*}_2^2 - \nrm*{\w_{t+1}- \w^*}_2^2} \\
&\leq \eta \En \brk[\big]{\nrm*{ \nabla f(\w_{t}; z_t) - \grad F(\w_t)}_2^2} + \eta \nrm*{\grad F(\w_t)}_2^2  \\ 
& \qquad \qquad + \frac{1}{2\eta} \prn[\big]{ \nrm*{\w_{t} - \w^*}_2^2 - \nrm*{\w_{t+1}- \w^*}_2^2} \\
&\leq \eta(\sigma^2 + L^2) + \frac{1}{2\eta} \prn[\big]{ \nrm*{\w_{t} - \w^*}_2^2 - \nrm*{\w_{t+1}- \w^*}_2^2},  \numberthis \label{eq:linearization_condition1} 
\end{align*}
where the inequality in the second line is given by the fact that $(a - b)^2 \leq 2 a^2 + 2b^2$ and the last line follows from using Assumption II (see \pref{eq:ass2}) which implies that $F(\w)$ is $L$-Lipschitz in $\w$ and that $ \En \brk{\nrm{ \nabla f(w; z_t) - \grad F(w)}_2^2} \leq \sigma^2$ for any $w$. Next, using the fact that $F$ is $\alpha	$-Linearizable, we have that there exists a $w^* \in \argmin_w F(w)$ such that for any $w$,  
\begin{equation}
	F(\w) -  F(w^*) \leq  \alpha \tri*{\nabla F(\w), \w - \w^*}.  \numberthis \label{eq:linearization_condition2}
\end{equation}
Setting $w^*$ to the one that satisfies \pref{eq:linearization_condition2}, and using the above bound for $w = w_t$ with the bound in \pref{eq:linearization_condition1}, we get that for any $t \geq 1$,  
\begin{align*} 
\frac{1}{\alpha} \prn*{ F(\w_{t}) - F^* } &\leq  \eta(\sigma^2 + L^2) + \frac{1}{ 2\eta} \prn[\big]{ \nrm*{\w_{t} - \w^*}_2^2 - \nrm*{\w_{t + 1} - \w^*}_2^2}, 
\end{align*} 
where $F^* \ldef{} F(w^*)$. Telescoping the above for $t$ from $1$ to $n$, we get that 
\begin{align*} 
\sum_{t=1}^{n} \frac{1}{\alpha} \prn*{ F(\w_{t}) - F^* } &\leq \eta n (\sigma^2 + L^2)  + \frac{1}{2\eta} \prn[\big]{ \nrm*{\w_1 - \w^*}_2^2 - \nrm*{\w_{n + 1} - \w^*}_2^2} \\ 
					 &\leq \eta n (\sigma^2 + L^2) + \frac{1}{2 \eta} { \nrm*{\w_1 - \w^*}_2^2}. 
					 \end{align*} 
					Dividing both the sides by $n$, we get that 
					 \begin{align*} 
\frac{1}{\alpha n} \sum_{t=1}^{n}\prn*{ F(\w_{t}) - F^* }  &\leq  \eta (\sigma^2 + L^2) + \frac{1}{2 \eta n} { \nrm*{\w_1 - \w^*}_2^2}. 
\end{align*} 

An application of Jensen's inequality on the left hand side, implies that for the point $\wh \w^{\mathrm{SGD}}_n \ldef{} \frac{1}{n} \sum_{t=1}^n \w_t$, we have that 
 \begin{align*}
\frac{1}{\alpha} \En \brk*{F(\wh \w^{\mathrm{SGD}}_n) - F^* } &\leq \eta (\sigma^2 + L^2) + \frac{1}{2 \eta n} { \nrm*{\w_1 - \w^*}_2^2}. 
\end{align*}
Setting $\eta = \frac{1}{\sqrt{n}}$ and using the fact that $\nrm*{\w_1 - \w^*}_2 \leq B$ in the above, we get that  
 \begin{align*} 
\En \brk*{F(\wh \w^{\mathrm{SGD}}_n) - F^* } &\leq \frac{\alpha}{\sqrt{n}} \prn*{ \sigma^2 + L^2 + B^2}, 
\end{align*}
which is the desired claim. Dependence on the problem specific constants ($\sigma, L$ and $B$) in the above bound can be improved further with a different choice of the step size $\eta$; getting the optimal dependence on these constants, however, is not the focus of this work. 
\end{proof} 

\subsection{Proof of \pref{thm:two_layer_failure}} 
\begin{figure} 
\begin{center} 
\includegraphics[height=0.4\textwidth]{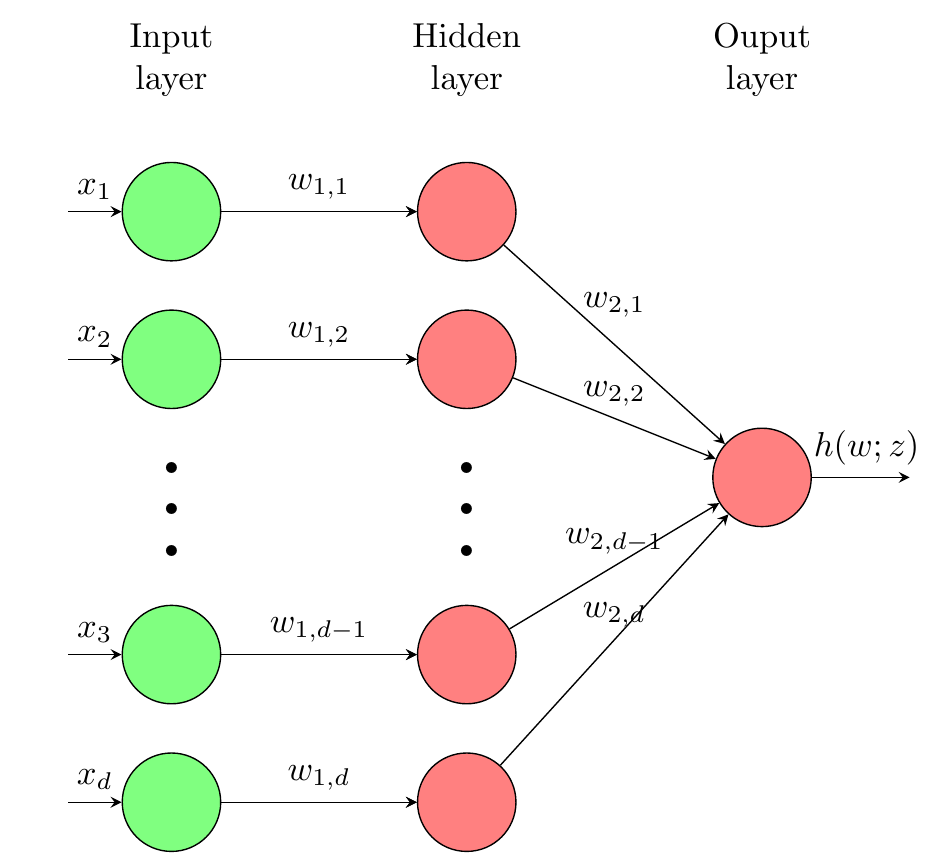} 	
\end{center} 
\caption{Two layer diagonal neural network. The weights are given by $w_1 \in \bbR^d$ and $w_2 \in \bbR^d$ for the first layer and the second layer respectively. The green nodes denote input nodes with linear activation function and red node denote hidden units with ReLU activation function.}  
\label{fig:nn_diagonal} 	
\end{figure} 

Let the input $\cX \in \crl*{0, 1}^d$ and the label $\cY \in \brk{-1, 1}$. Consider a two layer neural network with $\relu$ activation and weights given by $w = (w_1, w_2)$ where $w_1 \in \bbR^d$ denotes the weights of the first layer and $w_2 \in \bbR^d$ denotes the weights of the second layer, as shown in \pref{fig:nn_diagonal}. When given the input $x$, the output of the network with weights $(w_1, w_2)$ is computed as 
\begin{equation*} 
	h(w; x) = \relu\prn{w^\top_2 \relu(w_1 \odot x)}. 
\end{equation*}
Here, the first layer of the neural network has sparse connections, i.e. each input node connects to only one hidden node. Such networks are denoted as diagonal two layer neural networks. We assume that the neural network is trained using absolute loss function. Specifically, the instantaneous loss on a sample $z = (x, y)$ is given by 
\begin{align*}
f(w; z) = \abs{y - h(w; x)} = \abs{y - \relu\prn{w^\top_2 \relu(w_1 \odot x)}}. \numberthis  \label{eq:nn_diagonal_value}
\end{align*} 

Since $f(w;z)$ is not smooth, for any weights $w = (w_1,  w_2)$ and sample $z = (x, y)$, we define the gradient of $f(w; z)$ as  
\begin{align*} 
\nabla_{\w_1} \ls(\w; z)[i]  &= - \sign{y - h(w; x)} \cdot  \mathbf{1} \crl{\w_{2}^\top \relu \prn*{\w_{1} \odot x} > 0 } \cdot (\w_{2, i} \cdot \mathbf{1}\{\w_{1,i} \cdot x_i > 0\} \cdot x_i) 
\intertext{and}  
\nabla_{\w_2} \ls(\w; z)[i]   &= - \sign{y - h(w; x)} \cdot \mathbf{1} \crl{\w_{2}^\top \relu \prn*{\w_{1} \odot x} > 0 } \cdot (\relu\prn{\w_{1,i} \cdot x_i}), \numberthis  \label{eq:nn_diagonal_grad} 
\end{align*} for $i \in [d]$. We next show that the population loss $F(w) \ldef{} \En_{z \sim \cD} \brk*{f(w; z)}$ is $1/2$-Linearizable.    
\begin{lemma} 
\label{lem:f_linearizable} 
Let $\cD$ be defined such that $x$ and $y$ are independent random variables with distributions
\begin{align*}
x \sim \text{Uniform}(\crl{0, 1}^d) \qquad \text{and} \qquad y \sim \cB(1/4). 
\end{align*}
Then, the population loss $F(w) \ldef{} \En_{z \sim \cD}[f(w; z)]$ is $1/2$-Linearizable. 
\end{lemma} 
\begin{proof} The population loss $F(w)$ is given by:
\begin{align*}
	F(w) = \En_{x, y} \brk*{\abs{y - h(w; x)}},
\end{align*} where $h(w; x) = \relu\prn{w^\top_2 \relu(w_1 \odot x)}$. Using the fact that $y \in \crl{-1, 1}$, and $\Pr(y = 1) = 1/4$ and is independent of $x$, the above can be written as: 
\begin{align*}
F(w) = \En_{x} \brk*{ \frac{3}{4} \abs{h(w; x)} + \frac{1}{4} \abs{1 - h(w; x)}}. 
\end{align*}
It is easy to verify that $F(w)$ is minimized when $h(w; x) = 0$ for every $x \in \crl*{0, 1}^d$, which occurs at the point $w = 0$. Furthermore, $F(0) = 1/4$. Next, note that for any $w$, and sample $z = (x, y)$, 
\begin{align*} 
\tri*{w, \grad_{w} f(w; z)} &= \tri*{w_1, \grad_{w_1} f(w; z)} + \tri*{w_2, \grad_{w_2} f(w; z)} \\ 
 &\overeq{\proman{1}} - \sum_{i=1}^{d} w_{1, i} \cdot \sign{y - h(w; x)} \cdot  \mathbf{1} \crl{\w_{2}^\top \relu \prn*{\w_{1} \odot x} > 0 } \cdot (\w_{2, i} \cdot \mathbf{1}\{\w_{1,i} \cdot x_i > 0\} \cdot x_i)  \\
 & \qquad  \qquad - \sum_{i=1}^{d} w_{2, d} \cdot  \sign{y - h(w; x)} \cdot \mathbf{1} \crl{\w_{2}^\top \relu \prn*{\w_{1} \odot x} > 0 } \cdot (\relu\prn{\w_{1,i} \cdot x_i}) \\ 
 &= - \sign{y - h(w; x)} \cdot  \mathbf{1} \crl{\w_{2}^\top \relu \prn*{\w_{1} \odot x} > 0 } \cdot (w_{2}^\top \relu \prn{\w_{1} \odot x})  \\
 & \qquad  \qquad - \sign{y - h(w; x)} \cdot \mathbf{1} \crl{\w_{2}^\top \relu \prn*{\w_{1} \odot x} > 0 } \cdot (w_{2}^\top \relu\prn{\w_{1} \odot x}) \\
 &=  - 2 \sign{y - h(w; x)} \, \relu (w_{2}^\top \relu \prn{\w_{1} \odot x})  \\
 &=  - 2 \sign{y - h(w; x)} \, h(w; x)  \\ 
 &= 2 \abs{y - h(w; x)} - 2 \sign{y - h(w; x)} y 
\end{align*} where the equality $\proman{1}$ follows from using the definition of $\grad_{w_1}f(w; z)$ and  $\grad_{w_2}f(w; z)$  from \pref{eq:nn_diagonal_grad}. Taking expectations on both the sides with respect to $z$, we get that 
\begin{align*}
\tri*{w, \grad_{w} F(w)} &= 2 \En_{x, y} \brk*{\abs{y - h(w; x)}} - 2 \En_{x, y} \brk*{\sign{y - h(w; x)} y} \\
		&= 2 F(w) - 2 \En_{x, y} \brk*{\sign{y - h(w; x)} y} \\
		&\geq 2 F(w) - 2 \En_{x, y} \brk*{\abs{y}} \\
		&= 2(F(w) - F(0)), 
\end{align*} where the last line follows by observing that $\En \brk*{\abs{y}} = \En_{x, y} \brk*{\abs{y} - h(0; x)} = F(0)$. Defining $w^* \ldef{} 0$, the above implies that  
\begin{align*} 
F(w) - F(w^*) \leq \frac{1}{2} \tri*{w - w^*, \grad_{w} F(w)}, 
\end{align*} thus showing that $F(w)$ is $1/2$-Linearizable. 
\end{proof} 
\begin{proof}[Proof of \pref{thm:two_layer_failure}] Consider the distribution $\cD$ over the instance space $\cZ = \crl{0, 1}^d \times \crl{0, 1}$ where 
\begin{align*} 
x \sim \text{Uniform}(\crl*{0, 1}^d),\qquad \text{and} \qquad y \sim \cB(1/4). \numberthis \label{eq:nn_distribution_defn} 
\end{align*}

We now prove the two parts separately below:  
\begin{enumerate}[label=$(\alph*)$, leftmargin=8mm] 
\item In the following, we show that SGD algorithm, run with an additional projection on the unit norm ball after every update, learns at a rate of $O(1/\sqrt{n})$. In particular, we use the following update step for $t \in [n]$, 
\begin{align*} 
w^{\mathrm{SGD}}_{t+1} \leftarrow \Pi_{w_1}\prn[\big] {\w^{\mathrm{SGD}}_{t} - \eta \nabla f(\w^{\mathrm{SGD}}_t;z_t)} 
\end{align*} 
where the initial point $w_1$ is chosen by first sampling $w'_1 \sim \cN(0, \bbI_d)$ and then setting $w_1 = w'_1/\nrm{w_1}$, and the projection operation $\Pi$ is given by 
\[\Pi_{w_1}(w) = \begin{cases}
   w & \text{ if }\|w - w_1\| \leq 1 \\
   w_1 + \frac{1}{\|w-w_1\|}(w-w_1) & \text{ otherwise.} 
\end{cases}\] 

After taking $n$ steps, the point returned by the SGD algorithm is given by $\wh w^{\text{SGD}}_n \ldef{} \sum_{t=1}^n w^{\text{SGD}}_t / n$. 

First note that, for the distribution given in \pref{eq:nn_distribution_defn}, \pref{lem:f_linearizable} implies that the population loss $F(w)$ is $1/2$-Liearizable w.r.t. the global minima $w^* = 0$. Next, note that for any point $w$ and data sample $z$, 
\begin{align*} 
\nrm{\grad f(w; z)}^2 &= \nrm{\grad_{w_1} f(w; z)}^2 + \nrm{\grad_{w_2} f(w; z)}^2  \\
								 &\leq \sum_{j=1}^d (w_{2, j} \cdot x_j)^2 + \sum_{j=1}^d (\relu(w_{1, j} \cdot x_j))^2 \\ 
								 &\leq\sum_{j=1}^d w_{2, j}^2 + \sum_{j=1}^d w_{1, j}^2 \\
								 &= \nrm*{w}^2 
\end{align*} where the inequality in the second line follows by plugging in the definition of $\grad_{w_1} f(w; z)$ and $\grad_{w_2} f(w; z)$, and by upper bounding the respective indicators by $1$. The inequality in the third line above holds because $\relu(h) \leq \abs{h}$ for any $h \in \bbR$,  and by using the fact that $x_j \in \crl{0, 1}$ for $j \in [d]$. Since, the iterates produced by SGD algorithm satisfy $\nrm{w_t^{\mathrm{SGD}}} \leq 1$  due to the projection step, the above bound implies that $\nrm{\grad f(w_t^\text{SGD}; z)} \leq 1$ for any $t \geq 0$, and thus $\nrm{\grad F(w_t^\text{SGD})} \leq 1$. 

The above bounds imply that Assumption II (in \pref{eq:ass2}) holds on the iterates generated by the SGD algorithm with $\max\crl{B, \sigma, L} \leq 2$. Furthermore, the population loss $F(w)$ is $1/2$-Linearizable. Thus, repeating the steps from the proof of \pref{thm:sgd_works_linearizable} on page~\pageref{proof:sgd_works_linearizable_proof}, and using the fact that $\nrm{\Pi_{w_1}(w) - w^*} \leq \nrm{w - w^*}$ for $w^* = 0$ to account for the additional projection step, we get that the point $\wh w^{\text{SGD}}_n$ returned by the SGD algorithm enjoys the performance guarantee 
\begin{align*} 
F(\wh w^{\text{SGD}}_n) - F^* \leq O\prn[\Big]{\frac{1}{\sqrt{n}}}. 
\end{align*}

\item  In the following, we will show that for $d > \log(10) 2^n + 1$, with probability at least $0.9$ over the choice of $S \sim \cD^n$, there exists an ERM solution for which 
$$ 
F(\w_{\mathrm{ERM}}) - \inf_{\w \in \reals^d} F(\w) \ge \Omega(1). 
$$  

Suppose that the dataset $S = \crl*{(x_i, y_i)}_{i=1}^n$ is sampled i.i.d.\ from the distribution $\cD$ given in \pref{eq:nn_distribution_defn}. A slight modification of \pref{lem:datatset_properties_basic_sco} implies that for $n \leq \log_2(d / \log(10)) $, with probability at least $0.9$, there exists a coordinate $\wh j$ such that $x_i[\wh j] = y_i$ for all $i \in [n]$. In the following, we condition on the occurrence of such a coordinate $\wh j$. Clearly, the empirical loss at the point $\wh w = (e_{\wh j}, e_{\wh j})$ is: 
\begin{align*}
	\wh F(\wh w) 
						 &= \sum_{i=1}^n \abs{y_i - \relu(e_{\wh j} \relu(e_{\wh j} \odot x_i))} = \sum_{i=1}^n \abs{y_i - x_i[\wh j]} = 0, 
\end{align*} where the last equality follows the fact that $x_i[\wh j] = y_i$ for all $i \in [n]$. Since $\wh F(w) \geq 0$ for any $w$, we get that the point $\wh w$ is an ERM solution. Next, we note that the population loss at the point $\wh w$ satisfies 
\begin{align*}
F(\wh w) - \min_{w} F(w)  &\overeq{\proman{1}} F(\wh w) - \frac{1}{4} \\ 
										 &\overeq{\proman{2}} \En_{x, y} \brk*{\abs*{y - \relu(e_{\wh j} \relu(e_{\wh j} \odot x))}} - \frac{1}{4} \\ 
										 &=  \En_{x, y} \brk{\abs{y - x[\wh j]}} - \frac{1}{4} \overeq{\proman{3}} \frac{1}{4}  
\end{align*} where the equality $\proman{1}$ follows by observing that $\min_{w} F(w) = 1/4$ (see the proof of \pref{lem:f_linearizable} for details), the equality $\proman{2}$ holds for $\wh w = (e_{\wh j}, e_{\wh j})$ and $\proman{3}$ follows by using the fact that $y \sim \cB(1/4)$ and $x[\wh j] \sim \cB(1/2)$ and that $x$ and $y$ are sampled independent to each other. Thus, there exists an ERM solution $w_{\text{ERM}} = \wh w$ for which the excess risk: 
\begin{align*}
F(w_{\text{ERM}}) - \min_{w} F(w) = \frac{1}{4}. 
\end{align*}
The desired claim follows by observing that the coordinate $\wh j$ above occurs with probability at least $0.9$ over the choice of $S \sim \cD^n$.  
\end{enumerate}
\end{proof}

\subsection{Expressing $\fA$ and $\fB$ using neural networks} \label{app:NN_representations} 

In this section, we show that the loss functions $\fA$ and $\fB$ can be represented using restricted neural networks (where some of the weights are fixed to be constant and thus not trained)  with $\poly(d)$ hidden units. We first provide neural network constructions that in addition to $\relu$, use square function  $\sigma(a) = a^2$  and square root function $\wt \sigma(a) = \sqrt{a}$ as activation functions.  We then give a general representation result in \pref{lem:relurepresentation} which implies that the activations $\sigma$ and $\wt \sigma$ can be approximated both in value and in gradients simultaneously using $\poly(d)$ number of $\relu$ units. This suggests that the functions  $\fA$ and $\fB$ can be represented using restricted RELU networks of $\poly(d)$ size.   

Note that in our constructions, the NN approximates the corresponding functionst in both value and in terms of the gradient. Thus, running gradient based optimization algorithms on the corresponding NN representation will produce similar solutions as running the same algorithm on the actual function that it is approximating. 

\begin{proposition}\label{prop:nn} 
Function $ \fA$ in Equation (\ref{eq:empfn_basic_cons_sco}) can be represented as a restricted diagonal neural network with square and square root activation functions with $O(d)$ units and constant depth. 
\end{proposition}
\begin{proof} In the following, we will assume that before passing to the neural network, each data sample $z = (x, \alpha, y)$ is preprocessed to get the features $\wt x$ defined as $\tilde{x} \ldef{} (x,  -\alpha \odot x)^T \in \bbR^{2d}$. The vector $\wt x$ is given as the input to the neural network. 

We construct a three layer neural network with input $\wt x \in \bbR^{2d}$, and weight matrices $W_1 \in \bbR^{2d \times 2d}$, $W_2 \in \bbR^{2d \times d}$ and $W_3 \in \bbR^{d \times 1}$  for the first layer, the second layer and the third layer respectively. The activation function after the $i$th layer is given by $\sigma_i: \bbR \mapsto \bbR$. The neural network consists of $d$ trainable parameters, given by $\w \in \bbR^d$, and is denoted by the function  $h(\w; \wt x):  \bbR^d  \times \bbR^{2d} \mapsto \bbR$. In the following, we describe the corresponding weight matrices and activation functions for each layer of the network: 

\begin{enumerate}[label=$\bullet$] 
\item \textbf{Layer 1:} Input: $\wt x \in \bbR^{2d}$. The weight matrix $W_1: \bbR^{2d \times 2d}$ is given by  
\begin{align*} 
	W_1[i, j] \ldef{} \begin{cases} \w[j] &\text{if~} 1 \leq i = j \leq d \\ 
	1 &\text{if~} d < i = j \leq 2d \\ 
	0 &\text{otherwise}
	\end{cases}. 
\end{align*}
The activation function $\sigma_1$ is given by $\sigma_1(a) = a$. Let $h_1(\w; \wt x) \ldef{} \sigma_1(\wt x W_1)$ denote the output of this layer. We thus have that  $h_1(\w;  \wt x)  = (\w \odot x, - \alpha \odot x)^T$. 

\item \textbf{Layer 2:} Input: $h_1(\w;  \wt x) \in \bbR^{2d}$. The weight matrix $W_2: \bbR^{2d \times d}$ is given by 
\begin{align*}
W_1[i, j] \ldef{} \begin{cases}
	1 &\text{if} ~ i - j \in \crl{0, d} \\
	0 & \text{otherwise} 
\end{cases}
\end{align*}
The activation function $\sigma_2$ is given by $\sigma_2(a) = a^2$. Let $h_2(\w;  \wt x) \ldef{} \sigma_2(h_1(\wt z, \w) W_2)$ denote the output of this layer. We thus have that $h_2(\w;  \wt x)[j] = \prn*{w[j] \odot x[j] - \alpha[j] \odot x[j]}^2$ for any $j \in [d]$.  

\item \textbf{Layer 2:} Input: $h_2(\w;  \wt x) \in \bbR^{d}$. The weight matrix $W_3: \bbR^{d \times 1}$ is given by the vector $\mb{1}_d = \prn{1, \ldots, 1}^T$. 

The activation function $\sigma_2$ is given by $\sigma_2(a) = \sqrt{a}$. Let $h_3(\w;  \wt x) \ldef{} \sigma_3(h_2(\w;  \wt x) W_3)$ denote the output of this layer. We thus have that $h_3(\w;  \wt x)[j] = \nrm*{\w \odot x - \alpha \odot x}$.
\end{enumerate} 

Thus, the output of the neural network is 
\begin{align*}
h(\w;  \wt x) &= h_3(\w;  \wt x) \\ 
&= \sigma_3( \sigma_2(\sigma_1(\wt x W_1) W_2) W_3) = \nrm*{(\w - \alpha) \odot x}. 
\end{align*}

Except for the first $d$ diagonal elements of $W_1$, all the other weights of the network are kept fixed during training. Thus, the above construction represents a restricted neural network with trainable parameters given by $\w$. Also note that in this first layer, any input node connects with a single node in the second layer. Such networks are known as diagonal neural networks \citep{gunasekar2018implicitb}. 

We assume that the network is trained using linear loss function, i.e. for the prediction $h(\w;  \wt x)$ for data point $(\wt x, y)$, the loss is given by
\begin{align*}
\ls(\w;  \wt x) &= y \cdot h(\w;  \wt x) = y \nrm*{ (\w - \alpha) \odot x},   \numberthis \label{eq:NN_loss_function} 
\end{align*}

Note that the above expression exactly represents $\fA(\w; z)$. This suggests that learning with the loss function $\fA$ is equivalent to learning with the neural network $h$ (defined above with trainable parameter given by $\w$) with linear loss. Furthermore, the network $h$ has a constant depth, and $O(d)$ units, proving the desired claim. 
\end{proof} 

We next show how to express the loss function $\fB$ using a neural network. 
\begin{proposition} \label{prop:nn2} 
Function $ f_{(\ref{eq:empfn_basic_cons})}$ in Equation (\ref{eq:empfn_basic_cons}) can be represented as a restricted diagonal neural network with square and square root activation functions with $O(d)$ units and constant depth.
\end{proposition}

\begin{proof}[{\bf Proof of \pref{prop:nn2}}~] 
In the following, we will assume that before passing to the neural network, each data sample $z = (x, \alpha)$ is preprocessed to get the features $\wt x \ldef{} (x,  -\alpha \odot x, \mb{1}_d, -\alpha, \mb{1}_d, \mb{1}_d)^T \in \bbR^{6d}$. The vector $\wt x$ is given as the input to the neural network.

We construct a four layer neural neural network with input $\wt x \in \bbR^{6d}$ and weight matrices $W_1 \in \bbR^{6d \times 6d}$, $W_2 \in \bbR^{6d \times 3d}$, $W_3 \in \bbR^{3d \times 2}$, $W_4 \in \bbR^{2 \times 1}$  for the four layers respectively. The activation function after the $i$th layer is given by $\sigma_i: \bbR \mapsto \bbR$. The neural network consists of $d$ trainable parameters, given by $\w \in \bbR^d$, and is denoted by the function  $h(\w;  \wt x):  \bbR^d  \times \bbR^{6d} \mapsto \bbR$. In the following, we describe the corresponding weight matrices and activation functions for each layer of the network: 

\begin{enumerate}[label=$\bullet$] 
\item \textbf{Layer 1:} Input: $\wt x \in \bbR^{4d}$. The weight matrix $W_1: \bbR^{6d \times 6d}$ is given by 
\begin{align*} 
	W_1[i, j] \ldef{} \begin{cases} \w[j] &\text{if~} i = j \text{~and~} j - \alpha d \leq d \text{~for~} \alpha \in \crl{0, 2, 4} \\
	1 &\text{if~} i = j \text{~and~} j - \alpha d \leq d \text{~for~} \alpha \in \crl{1, 3, 5} \\ 
	0 &\text{otherwise}
	\end{cases}. 
\end{align*}
The activation function $\sigma_1$ is given by $\sigma_1(a) = a$. Let $h_1(\w;  \wt x) \ldef{} \sigma_1(\wt x W_1)$ denote the output of this layer. We thus have that  $h_1(\w;  \wt x)  = (\w \odot x, - \alpha \odot x, \w, - \alpha, \w, \mb{1}_d)^T$.  

\item \textbf{Layer 2:} Input: $h_1(\w;  \wt x) \in \bbR^{6d}$. The weight matrix $W_2: \bbR^{6d \times 3d}$ is given by 
\begin{align*}
W_1[i, j] \ldef{} \begin{cases}
	1 &\text{if} ~ i - j \in \crl{0, d} \text{~and~} j \leq d \\
		1 &\text{if} ~ i - j \in \crl{d, 2d} \text{~and~} d < j \leq 2d \\ 
		1 &\text{if} ~ i = j + 3d  \text{~and~} 2d < j \\ 	
	0 & \text{otherwise} 
\end{cases}. 
\end{align*} 
The activation function $\sigma_2$ is given by $\sigma_2(a) = a^2$. Let $h_2(\w;  \wt x) \ldef{} \sigma_2(h_1(\wt z, \w) W_2)$ denote the output of this layer. We thus have that $h_2(\w;  \wt x)[j] = \prn*{\w[j] \odot x[j] - \alpha[j] \odot x[j]}^2$ for any $j \in [d]$, $h_2(\w;  \wt x)[j] = \prn*{\w[j] - \alpha[j]}^2$ for any $d < j \leq 2d$ and $h_2(\w;  \wt x)[j] = \prn*{\w[j]}^2$ for $2d < j \leq 3d$.   

\item \textbf{Layer 3:} Input: $h_2(\w;  \wt x) \in \bbR^{3d}$. The weight matrix $W_3: \bbR^{3d \times 2}$ is given by 
\begin{align*}
W_3[i, j] &= \begin{cases}
	\frac{1}{2} & \text{if}~ j = 1~ \text{and}~ 1 \leq i \leq d \\ 
	- \frac{c_n}{2} & \text{if} ~ j = 1 ~ \text{and}~ d + 1 \leq i \leq 2d \\ 
	1   & \text{if}~ j = 2~\text{and}~2d + 1 \leq i \leq 3d
\end{cases}. 
\end{align*}
For the first node in the output of this layer, we use the activation function $\sigma_2(a) = \sqrt{a}$ and for the second node, we use the activation function  $\sigma_2(a) = \max\crl{1, a^2}$. Let $h_3(\w;  \wt x) \ldef{} \sigma_3(h_2(\w;  \wt x) W_3)$ denote the output of this layer. We thus have that $h_3(\w;  \wt x)[j] = (\frac{1}{2}\nrm*{(\w - \alpha) \odot x}^2 - \frac{c_n}{2}\nrm*{(\w - \alpha) \odot x}^2, \max\crl{1, \nrm{\w}^4})$.

\item \textbf{Layer 4:} Input: $h_3(\w;  \wt x) \in \bbR^{2}$. The weight matrix $W_4: \bbR^{2 \times 1}$ is given by $W_4 = (1, 1)$, and the activation function $\sigma_4$ is given by $\sigma_4(a) = a$. Let $h_4(\w;  \wt x) \ldef{} \sigma_3(h_2(\w;  \wt x) W_3)$ denote the output of this layer. We thus have that $h_4(\w;  \wt x)[j] = \frac{1}{2}\nrm*{(\w - \alpha) \odot x}^2 - \frac{c_n}{2}\nrm*{(\w - \alpha) \odot x}^2 + \max\crl{1, \nrm{\w}^4}$.
\end{enumerate} 

Thus, the output of the neural network is given by 
\begin{align*}
h(\w;  \wt x) &= h_4(\w;  \wt x) \\ 
&= \sigma_4(\sigma_3( \sigma_2(\sigma_1(\wt x W_1) W_2) W_3)W_4) \\
&= \frac{1}{2}\nrm*{(\w - \alpha) \odot x}^2 - \frac{c_n}{2}\nrm*{(\w - \alpha) \odot x}^2 + \max\crl{1, \nrm{\w}^4}. 
\end{align*}

In the above construction, the first layer can be thought of as a convolution with filter weights given by $\diag(\w, \mb{1}_d)$ and stride $2d$. While training the neural network, we keep all the weights of the network fixed except for the ones that take values from $\w$ (in the weight $W_1$). Thus, the above construction represents a restricted CNN with trainable parameters given by $\w$. 

Furthermore, for the prediction $h(\w;  \wt x)$ for data point $\wt x$, we treat the output of the neural network as the loss, which is given by 
\begin{align*}
\ls(\w;  \wt x) &= h(\w;  \wt x)  + \max\crl{1, \nrm*{w}^4} \\
&= \frac{1}{2}\nrm*{(\w - \alpha) \odot x}^2 - \frac{c_n}{2}\nrm*{(\w - \alpha) \odot x}^2  + \max\crl{1, \nrm*{w}^4}.
\end{align*}

Note that the above expression exactly represents $\fB(\w; z)$. This suggests that learning with the loss function $\fB$ is equivalent to learning with the neural network $h$ (defined above with trainable parameter given by $\w$).  Furthermore, the network $h$ has a constant depth, and $O(d)$ units, proving the desired claim. 
\end{proof} 

We next provide a general representation result which implies that the activation functions $\sigma(a) = a^2$ (square function) and $\wt \sigma(a) = \sqrt{a}$ (square root function), used in the constructions above, can be approximated both in value and in gradients simultaneously using $\poly(d)$ number of $\relu$ units. 

\begin{lemma} \label{lem:relurepresentation} 
Let $f: [a, b] \to \bbR$ be an $L$-Lipschitz and $\alpha$-smooth function. Then for any $\epsilon > 0$, there is a function $h: [a, b] \to \bbR$ that can be written as a linear combination $h$ of $\ceil[\big]{\frac{(b-a)\max\{L,\alpha\}}{\epsilon}} + 1$ ReLUs with coefficients bounded by $2L$ such that for all $x \in [a, b]$, we have $|f(x) - h(x)| \leq \epsilon$, and if $h$ is differentiable at $x$, then $|f'(x) - h'(x)| \leq \epsilon$.
\end{lemma}
\begin{proof}
Consider dividing up the interval $[a, b]$ into equal $n$ equal sized intervals of length $\delta = \frac{b-a}{n}$, for $n = \lceil\frac{(b-a)\max\{L,\alpha\}}{\epsilon}\rceil$. Define $a_i = a + i\delta$ for $i = 0, 1, \ldots, n$. Let $h$ be a piecewise linear function that interpolates $f$ on the $n+1$ endpoints of the intervals. For any such interval $[a_i, a_{i+1}]$, by the mean value theorem, the slope of $h$ on the interval is equal to $f'(x_i)$ for some $x_i \in (a_i, a_{i+1})$, and hence is bounded by $L$ since $f$ is $L$-Lipschitz. Furthermore, by the $L$-Lipschitzness and $\alpha$-smoothness of $f$, for any $x \in (a_i, a_{i+1})$, we have $$|f(x) - h(x)| \leq \delta L \leq \epsilon$$ and $$|f'(x) - h'(x)| = |f'(x) - h'(x_i)| = |f'(x) - f'(x_i)| \leq \delta \alpha \leq \epsilon.$$ Now, by \pref{lem:piecewiselinear} we can represent $h$ as a linear combination of $n+1$ ReLUs with coefficients bounded by $2L$. 
\end{proof}

\begin{lemma} \label{lem:piecewiselinear} 
Any piecewise linear function $f: \bbR \rightarrow \bbR$ with $K$ segments can be written as a linear combination of $K+2$ ReLUs with coefficients bounded by twice the maximum slope of any segment of $f$.
\end{lemma} 
\begin{proof}
A piecewise linear function $f: \bbR \rightarrow \bbR$ is fully determined by the endpoints $a_1 < a_2 < \cdots < a_K$ for some positive integer $K$ which define the segments of $f$, the value $f(a_1)$, and the slopes $m_0, m_1, \ldots, m_K \in \bbR$, such that the slope of $f$ on the segment $(a_i, a_{i+1})$ is $m_i$, where we define $a_0 := -\infty$ and $a_{K+1} = +\infty$ for convenience. Specifically, we can write $f$ as the following:
\[f(x) = f(a_1) + \begin{cases}
	m_0(x - a_1) & \text{ if } x < a_1 \\
	\sum_{i=1}^{\ell-1} m_i(a_{i+1} - a_i) + m_\ell(x - a_\ell) & \text{ if } x \in [a_\ell, a_{\ell+1})
\end{cases}\] 
Now define $\sigma(x) = \max\{x, 0\}$ to be the ReLU function, and consider the function $h$ defined as
\[h(x) = f(a_1) - m_0\sigma(a_1 - x) + m_0\sigma(x - a_1) + \sum_{i=1}^K(m_i - m_{i-1})\sigma(x - a_i).\]
Since this is a linear combination of ReLUs, it is a piecewise linear function. By direct calculation, one can check that $h(a_1) = f(a_1)$, the endpoints of the segments of $h$ are $a_1, a_2, \ldots, a_K$, and the slopes of the segments are $m_0, m_1, \ldots, m_K$. Hence, the function $f = h$. 
\end{proof}

\end{document}